\documentclass[11pt]{article}
\usepackage{fullpage}
\usepackage{amsmath}
\usepackage{amsfonts}
\usepackage{amssymb}  
\usepackage{authblk}
\usepackage{dsfont}
\usepackage{float}

\usepackage{graphicx}
\usepackage{pgfplots}
\usepackage{tikz}
\usepackage{subcaption}

\usepackage[authoryear,round]{natbib}

\newcommand{\blindversion}[1]{}

\usepackage{nameref}
\usepackage[plainpages=false]{hyperref}
\usepackage{cleveref}

\usepackage{amsmath}
\usepackage{amssymb}
\usepackage{mathtools}
\usepackage{graphicx}
\usepackage{xcolor}
\usepackage{pgfplots}
\usepgfplotslibrary{fillbetween}

\usepackage{latexsym}
\usepackage{amsmath,amssymb,amsthm}
\usepackage{epsfig}
\usepackage[right=0.8in, top=1in, bottom=1.2in, left=0.8in]{geometry}
\usepackage{setspace}
\usepackage{cleveref}
\usepackage{amssymb}
\usepackage{thmtools}
\usepackage{color}
\newtheorem{theorem}{Theorem}

\newtheorem{lemma}{Lemma}

\newtheorem{definition}{Definition}

\newtheorem{remark}{Remark}

\newtheorem{proposition}{Proposition}%[section]

\usepackage{amsmath}

\newcommand{\Fset}{{\cal X}}

\newcommand {\beq}{\begin{equation}}
\newcommand {\eeq}{\end{equation}}
\newcommand {\beqn}{\begin{equation*}}
\newcommand {\eeqn}{\end{equation*}}
\newcommand {\bear}{\begin{eqnarray}}
\newcommand {\eear}{\end{eqnarray}}
\newcommand {\bearn}{\begin{eqnarray*}}
\newcommand {\eearn}{\end{eqnarray*}}

\usepackage{xcolor}
\usepackage{tikz}
\usetikzlibrary{intersections}
\usetikzlibrary{positioning}
\usetikzlibrary{decorations.pathreplacing}
\usetikzlibrary{matrix}

\usepackage{pgfplots}
\pgfplotsset{compat=newest}
\usepgfplotslibrary{fillbetween}

\usepackage{enumerate}

\usepackage[algoruled,boxed,lined]{algorithm2e}
\DeclareMathOperator*{\argmin}{arg\,min}
\DeclareMathOperator*{\argmax}{arg\,max}

\usepackage{soul}
\usepackage[normalem]{ulem}

\begin{document}
\setstretch{1.3}

\title{Contextual Inverse Optimization: Offline and Online Learning\footnote{An early version of some of the results in present paper  appeared in COLT'21 with the title ``Online Learning from Optimal Actions.'' Only the abstract appeared in the conference proceedings. We thank Alberto Seeger, the COLT program committee members, the area editor, the associate editor and the referees for their valuable feedback.}}

\author{
  Omar Besbes\footnote{Columbia Business School --- ob2015@gsb.columbia.edu.}, ~Yuri Fonseca\footnote{Columbia Business School --- yfonseca23@gsb.columbia.edu.}, ~and Ilan Lobel\footnote{NYU Stern School of Business --- ilobel@stern.nyu.edu.}
}
%\date{June 23, 2022}
\date{first version: January 2021; this version: June 2023}
%\begin{document}
\maketitle
\begin{abstract}
	We study the problems of offline and online contextual optimization with feedback information, where instead of observing the loss, we observe, after-the-fact, the optimal action an oracle with full knowledge of the objective function would have taken. We aim to minimize regret, which is defined as the difference between our losses and the ones incurred by an all-knowing oracle. In the offline setting,  the decision-maker has information available from past periods and needs to make one decision, while in the online setting, the decision-maker optimizes decisions dynamically over time based a new set of feasible actions and contextual functions in each period. For the offline setting, we characterize the optimal minimax policy, establishing the performance that can be achieved as a function of the underlying geometry of the information induced by the data. In the online setting, we leverage this geometric characterization to optimize the cumulative regret. We develop an algorithm that yields the first regret bound for this problem that is logarithmic in the time horizon. Finally, we show via simulation that our proposed algorithms outperform previous methods from the literature. 
\end{abstract}


\section{Introduction}\label{sec:intro}

The two classical frameworks for studying decision-making with learning are the offline setting, in which a decision-maker has access to a data set of contexts, actions, and respective payoffs, and the online setting, in which a decision-maker in each period is given a context, chooses an action and receives feedback in the form of a payoff. Crucially, both of these frameworks assume the decision-maker observes payoffs. However, there are important settings in which a machine learning algorithm might not have access to the payoff from a decision, but it does have access after-the-fact to the decision an expert  would have taken in that situation. 

Learning problems of this form arise in a variety of settings. Consider a medical machine learning system trying to learn to emulate a doctor's approach to treating a disease. The machine learning system can observe the patient covariates (demographic information, medical history, blood work, etc.) and needs to issue a treatment recommendation for a junior doctor to follow. A senior doctor (an expert) is not available immediately, but is able to review the treatment decision after-the-fact and input what his or her decision would have been instead. That is, we have access to what an optimal solution would have been to the problem we faced. However, we might never receive feedback on whether a treatment worked, as patients might not inform the medical practice what the outcome of the treatment was, or the feedback might be significantly delayed in time. A similar scenario occurs if we are trying to build a system to predict users' preferences.  Most often, we can observe what the user actually did after-the-fact, but we cannot observe the payoff to the user of a particular decision. The same is true of some machine learning systems that are designed to learn how to operate an autonomous system. We might be able to learn by observing what a human driver (the expert) would have done, but we might not have any payoff feedback. To formalize this class of problems, we consider a general contextual optimization problem with both offline and  online data.

For the offline setting, the decision-maker has access to data $\mathcal D$  of  $N$  contextual optimization problems, each  consisting of a set of feasible actions $\mathcal X_i \subset \mathbb R^n$, a context function $f_i: \mathbb R^n \to \mathbb R^d$ and an optimal solution $x^\star_i \in \argmin_{x\in \mathcal X_i} f(x)'c^\star$ for some \textit{unknown} vector $c^\star$. The challenge is to devise a good action in a new problem, based on the offline data. In particular, a key question pertains to how ``informative" the offline data is for future decision-making in new environments not necessarily seen before. 

  For the online setting, in each period $t$, the decision-maker receives a new feasible set $\mathcal X_t \subseteq \mathbb R^n$ and a new context function $f_t: \mathbb R^n \to \mathbb R^d$, and must select some action $x_t \in \mathcal X_t$ in order to minimize an underlying objective function $f_t(x_t)'c^\star$. As in the offline setting, $c^\star$ is an unknown cost vector, which is only known to belong to some initial knowledge set $C_0 \subseteq  \mathbb R^d$. The key question in this setting is if the decision-maker can gain information about $c^\star$ over time, and what type of performance can be achieved over time when we only observe past optimal actions. % and one faces an arbitrary sequence of optimization problems as above. 
  Informally speaking, we are interested in understanding how well can the decision-maker learn to mimic the expert. In both the offline and online settings, we aim to minimize the worst-case regret, which is measured in terms of the suboptimality gap of the actions taken. In particular, the regret we study captures the quality of the decision made under the \textit{true} cost and formally introduced in \Cref{sec:pb-form}. While not directly observable, we derive theoretical bounds on it.

\subsection{Main contributions}

In the offline setting, we do not make any distributional assumptions on the contextual optimization problems. To the best of our knowledge, we establish the first data-dependent minimax regret result for this problem. We do this by exploiting the underlying geometry of the contextual optimization problem. For the offline problem with data $\mathcal D$, we first characterize the set of cost vectors $C(\mathcal D)$ that are consistent with the offline data and establish that a key driver of performance is the uncertainty angle of the information set, $\alpha(C(\mathcal D))$, which is a measure of how large a revolution cone do we need to contain $C(\mathcal D)$. The axis of such a revolution cone is what's called the circumcenter of $C(\mathcal D)$. We show in \Cref{theorem:one_period_problem} that the offline policy that guarantees an optimal worst-case regret bound  for an arbitrary new optimization problem is to treat the circumcenter of $C(\mathcal D)$ as if it were the true cost. In particular, we show for the offline problem that regardless of the data size, the worst-case future  regret is given by $\sin \alpha(C(\mathcal D))$ if $\alpha(C(\mathcal D)) \leq \pi/2$ and 1 if $\alpha(C(\mathcal D)) > \pi/2$.

For the online case, our first result is a negative one. In \Cref{lemma:insuff_circumcenter}, we show that a na\"ive application of the circumcenter policy fails, in the sense that it may incur regret that is linear in the time horizon. Quite interestingly, this policy may, in all periods, incur significant regret while \textit{also} failing to learn any meaningful new information on the underlying cost vector. This is due to the potentially complex geometry of the knowledge set and highlights how nature may counter the decision-maker and ``gain" on both fronts: inflict regret \textit{and} limit information collection to ensure high future regret. 

 In turn, we develop an approach that leverages a series of ideas to exploit the geometry of the problem while avoiding the pitfalls of the na\"ive  (greedy) circumcenter policy.  We first assume that the initial knowledge set $C_0$ lives in a pointed cone. We regularize our knowledge sets by replacing them with ellipsoidal cones that contain them. This regularization in conjunction with an adjusted circumcenter policy introduces a trade-off for nature: now nature will either inflict high regret or will not enable the decision-maker to collect useful information on the cost vector, but can't anymore achieve both at the same time. In other words, the decision-maker can now indirectly ``force" nature to reveal information about the cost vector, a phenomenon we coin ``inverse exploration." We then adapt the ellipsoid method from optimization theory to deal with ellipsoidal cones instead and appropriately update our regularized sets over time. The key to our  algorithm's performance is to stop the ellipsoidal cones from becoming ill-conditioned, which we are able to achieve by not performing  ellipsoidal updates in periods in which the decision is nearly optimal. In \Cref{theorem:regret}, we establish that the algorithm we construct, properly tuned, achieves a worst-case regret bounded by $\mathcal O(d^2 \ln (T \tan \alpha(C_0)))$. 
 
The algorithm above was constructed assuming that the initial knowledge set $C_0$ lived  in a pointed cone, and enabled to highlight a first set of phenomena at play in this class of problems. In Section \ref{sec:algo_design_general}, we drop this assumption. To address this generalized version of the problem, we establish  that it is possible to always maintain a subspace such that the projection of our knowledge set onto it has a bounded uncertainty angle (lives within a pointed cone). We then extend our algorithm to have three different kinds of periods: periods where our action is nearly optimal (no knowledge set update is needed), periods where we perform an ellipsoidal cone update, and periods where we add an extra dimension to our subspace. There are two new key steps in this algorithm. First, we need to robustify the ellipsoidal cuts to account for potential error introduced by the projection onto a subspace. Second, we need to construct a new ellipsoidal cone every time we increase the dimension of the subspace, and we need to argue that this new ellipsoidal cone has a bounded uncertainty angle. For the constructed algorithm, properly tuned,  we are able to obtain a universal $\mathcal O(d^4 \ln T)$ regret bound in \Cref{theorem:regret_general_case}, regardless of the initial knowledge set, which is the main result of this paper. For comparison, the prior state-of-the-art results (algorithms and associated performance) on the regret notion we study were presented in \cite{barmann2017emulating} and report regret $\mathcal O(\sqrt T)$, for what the authors defined there as ``solution error".

We complement our theoretical analysis with a numerical study. We illustrate the performance of our algorithms on an important type of contextual optimization problem where the expert (consumer) solves a sequence of knapsack problems with features, prices and products varying over time. We compare the performance of our algorithms in both the pointed and the general case with the Exponential Weighs Update (EWU) presented in \cite{barmann2017emulating} and Online Gradient Descent (OGD) presented in \cite{barmann2018online}. We show that our algorithms enjoy good empirical performance and outperform the benchmarks in all instances tested.
 
 These results provide new achievability results for both the offline and the online setting, but also novel algorithmic ideas to account for the special nature of the feedback associated with optimal actions. We also hope that some of these ideas can be useful in tackling more general versions of the problem, such as problems with noise in the feedback.

\vspace{.05in}
\textbf{Structure of the paper.} Next, we position our paper in the broad literature. We formulate the problem in \Cref{sec:pb-form}. We analyze the offline setting in \Cref{sec:one-period} and the online setting in Sections \ref{sec:circumcenter_policy} and \ref{sec:algo_design_general}. In Section \ref{sec:numerics}, we conduct illustrative numerical experiments. In Section \ref{sec:conclusion}, we discuss open questions and possible extensions. We discuss connections to other models in the literature in Appendix \ref{sec:appl} and present the proofs of the results in the paper in Appendix \ref{app:proofs}.

\subsection{Related Literature}\label{sec:lit}

The problem formulation we study allows us to encompass and relate to a variety of problems studied earlier in the literature such as inverse optimization, imitation learning and structured prediction problems with a host of associated applications. We next broadly discuss how our work relates to these  streams of literature.

 In the offline setting, inverse optimization  typically refers to the problem of finding an objective function given an optimal solution and a feasible set. \cite{ahuja2001inverse} defined this problem early on and we refer the reader to \cite{chan2019inverse} for recent developments. See also \cite{keshavarz2011imputing}, \cite{bertsimas2015data}, \cite{aswani2018inverse},  \cite{thai2018imputing}.  Perhaps closest to our offline formulation is  \cite{esfahani2018data}, which introduces, among others, a convex loss denoted suboptimality loss to select among models in a distributionally robust framework. They make distributional assumptions and work with strong convex nominal optimization problems and with noisy observations, and provide out-of-sample guarantees for the Conditional Value-at-Risk for these losses. In the present paper, we analyzes a more general nominal optimization problem, but without noise, and provide a guarantee for the regret under an arbitrary new  optimization instance. Moreover, we define the regret using a different loss, that quantifies the suboptimality gap of the decision-maker's action with respect to the expert under the true cost vector. We discuss this difference in \Cref{sec:pb-form}.

In the online setting, perhaps the closest paper to ours is  \cite{barmann2017emulating, barmann2018online}, which studies an online version of inverse optimization. The key step in \cite{barmann2017emulating} is a reframing that allows the authors to leverage online learning algorithms such as online gradient descent and exponential weights updates, and prove a regret bound of $\mathcal O(\sqrt{T})$. We discuss this result and their approach in more detail after we state Theorem 4. In \cite{ward2019learning}, the authors develop an online learning  algorithm that is built on exponential weights updates for dynamic scheduling with a regret guarantee of $\mathcal O(\ln T\sqrt T)$. \cite{jabbari2016learning} studies a related setting with a  sequence of linear programming problems associated with the behavior of a rational agent. Under the assumption that the true cost vector is separated from all other candidate ones through the use of a precision parameter, they show that the ellipsoid algorithm ensures that the number of mistakes is bounded by a polynomial function of the precision parameter. In contrast, we do not make any separation assumption and work with a continuum of cost vector candidates. Also related is \cite{dong2018generalized} who analyze a model with noisy observations of the expert's actions while restricting attention to  a convex optimization problem, and again obtain $\mathcal O(\sqrt{T})$ regret in terms of prediction accuracy. In a different, but related class of problems, \cite{amin2017repeated} study regret in an online problem in which the expert has access to a different type of feedback: in addition to the expert's action, the decision-maker is also provided a binary feedback on the level of suboptimality of the decision which is based on the true unknown set of parameters. 

Relatedly, a class of applications in marketing is choice-based  conjoint analysis, where consumers are provided with a list of products to choose from in a questionnaire. The answers are used to discover the consumer preferences.  In this setting, \cite{Toubia_2004,Saure_Vielma_2019} show how to leverage  polyhedral or ellipsoidal methods to select questionnaires that allow one to recover the consumer utility function. In these, the decision-maker chooses the optimization problems that will be solved by the consumer through the sequence of questionnaires. In contrast, when interpreting our model in this application, the optimization problems solved by the consumer can be arbitrary and cannot be optimized by the decision-maker. Our online contextual inverse optimization formulation is also related to studies on online contextual pricing and contextual search (\cite{roth2016watch}, \cite{lobel2018multidimensional},  \cite{leme2018contextual}, \cite{cohen2016feature}, and  \cite{liu2021optimal}). While we are able to leverage some ideas from this literature such as the use of ellipsoids and projected knowledge sets, we highlight, however, that the nature of the problem studied here is substantially different from the contextual pricing/search literature due to the nature of the feedback, as there is no control in our setting on the feedback one sees. The recent studies \cite{feng2018learning} and \cite{chen2021model} study robust assortment and price optimization based on revealed preferences associated with choice data.    Additional early works that focus on learning utility functions from revealed preferences include  \cite{beigman2006learning}, \cite{zadimoghaddam2012efficiently}, \cite{balcan2014learning}, \cite{amin2015online}.

Imitation learning (\cite{osa2018algorithmic}) is also a framework where the goal is to learn from optimal actions (or demonstrations). When inverse reinforcement learning algorithms are used, the goal is to discover the reward function associated with a Markov decision process (MDP) that describes the problem of interest. In this stream, two important approaches are feature-matching (\cite{abbeel2004apprenticeship}) and the maximum entropy method (\cite{ziebart2008maximum}). In both frameworks, the reward function is described by an  unknown linear combination of a vector of known features. The goal is to learn the weights of the reward function. In \cite{ratliff2006maximum}, the authors use the Maximum Margin Planning framework, also commonly used in structured prediction problems  that we describe below, for the online problem that consists of solving a sequence of MDPs. The authors apply algorithmic techniques from the online convex optimization literature to obtain a $\mathcal O(\sqrt T)$ regret in terms of the prediction error. There, at each time $t$, the decision-maker faces a different MDP. We refer to \cite{arora2021survey} for a review in algorithmic approaches to inverse reinforcement learning, and to \cite{BBS2021} for a recent application in the context of learning what drives decisions of high performance workers in order to  train and increase performance of new or less experienced workers.% with great financial benefit for firms .

Finally, our work  also relates to problems in structured prediction (\cite{taskar2005learning,nowozin2011structured,osokin2017structured}), where we observe a set of covariates denoted as the input and another set of covariates denoted as the output, and the goal is to discover a mapping from the inputs to the outputs. Several approaches for solving structured prediction problems assume an underlying optimization problem parametrized by the inputs, a stream of methods usually referred to as ``the argmax formulation.'' The objective function to be optimized can be seen as a score function, and the solution that maximizes the score is given by the best possible prediction for the output.  An early and influential framework to solve the argmax formulation is that of Maximum Margin Planning by \cite{taskar2005learning}. The Conditional Random Field formulation for structured prediction introduced in \cite{sutton2006introduction} can also be seen as an argmax formulation where the score function is the likelihood function of the data observed.
 
\section{Problem formulation}\label{sec:pb-form}

A contextual optimization problem will be defined by two objects: the first is a feasible set of actions $\Fset$, which is a compact subset of an Euclidean space $\mathbb{R}^n$, and the second is a context function $f \in \mathcal F$, where $\mathcal F$ denotes the set of  $L$-Lipschitz continuous functions from $\mathbb R^n$ into  $\mathbb{R}^d$. As a normalization, we restrict our analysis to the case $\Fset \in \mathcal{B}$, where $\mathcal{B}$ denotes the set of all compact subsets of $\mathbb{R}^n$ with diameter at most 1, i.e., $\sup_{x_1,x_2 \in {{\cal X}}}\|x_1-x_2\| \leq 1$ (the operator $\|\cdot\|$ refers to the Euclidean norm). Similarly, we also assume the Lipschitz constant $L$ from the definition of the set of functions $\mathcal F$ is 1. 

 For a given cost vector $c$, feasible set $\Fset$ and context function $f$, we introduce the following problem:
\begin{equation}\label{eq:forward_problem}
 \psi(c,{\cal X},f) = \argmin_{x \in {\cal X}}   f(x)'c.
\end{equation}
We will refer to Problem \eqref{eq:forward_problem} as the \emph{forward (contextual) problem}.  We do not assume that the  problem in Eq. \eqref{eq:forward_problem} satisfies any additional structure such as linearity or convexity, neither that its solution be unique.\footnote{The formulation we consider is fairly general and encompasses various prototypical problems. We discuss some examples in  \Cref{sec:appl}.}

The problem we will face in  both the offline and online versions is one in which the underlying cost vector $c^\star$ will be fixed but \emph{unknown} across problems. In turn, the decision-maker will be trying to optimize an object related to the forward problem (or multiple instances of such a problem in the online setting) but without knowledge of $c^\star$. Information about $c^\star$ will come in the form of past optimal actions that would have been chosen by an expert with knowledge of $c^\star$. We will assume that the vector of unknown parameters  $c^\star$ has dimension  $d \geq 2$ as the case $d=1$ case is trivial. 

Throughout the paper, to lighten notation,  we make  the assumption that $c^\star$ lives in a $d$-dimensional sphere with radius 1, which we denote by $S^d$. However, we note that the analysis we develop does not rely on this assumption and actually leads to performance bounds that can be interpreted as being per unit of the true cost vector norm. We comment further on this point following Theorem \ref{theorem:regret_general_case}.

\subsection{Offline setting}

In the offline setting, we assume that the decision-maker has access to an offline data set ${\cal D}$ consisting of $N$ past problem instances and associated optimal decisions: $\mathcal D = \{\Fset_i,f_i,x^\star_i\}_{i = 1,\cdots, N}$. That is, for each observation $i = 1,\cdots, N$, we are given a set of feasible actions $\Fset_i \in \mathcal B$, a context function $f_i \in \mathcal F$, and an optimal action $x^\star_i \in \psi(c^\star, \Fset_i,f_i)$.

Given a data set $\mathcal D$, the decision-maker can restrict the set of cost vectors to a subset that is consistent with the optimal actions. In particular, we define $C(\mathcal D)$ to be the  set of cost vectors consistent with $\mathcal D$:
\begin{align}\label{eq:knowledge_set_offline}
C(\mathcal D) = \{c \in S^{d}:    c'f_i(x^{\star}_i) \leq c'f_i(x), \;\forall \;x\in {\cal X}_i, \; i = 1,\cdots, N\}.
\end{align}

The decision-maker selects a mapping $\pi$ from the data and the current problem instance $({\cal X},f)$ that materializes, into a feasible action $x^\pi$. We let $\mathcal P$ denote the set of all such mappings. We will measure the performance of a policy through regret, which is the difference between  the decision-maker's loss and the loss that could have been achieved with knowledge of $c^\star$:
\begin{eqnarray}
\mathcal R^\pi\left(c^\star,{\cal X}, f\right) &=&  \bigl(f(x^{\pi})-f(x^\star)\Bigr)'c^\star. \label{def:regret-offline}
\end{eqnarray}

In particular, the decision-maker aims to minimize the worst-case regret:% and we define the minimax regret, given offline data ${\cal D}$ as
\begin{equation}\label{eq:wcr-offline}
%{\cal R}^*({\cal D}) = \inf_{\pi \in {\cal P}} \; 
{\cal WCR}^{\pi}({\cal D}) = \sup_{c^\star \in C(\mathcal D),~\Fset \in \mathcal{B}, ~f \in \mathcal{F}}  \mathcal R^\pi(c^\star, \Fset, f).
\end{equation}

Note we are interested in providing guarantees for a strong type of risk, which is the worst-case loss,
where the worst-case is taken over any possible new optimization problem instance that can materialize
out-of-sample\footnote{While we do not pursue these here, other risk functions could also be considered, for instance, by assuming a
distribution for the instance of optimization problems and computing the expected suboptimallity gap instead of the worstcase.}.

\subsection{Online setting}

In the online setting, the time horizon is denoted by $T$. At  every period $t \in \{1,...,T\}$, the feasible actions set $\Fset_t \in \mathcal B$ and the context function $f_t \in \mathcal F$ are revealed, and the decision-maker needs to select an action $x_t$ in $\Fset_t$. Upon selecting an action, the decision-maker incurs a cost given by  $f_t(x_t)'{c^\star}$, where  $c^\star \in \mathbb{R}^d$ is unknown to the decision-maker. In particular, the incurred cost $f_t(x_t)'{c^\star}$ is never revealed to the decision-maker. At the end of the period, the decision-maker observes an  optimal oracle action $x_t^\star \in \Fset_t$. While the optimal action $x_t^\star$ is revealed too late in period $t$ to be useful in that period, it potentially allows the decision-maker to make better decisions in periods $t+1$ onwards since it contains information about the cost vector ${c^\star}$.

We assume that initially the decision-maker knows only that  $c^\star$ belongs to some  initial knowledge set $C_0$. We assume that $C_0 \subset \mathbb R^d$ is a closed subset of $S^d$, the  sphere with unit radius. We let  $\mathcal{I}_t$ denote the information that is available when making a decision in period $t$. In particular, ${\cal I}_1 = \{\Fset_{1}, f_{1}\}$, and for all $t\ge 1$, 
$\mathcal{I}_{t+1} = \left\{(\Fset_s, f_s, x_s^\pi, x^\star_s): 1\le s \le t \right\}~\cup~\{\Fset_{t+1}, f_{t+1}\},
$ where $x_s^\pi$ will be defined shortly.
We naturally focus on non-anticipatory policies (i.e., policies such that the action in period $t$ is measurable with respect to the history $\mathcal{I}_{t}$). With some abuse of notation from the offline setting, we let ${\mathcal P}$ denote this set of policies, and for any policy $\pi \in \mathcal P$, we denote by $x_t^{\pi}$ the action it prescribes in period $t$. Given the information collected up to time $t$, the decision-maker can restrict the set of cost vectors to a subset that is consistent with the optimal actions observed. In particular, we define $C(\mathcal{I}_t)$ to be the set of cost vectors consistent with ${\mathcal I}_t$:
\begin{align}\label{eq:knowledge_set}
C(\mathcal{I}_t) = \{c \in C_0:    c'f_s(x^{\star}_s) \leq c'f_s(x), \;\forall \;x\in {\cal X}_s, \; s = 1,\cdots, t-1\}.
\end{align}

We will measure the performance of a policy through regret over $T$ periods, which is the difference between the cumulative performance achieved and the cumulative performance that could have been achieved with knowledge of $c^\star$. For  notational simplicity, we define $\vec{\mathcal X}_T = (\Fset_1,...,\Fset_T)$ and $\vec{f}_T = (f_1,...f_T)$. 
 For a given policy $\pi \in {\cal P}$, we define the regret over $T$ time periods as:
\begin{eqnarray}
{\cal R}^{\pi}_T\left(c^\star, \vec{\cal X}_T,\vec f_T\right) &=& \sum_{t=1}^T \bigl(f_t(x_t^{\pi})-f_t(x^\star_t)\Bigr)'c^\star. \label{def:regret}
\end{eqnarray}

The decision-maker's problem is to choose a policy $\pi \in \mathcal P$ in order to minimize its cumulative regret assuming the cost vector $c^\star$, the feasible sets $\vec{\mathcal X}_T$ and the context functions $\vec{f}_T$ are chosen adversarially by nature. In other words, we want to analyze the worst-case regret for a policy $\pi$:
\begin{eqnarray}\label{eq:multi-period-goal}
{\cal WCR}^{\pi}_T(C_0) = \;\sup_{c^\star \in C_0,~\vec{\mathcal X}_T \in \mathcal{B}^T ,~\vec{f}_T \in \mathcal{F}^T}{\cal R}^{\pi}_T\left(c^\star,\vec{\cal X}_T,\vec f_T\right).
\end{eqnarray}

With a slight abuse of notation, we represent the worst-case regret in the online as a function of the initial knowledge set, in contrast with Eq. \eqref{eq:wcr-offline}, where it is a function of the offline data. 

\textbf{Remark on the objectives:} Some previous studies focused on alternative loss functions. In the offline setting, \citep{esfahani2018data, barmann2017emulating} consider the loss defined as $\bigl(f(x^{\star})-f(x^\pi)\Bigr)'c^\pi$, which is the suboptimality gap of the expert with respect to the decision-maker under the decision-maker guess of cost vector. This loss, sometimes referred to as objective loss, is observable and enjoys nice properties such as convexity. We highlight here that, in contrast, in the analysis of both the offline and the online setting, we analyze the regret in terms of suboptimality gap of the decision-maker's action with respect to the expert's actions under the \textit{true cost vector} $c^\star$. The regret we study admits a direct interpretation for inverse optimization problems: it quantifies precisely the suboptimality gap of the decision-maker choices and isolates the ``cost of not knowing $c^\star$.''

In the online setting, \cite{barmann2017emulating,barmann2018online} bound the cumulative losses associated with per period loss given by $\bigl(f(x^{\star})-f(x^\pi)\Bigr)'(c^\pi-c^\star)$, which they later use to establish bounds on the same loss that we work with, namely $\bigl(f(x^\pi) - f(x^{\star})\Bigr)'c^\star$. In the online setting, a contribution of our work is to show that it is possible to analyze the cumulative regret associated with loss $\bigl(f(x^\pi) - f(x^{\star})\Bigr)'c^\star$ directly, and, achieve tighter upper bounds for this quantity.
\section{The Offline Problem}\label{sec:one-period}

In this section, we focus on the offline problem: we start from a data set $\mathcal D$ and aim to find a policy $\pi \in \mathcal P$ to minimize ${\cal WCR}^{\pi}({\cal D})$, as defined in Eq. \eqref{eq:wcr-offline}. To that end, we start by defining a subset of policies whose performance is more amenable to analysis, the class of proxy policies (cf. \Cref{lemma:01}). This result motivates us to define two important geometric objects, the uncertainty angle and the circumcenter of a set (cf. \Cref{def:set_angle}) that impact the worst-case performance of proxy policies. Finally, we establish that an appropriate proxy policy based on the circumcenter achieves the best uniform performance over problem classes (cf. \Cref{theorem:one_period_problem}) and the worst-case performance is given by the uncertainty angle.

%%%%%%%%%%%%%%%%%%%%%%%%%%%%%%
%%%%%%%%%%%%%%%%%%%%%%%%%%%%%%

\subsection{Proxy Policies}\label{app:proxy}

We begin our analysis by defining a class of policies that will play a central role in our paper.

 \begin{definition}[Proxy policies]\label{def:proxy} We say $\pi \in \mathcal P$ is a \emph{proxy policy} if the action selected by the policy $\pi$, $x^\pi$,  is consistent with some cost vector $c^\pi$. In particular, let $\psi$ be as defined in Eq. \eqref{eq:forward_problem} and let $\mathcal P'$ be the set of proxy policies defined as follows 
\[{\cal P}'=\: \Bigl\{ \pi \in {\cal P} :\mbox{ there exists some } c^\pi \in S^d \mbox{ such that } x^{\pi}  \in   \psi (c^{\pi},{\cal X},f) \mbox{ for all } \mathcal X \in \mathcal B \mbox{ and } f \in \mathcal F\Bigr\}. \]
\end{definition}

 For an illustrative example of a policy in $\cal P'$, suppose that $\psi(c,\Fset,f)$ corresponds to a linear program. If $\Fset$ is a polytope, then every policy that chooses an action $x^\pi$ in the interior of $\Fset$ cannot belong to $\cal P'$, since the only cost vector that makes an interior point optimal would be the origin, and the latter does  not belong to the unit sphere. In this case, a policy in ${\cal P}'$ would only induce decisions corresponding to extreme points. 
  Through proxy policies, one can think about choosing costs vectors  instead of actions of possibly complicated instances $\Fset$ chosen by nature. This idea appears in the literature and was used in problems such as structured prediction  and  online inverse optimization (see, for instance,  \cite{taskar2005learning} and \cite{barmann2017emulating}).

Because a proxy policy is defined in terms of a cost vector, we will, with a slight abuse of notation, refer to a proxy policy by its cost vector $c^\pi$, and use $x^\pi$ to represent an action implied by $c^\pi$. This representation creates a bit of ambiguity with respect to the value of $x^\pi$ when the forward problem $\psi(c^{\pi},{\cal X},f)$ has multiple optimal solutions. To address this ambiguity, we will consider the worst-case choice of $x^\pi$ when introducing our loss function for proxy policy. Recall that $c^\star$ is the true cost vector and $x^\star$ is an optimal solution of the forward problem given $c^\star$. Then, the worst-case  regret for a policy $c^\pi$ given and a true cost vector $c^\star$ is given by:
\begin{eqnarray}\label{pb-0} 
 \mathcal{L}(c^\pi, c^\star) &=& \sup_{{x}^\pi \in \psi(c^\pi,{\cal X},f),\;\mathcal X \in \mathcal B,\; f \in \mathcal F }  ~ \Bigl(f(x^\pi)-f(x^\star)\Bigr)'c^\star.
 \end{eqnarray}

We define $\theta(\cdot,\cdot)$ to be the angle between two vectors, i.e., $\theta(c,\hat c) = \arccos c'\hat c/\|c\|\|\hat c\|$, with $\theta(c,0)$ defined to be zero. Our next result shows that we can bound the worst-case regret loss under policy $c^\pi$ using the angle between the vectors $c^\pi$ and $c^\star$. This bound is tight in the sense that there exists $\mathcal X$ and $f$ for which this bound holds with equality.

  \begin{lemma}[Realized worst-case regret]\label{lemma:01} Let $c^{\pi}$ be a policy in $\mathcal P'$ and $c^\star$ be a cost in $S^d$. Then,
    \begin{eqnarray*}
       \mathcal{L}({c}^\pi, c^\star)
       &=& \begin{cases}
      \sin \theta(c^\star,{c}^\pi) & \mbox{if }\theta(c^\star,{c}^\pi) < \pi/2, \\
      1 & \mbox{otherwise}.
   \end{cases}
    \end{eqnarray*}
    \end{lemma}
    
This lemma proves that if the angle between $c^\pi$ and $c^\star$ is small, then the regret loss due to policy $c^\pi$ is small. If the angle is $\pi/2$ or larger, then we risk incurring the maximal regret of 1. The proof of the lemma is based on, first constructing a semi-definite relaxation of the optimization problem  from Eq. \eqref{pb-0}, then explicitly solving it, and finally constructing an instance to show that the relaxation is tight. 

 \Cref{lemma:01} yields significant intuition: if ones wants to minimize the worst-case regret for policies in $\cal P'$, then one should select the proxy cost vector $c^\pi$ that minimizes the worst-case angle with respect to the unknown $c^\star$. We explore this idea in the next subsection.

%%%%%%%%%%%%%%%%%%%%%%%%%%%%%%
%%%%%%%%%%%%%%%%%%%%%%%%%%%%%%

\subsection{The Uncertainty Angle and the Circumcenter} \label{sec:geom_objects} 

We now define two geometric objects that will be key to our analysis: the uncertainty angle and the circumcenter of a set. 

\begin{definition}[Uncertainty angle and circumcenter]\label{def:set_angle}
Let $C$ be a nonempty set. We define the \emph{uncertainty angle} of $C$, $\alpha(C)$, as
\begin{equation}\label{eq:def-alpha}
\alpha(C) = \inf_{\hat{c} \in S^d}  \sup_{c \in C} \theta (c,\hat{c}).
\end{equation}
If the infimum is attained, we call the minimizer of the equation above the \emph{circumcenter} $\hat c(C)$ of set $C$.
\end{definition}

To understand the notion of uncertainty angle, it is useful to think in terms of revolution cones that contain the set $C$. A revolution cone $ K(\hat c,\gamma)$ in $\mathbb R^d$ is defined by two objects, an axis $\hat c \in S^d$ and an aperture angle $\gamma \in [0,\pi/2)$, and consists of all points that are within an angle $\gamma$ of $\hat c$ (see  \Cref{fig:uncertainty_angle}). If we assume that $C$ is contained within a revolution cone and that $K(\hat c, \gamma)$ is the revolution cone with the smallest aperture angle containing $C$, then $\alpha(C) = \gamma$. In this case, the vector $\hat c$ is the circumcenter of the set $C$. In other words, the circumcenter $\hat c(C)$ of a set $C$ is the axis of the revolution cone containing it that has the smallest aperture angle. Our definition of circumcenter is a slight generalization of the one used in \cite{henrion2010inradius}, which defines the notion of the circumcenter of a cone (our definition applies to an arbitrary nonempty set $C$). If $C$ is not contained in a revolution cone, $\alpha(C)$ could take values from $\pi/2$, if $C$ is a half-space, to $\pi$, if $C$ is the full Euclidean space. When $\alpha(C) \geq \pi/2$, we will continue to refer to a vector $\hat c(C)$ that minimizes Eq. \eqref{eq:def-alpha} as the circumcenter of $C$, even though  $C$ is not contained in a revolution cone.

\begin{figure}[ht]
    \centering
    \includegraphics[scale = .4]{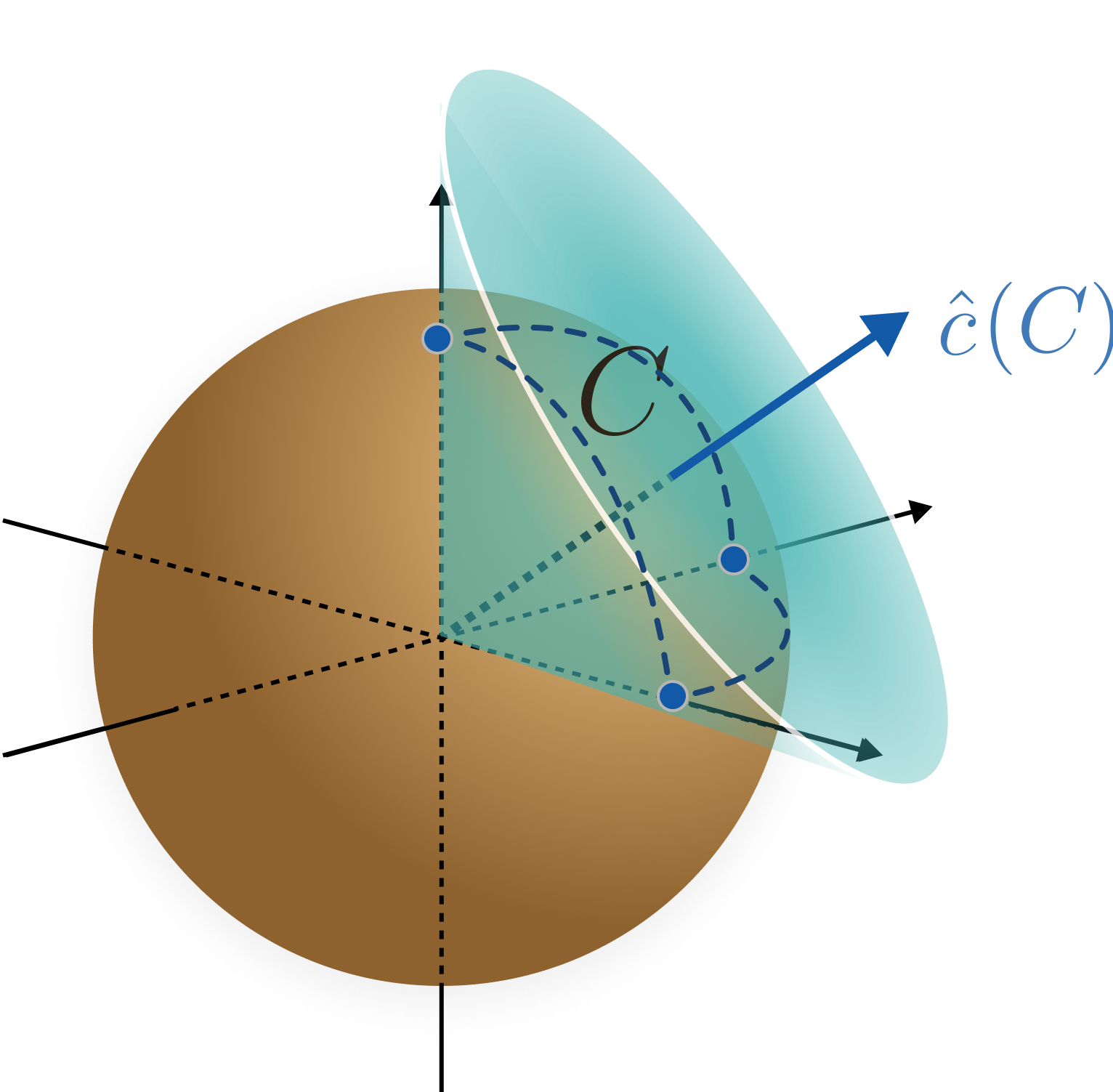}
    \caption{Set $C$ (intersection of the nonnegative orthant with the unit sphere) and the smallest aperture-angle revolution cone (with axis $\hat c(C)$) containing it. We increased the norm of $\hat c(C)$ for illustrative purposes.}
    \label{fig:uncertainty_angle}
\end{figure}

\begin{lemma}[Existence and uniqueness of circumcenter] Let $C$ be a nonempty set. There always exists a circumcenter $\hat c(C) \in S^d$ that minimizes the right-hand side of Eq. \eqref{eq:def-alpha}. Furthermore, if $C \setminus \{0\}$ is nonempty and $\alpha(C) < \pi/2$,  then the circumcenter is unique.
\label{lemma:existence_circuncenter}
\end{lemma}

\subsection{The Circumcenter Policy and the Offline Minimax Regret} \label{sec:theorem_one}

The circumcenter of a set $C$ is the point that minimizes the angle of other points in $C$ with respect to it. Therefore, a potentially good policy for the offline problem is to treat the circumcenter of $C$ as a proxy cost. That is, we choose the action $x^\pi \in \mathcal X$ by solving Eq. \eqref{eq:forward_problem} with  $c = \hat c(C)$. We call this policy the \emph{circumcenter policy} and use the representation $c^\pi = \hat c(C)$ to denote the proxy cost associated with it.

\begin{theorem}[Offline minimax regret]\label{theorem:one_period_problem}
The worst-case regret of the circumcenter policy $\pi$ is upper bounded as follows
	$$
	 {\cal WCR}^\pi({\cal D}) \: \le \:  
     \mathbf{1}\{\alpha(C({\cal D})) < \pi/2\} \: \sin \alpha(C({\cal D})) \: +\: \mathbf{1}\{\alpha(C({\cal D})) \ge \pi/2\}.
   $$

Furthermore, for any angle $\bar \alpha \in [0,\pi]$, there exists an information set $C$ such that $\alpha(C) = \bar \alpha$ and no policy $\pi \in {\cal P}$ can achieve worst-case regret lower than $\sin (\bar \alpha)$ (respectively $1$)  if $\bar \alpha < \pi/2$ (respectively $\bar \alpha \ge \pi/2$).  
\end{theorem}

\Cref{theorem:one_period_problem} characterizes the minimax regret as a function of the  uncertainty angle $\alpha(C)$. When $ \alpha(C) < \pi/2$,  the best uniform (over all knowledge sets) regret is equal to $\sin \alpha(C)$. If $\alpha(C) \geq \pi/2$, then nature may be  able to cause maximal regret (the maximal regret is 1 because we assumed that $c^\star$ has norm 1, the set $\mathcal X$ has diameter at most 1, and the context function $f$ is 1-Lipschitz continuous). In particular, for the circumcenter policy, the minimax regret for any offline data $\mathcal D$ is an instance-dependent bound given by $\sin \alpha(C(\mathcal D))$ if $\sin \alpha(C(\mathcal D)) < \pi/2$, or $1$ otherwise. This result highlights the driver of regret in offline problems and will be a crucial stepping stone in the online setting.

\section{The Online Setting: Initial Analysis and Intuition}\label{sec:circumcenter_policy}

The previous section constructs a policy called the circumcenter policy for the offline problem that offers robust performance guarantees. This section presents a set of initial results in the online setting that highlight the intricate interplay between learning and regret in this class of problems.

The offline learning problem is equivalent to the online learning problem with $T=1$, so repeatedly applying the circumcenter policy will correspond to a greedy policy in the online setting.  We establish that such a policy can fail (cf. \Cref{lemma:insuff_circumcenter}), even in ``benign'' cases. This will highlight that in this class of problems it possible have poor regret while also not collecting any meaningful new information on the unknown cost vector.  To build ideas gradually, we first zoom in on the nature of information collection in this class of problems in \Cref{sec:information} given the type of feedback collected. Focusing on the case when the initial knowledge set $C_0$ satisfies $\alpha(C_0) < \pi/2$, we then introduce in \Cref{sec:ellip_cones} an approach that considers supersets of the knowledge sets (in the form of ellipsoidal cones) and an approach to update these. These supersets ensure that useful information is collected over time even if one  uses a greedy approach (choosing the circumcenter of the superset as a proxy cost), which will automatically balance the trade-off between instantaneous performance and information gain. In \Cref{sec:theorem_2}, we propose an algorithm for the pointed case and derive an upper bound on its performance in \Cref{sec:theorem_2}. We show that when the initial knowledge set is pointed, the \texttt{EllipsoidalCones} algorithm  ensures logarithmic regret in the time horizon.

\subsection{The Greedy Circumcenter Policy}\label{app:circumcenter}

\Cref{theorem:one_period_problem} implies that the circumcenter policy is optimal in a worst-case sense for the one-period problem. A natural candidate policy for the multi-period problem is therefore to simply use the circumcenter policy in each period, a policy we call the \emph{greedy circumcenter policy}. Formally,
\[\pi_{greedy} = \: \Bigl\{ x^\pi_t \in \psi(\hat c(C(\mathcal I_t)),\Fset_t,f_t)), \; \mbox{for every }t \leq T\Bigr\}, \label{def:greedy}
\]
where, $
C(\mathcal{I}_t) = \{c \in C_0:    c'f_s(x^{\star}_s) \leq c'f_s(x), \;\forall \;x\in {\cal X}_s, \; s = 1,\cdots, t-1\}$. This policy might seem appealing due to its simplicity. However, the next proposition shows that  the worst-case regret incurred by the greedy circumcenter policy is linear in the time horizon, even if the initial knowledge set is pointed.

\begin{theorem}[Insufficiency of the greedy circumcenter policy]\label{lemma:insuff_circumcenter}
 There exists knowledge sets $C$ such that the worst-case regret of the policy $\pi_{greedy}$ satisfies
	$$
    {\cal WCR}^{\pi_{greedy}}_T(C) = \Omega(T). 
	$$
\end{theorem}

The proof of this proposition provides insights into the limitation of the greedy circumcenter policy for the multi-period problem. We detail here the key drivers of this linear regret through an example. Consider the example presented in Figure \ref{fig:inverse_exploration_fail} with initial knowledge set $C_0$.  The associated circumcenter  $\hat{c}(C_0)$ is on the boundary of the knowledge set (See Figure \ref{fig:inverse_exploration_fail}(a)). %\ref{fig:greedy_initial_set} and  \ref{fig:greedy_circumcenter}). % (see \Cref{fig:set_C0} in  \Cref{example:01}). 
In such a case, nature can construct adversarial instances $\{\Fset_t,f_t\}$ such that the decision-maker will not be able to update meaningfully the knowledge set and, \textit{at the same time}, incur positive regret. Indeed, suppose that nature selects the identity as the context function and a feasible set with two actions, $\mathcal X_1 = \{x_1,0\}$, where $x_1$ is such that $ -\epsilon \leq \hat{c}(C_0)'x_1 < 0$ for a small and positive $\epsilon$, so that $x_1$ is strictly better than $\{0\}$ but with a small margin (Figure \ref{fig:inverse_exploration_fail}(b)). Then the circumcenter policy would prescribe to select $x_1$. Suppose that $c^\star$ is as depicted in the figure. Then $x_1^\star=0$ and the updated information after observing $x_1^\star$ would be minimal. Indeed, we would have $C({\cal I}_2)= C_0 \cap \{c ' x_1 \ge 0\}$, which would almost coincide with $C_0$. We depict the updated set in  Figure \ref{fig:inverse_exploration_fail}(c).  %(See Figure \ref{fig:greedy_initial_instance} and \Cref{fig:greedy_update}). 
If nature repeatedly uses perturbations of such an instance,  the cumulative regret will be linear.

\begin{figure}[ht]
  \begin{subfigure}{0.33\textwidth}\label{fig:greedy_circumcenter}
    \includegraphics[width=\linewidth]{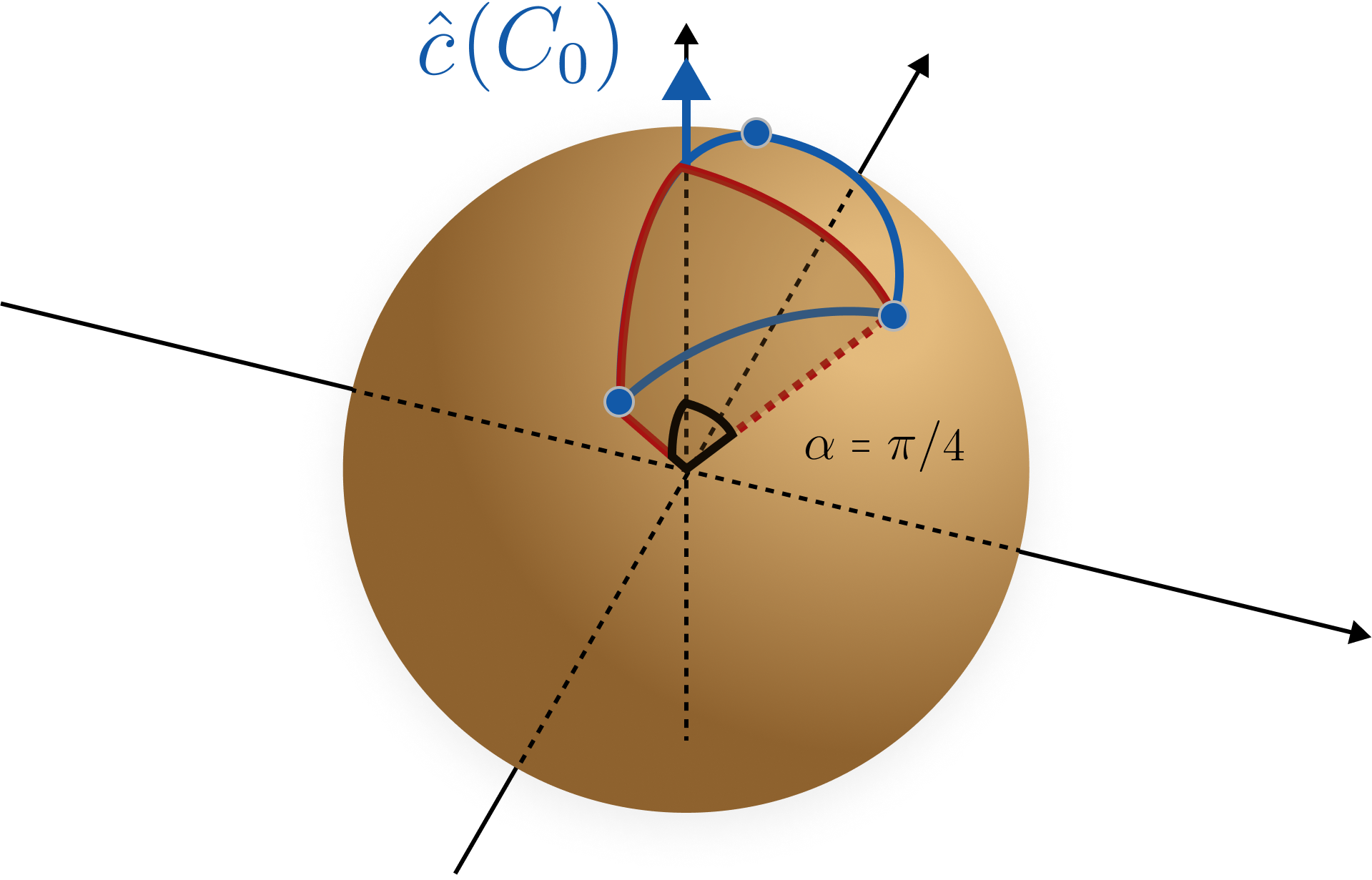}
    \caption{}
  \end{subfigure}%
  \hspace*{\fill}   
    \begin{subfigure}{0.33\textwidth}\label{fig:greedy_initial_instance}
    \includegraphics[width=\linewidth]{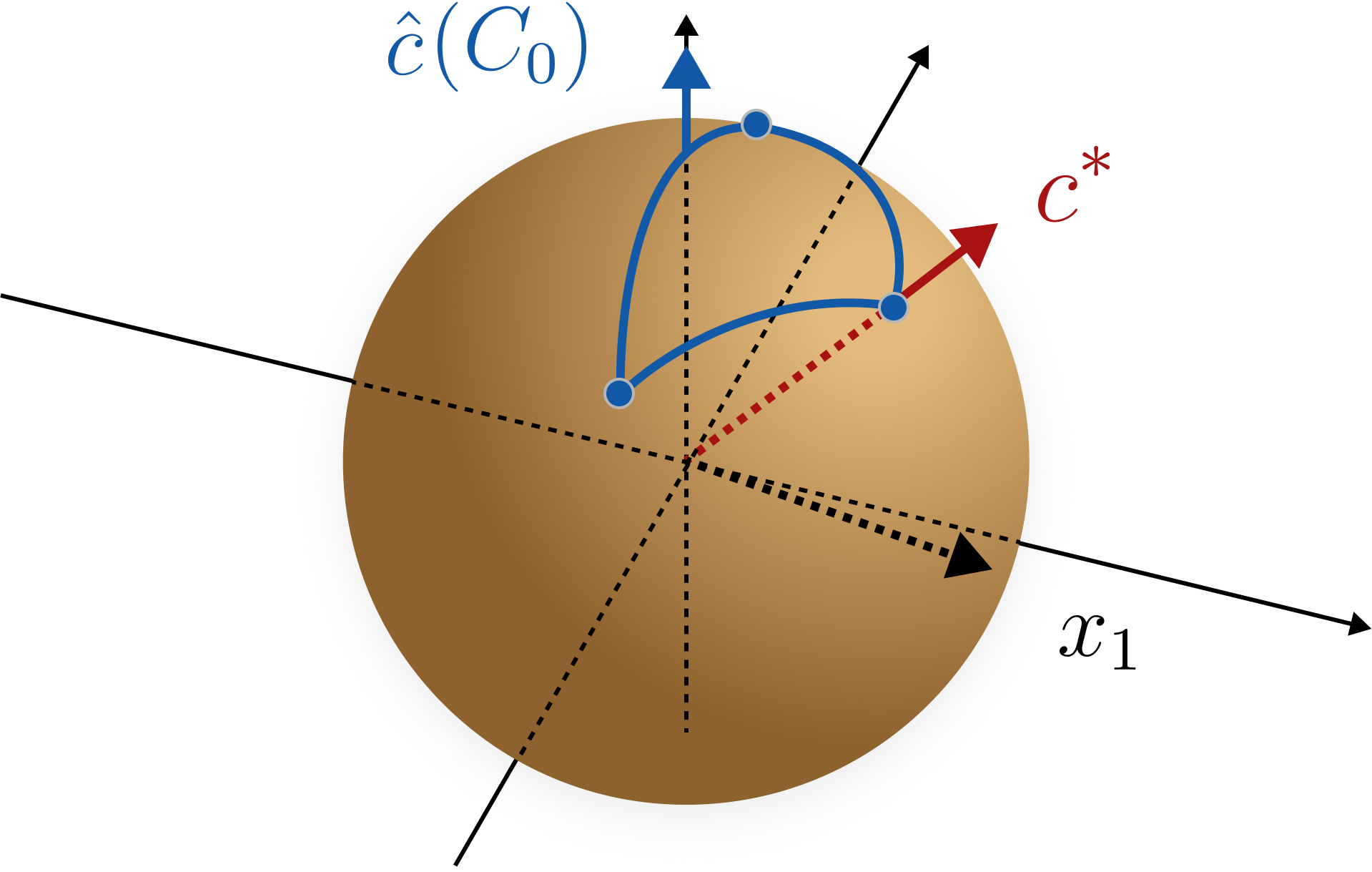}
    \caption{}
  \end{subfigure}%
  \hspace*{\fill}   % maximize separation between the subfigures
  \begin{subfigure}{0.33\textwidth}\label{fig:greedy_update}
    \includegraphics[width=\linewidth]{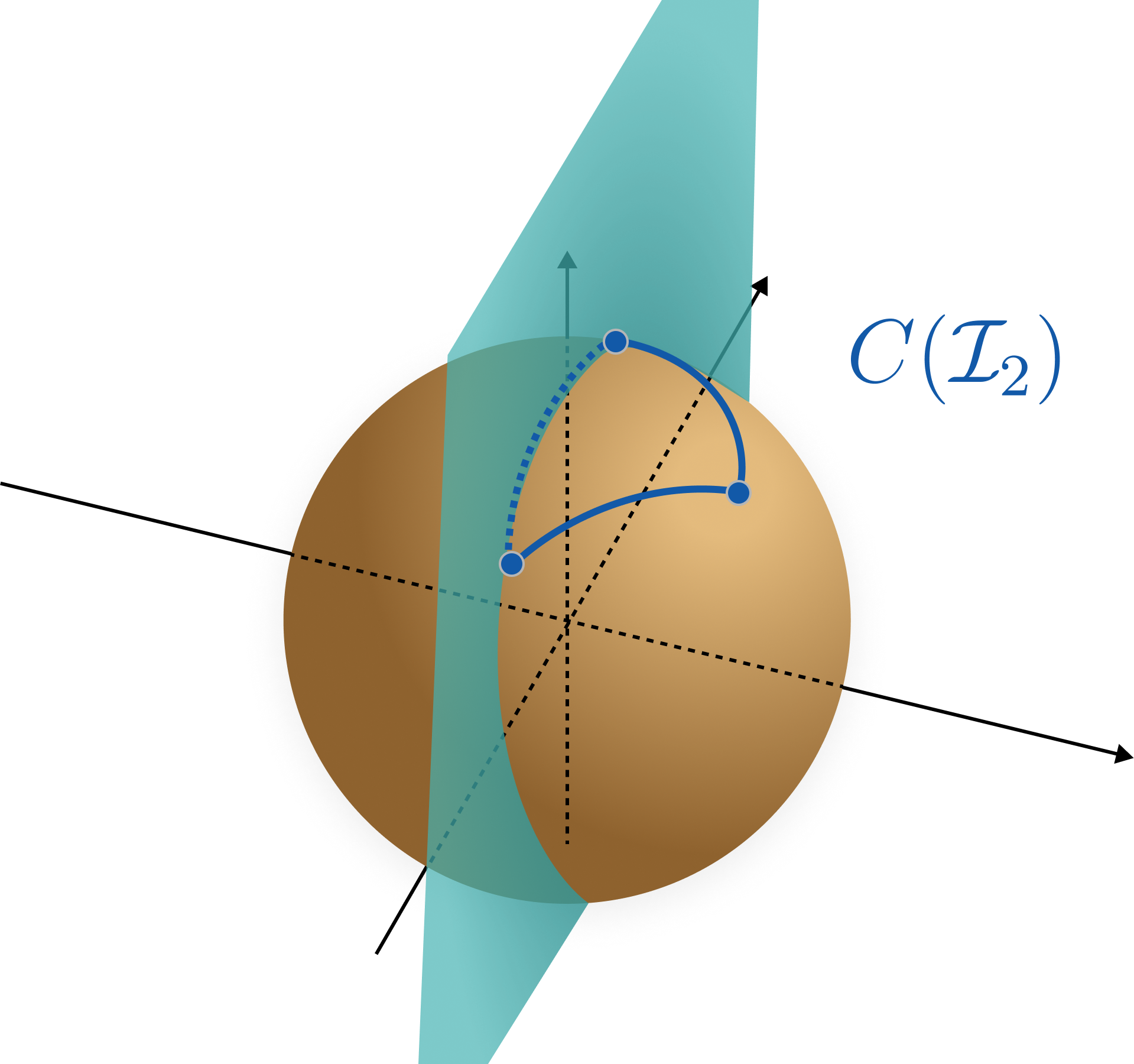}
    \caption{}
  \end{subfigure}%
  \hspace*{\fill}   
  \caption{Example where the greedy policy fails. In $(a)$, we have the initial knowledge set $C_0$ where the circumcenter lies in the border of $C_0$ and achieves angle equal to $\pi/4$ with the extreme points. In $(b)$, the optimization instance $f_1$ as the identity, $\mathcal X_1 = \{x_1,0\}$ and the true cost vector. In $(c)$ the updated set with minimal information collection despite the high regret.}
  \label{fig:inverse_exploration_fail}
\end{figure}

 Note that the knowledge set at time $t+1$ can be represented by the intersection of the knowledge set at time $t$ with all of the halfspaces that characterize the optimality of $x_t^\star$, i.e.,
\begin{equation*}
C(\mathcal{I}_{t+1}) = C(\mathcal{I}_t) \cap \{c \in \mathbb{R}^{d}:    f_t(x^{\star}_t)'c \leq f_t(x)'c, \;\forall \;x\in {\cal X}_t\}.
\end{equation*}
That is, we update the knowledge set by intersecting it with a collection of half-spaces. Thus, in general, the risk of the circumcenter being on the boundary of the knowledge set is a significant one.

At a higher level, the greedy circumcenter policy suffers because it does not introduce sufficient tension for nature's problem. When nature tries to counter the policy, it is able to  both inflict high regret and limit information collection for the decision-maker. This highlights that, in order to counter nature, the decision-maker needs to design a policy that \textit{induces} nature to trade-off instantaneous regret performance and the prevention of information collection. This would allow the decision-maker to accumulate information over time. As the decision-maker can only learn indirectly by forcing nature to reveal information about the nature of the costs, we refer to this phenomenon as \emph{inverse exploration}. \Cref{lemma:insuff_circumcenter} is not overly conservative in the sense that the instances used to prove the result are simple sequences of linear objective functions with two actions available. In fact, the main driver of the result is the dynamics of information collection and how the knowledge set evolve over time. Therefore, in the next two subsections, we describe this dynamic in detail and how to avoid the pitfalls of the greedy circumcenter policy presented in \Cref{lemma:insuff_circumcenter}.

\subsection{Information Collection and Knowledge Set Update}\label{sec:information}

At each period, after observing the expert's action $x_t^\star$, the knowledge set update corresponds to adding  constraints $(f_t(x)-f_t(x^\star))'c \geq 0$ for each $x \in \mathcal X_t$. Let us define: 
\[\delta_t(x) = \frac{f_t(x)-f_t(x^\star_t)}{\|f_t(x)-f_t(x^\star_t)\|},\]
with $\delta_t(x) = 0$ if $f_t(x)-f_t(x^\star_t) = 0$. Then, the full update constraints can be rewritten as $\delta_t(x)'c \geq 0$ for all $x \in \mathcal X_t$. This is a fairly complicated procedure since the object $\{c \in \mathbb{R}^{d}:    \delta_t(x) \geq 0, \;\forall \;x\in {\cal X}_t\}$ is the intersection of potentially infinitely many halfspaces. Therefore, instead of focusing on all of the halfspaces, we will focus only on one specific halfspace, a procedure we will call the relaxed update. Denote $\delta^\pi_t = \delta_t(x^\pi_t)$ as the \emph{effective difference}, which is the vector associated with the action $x_t^\pi$ chosen at period $t$. We use the name effective difference because each period regret is simply the inner product of $\delta^\pi_t$ with $c^\star$ (up to the norm of $\delta_t^\pi$, which is bounded by 1). We will define the relaxed update to be equal to \begin{equation} \label{eq:relaxed_update}C(\mathcal I_t)\cap \{c \in \mathbb R^d:\; {\delta^\pi_t}'c \geq 0\}.\end{equation} That is, in the relaxed update we include only one new constraint at time $t$, the one associated with what we call the effective difference. The effective difference satisfies two important inequalities that will be crucial in order to understand how inverse exploration manifests itself. First, ${\delta_t^\pi}'c^\star \geq 0$ since $x_t^\star$ is optimal with respect to $c^\star$. Second, for any proxy policy $c^\pi_t$,  we have that $x_t^\pi$ is optimal with respect to $c_t^\pi$, which implies that $f_t(x_t^\pi)'c^\pi_t \leq f_t(x)'c_t^\pi$ for all $x \in \Fset_t$, and thus, ${\delta_t^\pi}'c^\pi_t \leq 0$. In turn, we have the following implication on the link between  $c^\pi_t$ and $C(\mathcal I_{t+1})$ if ${\delta_t^\pi} \neq 0$: 
\begin{equation}\label{eq:feasible_feedback}
{\delta_t^\pi}'c^\pi_t \leq 0\mbox{ and }{\delta_t^\pi}'c^\star \geq 0, \mbox{ for  }c^\star \in C(\mathcal I_t) \implies c_t^\pi \notin \mbox{int}(C(\mathcal I_{t+1})). 
\end{equation} 

The two conditions in Eq. \eqref{eq:feasible_feedback} imply that if $c_t^\pi$ is ``sufficiently" in the relative interior of $C(\mathcal I_t)$, only two cases can happen: either  we are able to remove a sufficient mass of candidates when updating $C(\mathcal I_{t+1})$,  or $\delta_t^\pi = 0$, implying no regret. Therefore, there exists a tension in this problem between minimizing instantaneous regret (which is achieved via the circumcenter policy) and gaining more information (which is achieved by selecting a proxy cost sufficiently in the interior of the knowledge set). 

 \begin{remark}[nature of the feedback] Here we highlight an important and fundamental difference from the contextual search literature discussed in Subsection \ref{sec:lit}.  There, the vector $\delta_t^\pi$ (commonly denoted as the contextual information) would have been known before choosing the proxy cost $c_t^\pi$, whereas in our case, $\delta_t^\pi$ materializes only in the end of the period. As a result, here  we can only indirectly affect it by our choice of proxy cost $c_t^\pi$.
 \end{remark}

Next, we show that if one starts with an uncertainty set that lives in a pointed cone, then, it is possible to conduct inverse exploration and ``force'' nature to trade-off between inflicting regret and limiting information collection. Our strategy is to replace the knowledge sets in our algorithm with  regularized supersets that contain $C(\mathcal I_t)$. These supersets will ensure that useful information is collected over time despite us using a greedy approach (choosing the  circumcenter of the superset as a proxy cost), which will automatically balance the trade-off between instantaneous performance and information gain.

\subsection{Ellipsoidal Cones}\label{sec:ellip_cones}

We replace the knowledge sets with ellipsoidal cones, which are better-behaved objects and guarantee the circumcenter is always in the interior of the superset. These supersets will allow us to leverage the powerful machinery from the ellipsoid method \citep{khachiyan1979polynomial} in order to update the knowledge set every time that we collect a new effective difference. 

In \Cref{fig:cone_fit}, $E$ is an ellipsoidal cone that contains the initial knowledge set that was used in \Cref{fig:inverse_exploration_fail}. $C_0$ is a ``bad" instance that can lead to a linear regret with the greedy circumcenter policy since its circumcenter is in the border of the set. $\hat c(E)$ is the circumcenter of the ellipsoidal cone that contains $C_0$ and is ``sufficiently" close to the original circumcenter but also sufficiently in the relative interior of $C_0$ to ensure inverse exploration. 

\begin{figure}[ht]
    \centering
    \includegraphics[scale = .25]{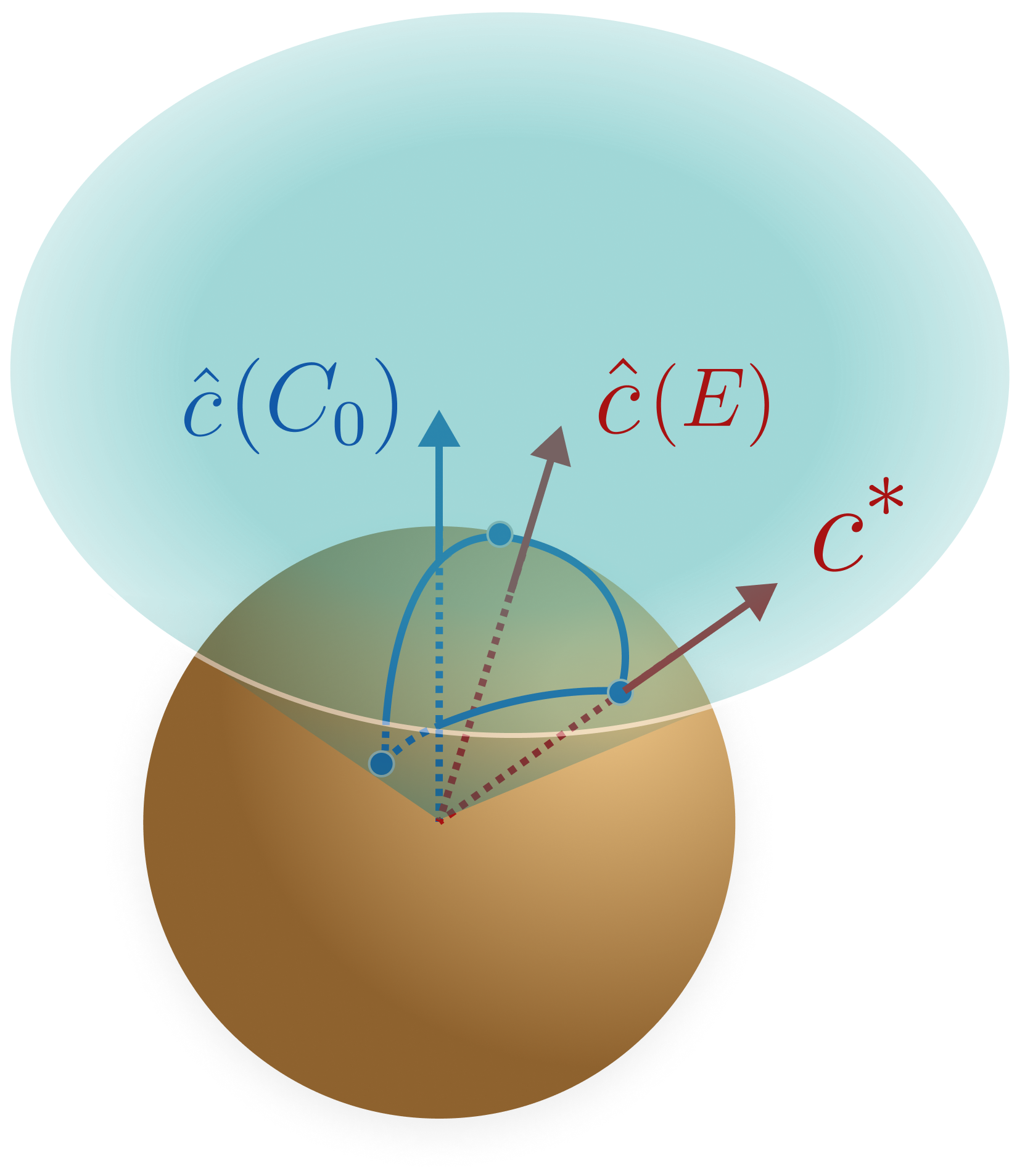}
    \caption{Example of an ellipsoidal cone $E$ that contains an initial knowledge set $C_0$, including the circumcenter of $C_0$, $E$ and a true cost candidate. The vectors were rescaled for illustrative purposes.}
    \label{fig:cone_fit}
\end{figure}

We will need notation to refer to specific components of a vector $c \in \mathbb{R}^d$. We use $c_{[i:j]}$ to denote the entries $i,...,j$ of the vector $c$. Before we present the general definition of an ellipsoidal cone, we introduce the notion of a standard-position ellipsoidal cone, which is one where the axis of the cone is the canonical vector $e_1$, and the eigenvectors of the generating ellipsoid are the canonical vectors $e_2$,...,$e_d$. Let $\mathbb D^d_{++}$ be the set of $d$-dimensional positive-definite diagonal matrices. 

\begin{definition}[Standard-position ellipsoidal cone]\label{def:soc}
We say that a set $E(W) \subseteq \mathbb{R}^d$ is a \emph{standard-position ellipsoidal cone} if there exists a  matrix $W \in \mathbb{D}^{d-1}_{++}$ such that
$$
E(W) = \left\{ c \in \mathbb{R}_+ \times \mathbb{R}^{d-1} : \; c_{[2:d]}'W^{-1}c_{[2:d]} \leq c_{[1]}^2\right\}.
$$
\end{definition}
The definition above implies that $E(W)$ is contained in the halfspace $\{c\in \mathbb R^d:\; e_1'c \geq 0\}$ and that its circumcenter $\hat c(W)$ is equal to $e_1$. We obtain other ellipsoidal cones via rotation.

\begin{definition}[Ellipsoidal cone]\label{def:ellipsoidal_cone}
We say that a set $E(W,U) \subseteq \mathbb{R}^d$ is an \emph{ellipsoidal cone} if there exists an orthonormal matrix $U \in \mathbb R^d \times \mathbb R^d$ and a standard-position ellipsoidal cone $E(W)$ such that
$
c \in E(W,U) \iff U^{-1}c \in E(W).
$
\end{definition}

Since the uncertainty angle is not affected by orthonormal transformations, it follows that, for an ellipsoidal cone $E(W,U)$, the circumcenter is given by $Ue_1$, i.e., $\hat c(E(W,U)) = Ue_1$. 

\textbf{Cone Updates.} A key advantage of keeping track of the knowledge superset as an ellipsoidal cone is that we can build on the ellipsoid method to update our knowledge superset over time. In the classic ellipsoid method, one starts each period with an ellipsoidal feasible set,  finds a cut through the center of the ellipsoid, removing half of the potential solutions, and then one constructs the L\"owner-John ellipsoid of the remaining half-ellipsoid, where the L\"owner-John ellipsoid of a convex set refers to the smallest volume ellipsoid that contains it.  The key argument in the ellipsoid method is that this process reduces the volume of the ellipsoid by $e^{-1/2d}$ at each time step if we are working in a $d$-dimensional space \cite[Lemma 3.1.28, page 74]{grotschel1993ellipsoid}. We will now argue that we can extend this method to work with ellipsoidal cones and that we obtain similar volume reduction guarantees.

\begin{figure}[ht]
  \begin{subfigure}{0.49\textwidth}
    \includegraphics[scale = .2]{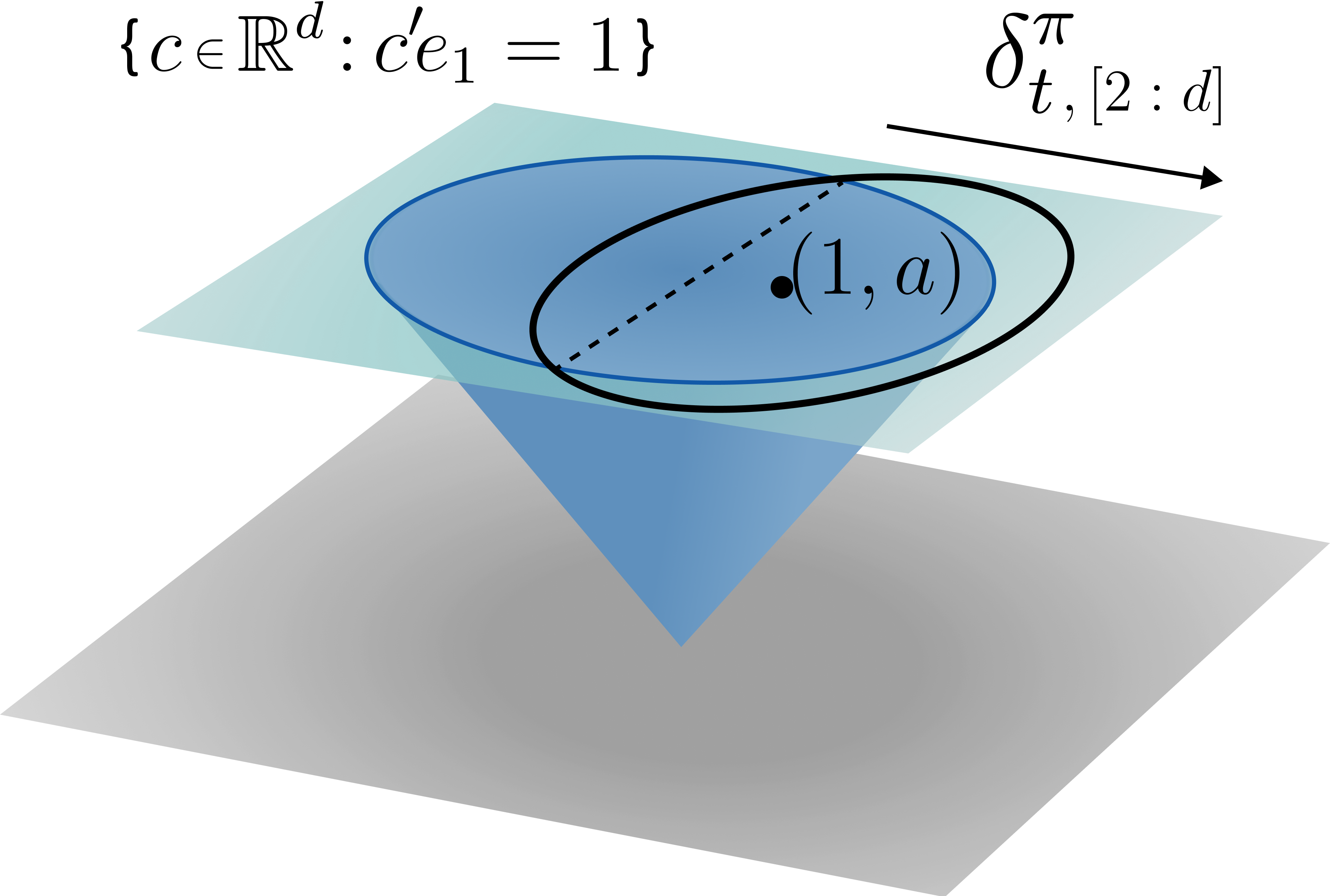}
    \caption{Ellipsoidal cone  intersected with the hyperplane defined by $\hat c(E(W_t,U_t)) = e_1$ and the effective difference projected onto the hyperplane. The updated ellipsoid (in black) remains on the hyperplane and is defined by the center $a$ and the matrix $N$ calculated with the standard ellipsoid method.} \label{fig:algo_a}
  \end{subfigure}%
  \hspace*{\fill}   % maximize separation between the subfigures
  \begin{subfigure}{0.49\textwidth}
    \includegraphics[scale = .2]{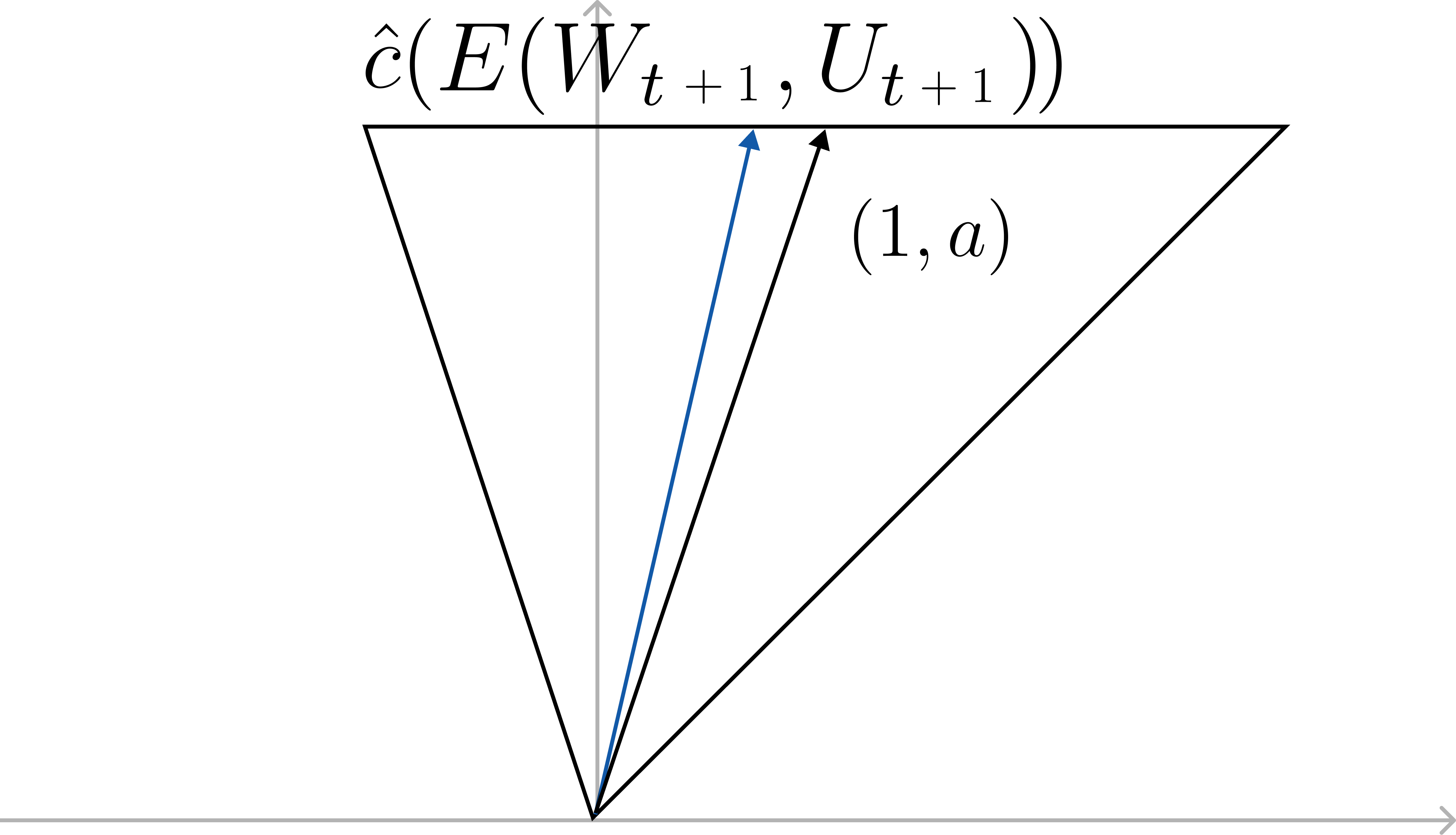}
    \caption{Cross-section of the updated ellipsoidal cone. Note that the vector $(1,a)$ obtained directly by standard ellipsoid method does not give us the circumcenter of the new updated ellipsoidal cone, which is computed instead via spectral decomposition.} \label{fig:algo_b}
  \end{subfigure}%
  \caption{The steps required to perform an update of the ellipsoidal cone.}
\end{figure}

We  now proceed to describe our cone-update algorithm. Intuitively, our  cone-update algorithm can be understood as a two-step procedure. We start by performing one step of the ellipsoid method using the ellipsoid obtained by the intersection of $E(W_t,U_t)$ with an appropriate hyperplane defined by its circumcenter $\hat c(E(W_t,U_t))$ (see \Cref{fig:algo_a}). In the second step, we calculate the circumcenter of the new ellipsoidal cone. As we illustrate in \Cref{fig:algo_b}, the circumcenter of the new ellipsoidal cone is not in general given by the center of the updated ellipsoid. Instead, we need to compute the spectral decomposition of an appropriate matrix to determine it. We call this algorithm \texttt{ConeUpdate} and it is fully laid out  in Algorithm \ref{algo:robust_cone_update}. As specified, it includes more parameters than we need here. Those extra parameters ($\eta$, $B$, and $p$) will only be needed when we get to \Cref{sec:algo_design_general}. We have to divide the cone update in two cases, since when $p = 2$, the underlying ellipsoid lives in a one-dimensional space and the ellipsoid update would reduce to a simple binary search. Otherwise we proceed with the standard update.

\begin{algorithm}[ht]
	\caption{\texttt{ConeUpdate}}
	\label{algo:robust_cone_update}
	\SetKwInOut{Input}{input}
	\SetKwInOut{Output}{output}
	\Input{Matrices $W,U$, projected vector $\delta$, margin $\eta$, basis $B$, dimension $p$.}
	$\bar \delta = U^{-1}B\delta_{[2:p]}/\|U^{-1}B\delta_{[2:p]}\|$\;
	
	\eIf{$p = 2$}{
	
	$\beta = -\eta \bar \delta$, $b = \sqrt W \bar \delta$, $a = \frac{1+\beta}{2}b$\;
	$V_1 = (1,a)'/\sqrt{1+a^2}$\;
	Let $V$ denote any orthonormal basis with $V_1$ as the first component\;
	$\widetilde U = UV$\;
	$\widetilde W = W/2 + \eta$\;
	
	}{
	
	$\beta \leftarrow -\frac{\eta}{\sqrt{{\bar \delta}'W{\bar \delta}}}$, $\quad$ $b \leftarrow \frac{W{\bar \delta}}{\sqrt{{{\bar \delta}'}W{\bar \delta}}}$, $\quad$ $a \leftarrow \frac{1+(p-1)\beta}{p}b$\;
	$N \leftarrow \frac{(p-1)^2}{(p-1)^2-1}(1-\beta^2)(W-\frac{2(1+(p-1)\beta)}{p(1+\beta)}bb')$\;
	$M \leftarrow  \begin{pmatrix} 1 & a' \\ a & aa' - N \end{pmatrix}$\;
	Let $V \Lambda V'$ denote the spectral decomposition of $M$\;
	$\widetilde U \leftarrow UV$\;
	Set $\widetilde W_{i,i} \leftarrow \lambda_{p-i}(N)$, $i = 1,\cdots, p-1,$ where $\lambda_i(N)$ are arranged in nonincreasing order\;
	
	}
	\Output{$\widetilde{W},\widetilde{U}$.}			
\end{algorithm}

\Cref{lemma:ellipsoidal_cone_update} will formalize two properties of  \texttt{ConeUpdate}. It shows that after collecting an effective difference, we can fit another ellipsoidal cone that contains the remaining set of feasible costs and, furthermore, we can guarantee a reduction on the product of the eigenvalues with the same rate as obtained by the ellipsoid method in $(d-1)$ dimensions. We let $\lambda_i(W)$ denote the $i^{th}$ eigenvalue of a matrix $W$ and $I_d \in \mathbb D_{++}^d$ denote the identity matrix.

\begin{lemma}[Ellipsoidal cone updates]\label{lemma:ellipsoidal_cone_update}
    Suppose we run algorithm \emph{\texttt{ConeUpdate}} with inputs $W_t$, $U_t$, $\delta_t^\pi$, $\eta=0$, $B=I_d$, and $p=d$. Assume that $C(\mathcal I_t) \subseteq E(W_t,U_t)$. Then, $C(\mathcal I_{t+1}) \subseteq E(W_t,U_t) \cap \{c\in \mathbb R^d: \; {\delta_t^\pi}'c \geq 0\} \subseteq E(W_{t+1},U_{t+1})$. Moreover, $\prod_{i=1}^{d-1}\lambda_i(W_{t+1}) \leq e^{-1/(d-1)}\prod_{i=1}^{d-1}\lambda_i(W_t)$.
\end{lemma}

%%%%%%%%%%%%%%%%%

\subsection{The Algorithm for the Pointed Case} \label{sec:theorem_2}%Building on the developments of the prior subsections, 
We now explore the  architecture for our algorithm. We start by replacing our initial knowledge set with a revolution cone that contains it (which is possible since we assume  $\alpha(C_0)< \pi/2$ for now),  %(which we assume to satisfy $\alpha(C_0)< \pi/2$) , 
and then at each period $t$, we choose the circumcenter of the ellipsoidal cone $E(W_t,U_t)$ as our proxy cost. We use the algorithm \texttt{ConeUpdate} whenever we want to update our knowledge superset. The last missing step is to propose a rule for deciding when our algorithm should perform a \texttt{ConeUpdate}.

Recall that the uncertainty angle of the knowledge set is key for bounding the one-period regret.  \Cref{prop:angle_and_eigenvalue}  relates  the uncertainty angle of an ellipsoidal cone $E(W,U)$ to the matrix $W$. 
\begin{proposition}[Uncertainty angle of ellipsoidal cones]\label{prop:angle_and_eigenvalue}
    Let $W \in \mathbb D^{d-1}_{++}$ and let $\lambda_{max}(W)$ denote the largest eigenvalue of the matrix $W$. Then, 
    $\alpha(E(W,U))= \arctan \sqrt{\lambda_{max}(W)}$.
\end{proposition}

\Cref{prop:angle_and_eigenvalue}, together with \Cref{theorem:one_period_problem}, highlight the need to control not just  the product of the eigenvalues over time, but also  the value of the largest eigenvalue of $W_t$. Controlling the largest eigenvalue is challenging because of a quirk of the ellipsoid method. When we perform an update step in \texttt{ConeUpdate}, we sometimes create an ellipsoid that has a larger eigenvalue than any eigenvalue in the original ellipsoid (this is discussed, for example, in \citet[Figure 6]{cohen2016feature}). Geometrically, the ellipsoid method runs the risk of creating ellipsoids that are ``long and skinny'' (in technical terms, ill-conditioned). If the ellipsoid has at least one large eigenvalue, then \Cref{prop:angle_and_eigenvalue} shows its uncertainty angle will be large.

The key to completing our pointed-case algorithm will be to use cone-updates  sparingly. Instead of updating our ellipsoidal cone after computing every effective difference, we will only update the ellipsoidal cone if we made a ``substantial'' mistake, where the threshold for determining whether a mistake is substantial is described by a new parameter $\epsilon > 0$. Let $\bar \delta_t$ be the rotated effective difference so that  we are working in standard position, i.e., $\bar \delta_t= U_t^{-1}\delta_t^\pi$. At every time $t$, we do not update the ellipsoidal cone if:% for the next time period if:
\begin{equation}\label{eq:condition-update}
{\bar \delta}'_{t,[2:d]}W_t {\bar \delta}_{t,[2:d]} \leq \epsilon^2.
\end{equation} 
The next result establishes that such periods lead to low regret.

\begin{lemma}[Sufficiency for one-period low regret]\label{lemma:epsilon}
    Assume that $C(\mathcal I_t) \subseteq E(W_t,U_t)$. Suppose that  Eq. \eqref{eq:condition-update} holds,  then the one-period regret in period $t$ is upper bounded by $\epsilon$.
\end{lemma}

In other words, whenever Eq. \eqref{eq:condition-update} holds, we have a certificate to ensure low regret in period $t$ and updating the ellipsoidal cone with the information provided by the current effective difference would risk causing some of the larger eigenvalues to grow for very little benefit.

 Therefore, we say period $t$ is a \emph{cone-update period} if  ${\bar \delta}'_{t,[2:d]}W_t \bar \delta_{t,[2:d]} > \epsilon^2$. Let us denote the number of cone-update periods by $$I^{\pi}_T = \sum_{t=1}^T \mathbf{1} \{{\bar \delta}'_{t,[2:d]}W_t \bar \delta_{t,[2:d]} > \epsilon^2\}.$$ If the period is not an update one, we call it  a \emph{low-regret period} since the regret in that period is upper bounded by $\epsilon$.  The next lemma argues that, since updates only occur when Eq. \eqref{eq:condition-update} is violated, the maximum eigenvalue of the ellipsoid is bounded as function of $I_T^\pi$.

\begin{lemma}[Largest eigenvalue under $\epsilon$-update]\label{lemma:stopping_rule}
 Suppose $\lambda_{min}(W_1) > \left( \frac{\epsilon}{10(d-1)}\right)^2$. Then,   after $I^{\pi}_T$ updates, we have that
    $$
    \lambda_{max}(W_t) \leq \left(\frac{10d}{\epsilon}\right)^{2(d-2)} \left(\lambda_{max}(W_1)\right)^{d-1} e^{-\frac{I_T^\pi}{(d-1)}},
    $$
    where $\lambda_{max}(W_1)$ denotes the largest eigenvalue of the initial ellipsoidal cone $E(W_1,U_1)$. 
    If  $\lambda_{min}(W_1) \le \left( \frac{\epsilon}{10(d-1)}\right)^2$, then $I^{\pi}_T=0$ and the per period regret is bounded above by $\epsilon$ for all periods. 
\end{lemma}
 The  main algorithm of the current section, \texttt{EllipsoidalCones},  chooses the circumcenter of the ellipsoid as the proxy cost in each period, and it updates the ellipsoid only when Eq. \eqref{eq:condition-update} is violated. A full specification is presented in Algorithm \ref{algo:cone_algo_1}. 
 
 \begin{algorithm}[ht]
	\caption{\texttt{EllipsoidalCones}}
	\label{algo:cone_algo_1}
	\SetKwInOut{Input}{input}
	\Input{time horizon $T$, dimension $d$, initial knowledge set $C_0$, parameter $\epsilon$.}
	Rotate $C_0$ to place the circumcenter of $C_0$ at $e_1$\;
	Construct a revolution cone containing $C_0$: $\left(W_{1} \leftarrow \tan^2 \alpha(C_0) I_{d-1},~ U_{1} \leftarrow I_{d-1}\right)$\;
	\For{$1 \leq t \leq T$}{
		Set $c_t^{\pi} \leftarrow U_te_1$, $\bar \delta = U^{-1}_t\delta_t^\pi$\; 
		Choose action $x^\pi_t \in \psi(c^\pi_t,\Fset_t,f_t)$\;
		Observe $x^\star_{t}$\;
		Set $\delta_{t}^{\pi} \leftarrow \frac{f_t(x^\pi_t) - f_t(x^\star_t)}{\|f_t(x^\pi_t) - f_t(x^\star_t)\|}$\;

		\eIf{${\bar \delta}'_{[2:d]}W_{t}{\bar \delta}_{[2:d]} \leq \epsilon^2$}{
		    $W_{t+1},U_{t+1} \leftarrow W_{t},U_{t}$\;
		}{
		    $W_{t+1},U_{t+1} \leftarrow \mbox{\texttt{ConeUpdate}}(W_t,U_t,\delta_t^\pi,0,I_d,d)$\;
		}
	}
\end{algorithm}

 We are now ready to present its performance.

\begin{theorem}[Regret for pointed case]\label{theorem:regret}
	Set $\epsilon = d/T$ in  \emph{\texttt{EllipsoidalCones}}. Consider any $C_0$ with $\alpha(C_0) < \pi/2$. If $\alpha(C_0) > \arctan ({d}/{T})$, then the regret is bounded as follows: 
	$$
	 {\cal WCR}^{\pi}_T\left(C_0\right) \leq 
      \mathcal{O}\left(d^2\ln (T \tan \alpha(C_0))  \right).
	$$
	If $\alpha(C_0) \leq \arctan ({d}/{T})$, then the regret is bounded by $\mathcal O(d)$.
	This algorithm runs in polynomial time in $d$ and $T$. 
\end{theorem}

This result says that if the uncertainty angle of the initial knowledge set is bounded, then we can obtain a strong performance guarantee. We have constructed a policy that avoids the pitfalls of the greedy circumcenter policy and yields logarithmic regret in the time horizon.  The key to this result was to exploit the underlying geometry of the problem at hand and the  regularization of the knowledge sets that enabled the decision-maker to force inverse exploration for nature. These ideas will be crucial when we relax the assumption of pointedness of the initial knowledge set in the next section.

We remark here that the worst-case regret admits a dependence $d^2$. It is possible that the dimension dependence might still be improved; the $d^2$ term in the regret bound is a consequence of our choice of using ellipsoidal cones as supersets. It is possible that an alternative superset that fits the knowledge set more tightly might lead to a better dependence on $d$.

\section{The Online Setting:  General Case}\label{sec:algo_design_general}

In this section, we provide an algorithm for the general case and prove our main result, a logarithmic regret even for the case where there is no prior knowledge about the true cost vector $c^\star$ (cf. \Cref{theorem:regret_general_case}). This algorithm will build on the ideas outlined in \Cref{sec:circumcenter_policy}, but the generality of the initial knowledge set requires a new feature. To that end, we first introduce the notion of \textit{relevant subspaces}, which are subspraces that we construct with relevant feedback information. Next, we show that focusing on ellipsoidal cones within the relevant subspace suffices to achieve logarithmic regret. Moreover, we show that every time we collect relevant information with respect to a new dimension not yet considered, we can restart the ellipsoidal cone approach in this higher dimensional space.

We now drop the assumption that the initial knowledge set lies within a pointed cone ($\alpha(C_0) < \pi/2$) and construct  an algorithm, \texttt{ProjectedCones},  to solve the general case. In this general setting, one key idea will be the following.  At each period $t$, if $C(\mathcal{I}_t)$ does not have an uncertainty angle less than $\pi/2$, we will search for a linear subspace, denoted by $\Delta_t$, such that the projection of $C(\mathcal{I}_t)$ onto this subspace has an uncertainty angle less than $\pi/2$. The high-level idea will be that, if the effective difference $\delta_t^\pi$ at period $t$ is sufficiently close to the subspace $\Delta_t$, then we can operate in that period by ignoring everything that does not lie in the subspace $\Delta_t$. In such periods, we can operate in a way that is similar to \texttt{EllipsoidalCones}. Meanwhile, if the effective difference $\delta_t^\pi$ is far from  $\Delta_t$, that means we will gain enough information in period $t$ in the sense that the projection of $C(\mathcal I_{t+1})$ will be pointed in a higher-dimensional subspace than $\Delta_t$. That is, when the effective difference $\delta_t^\pi$ is far from $\Delta_t$, we will make progress via a \emph{subspace update} from $\Delta_t$ to $\Delta_{t+1}$. In what follows, we will show the modifications needed to operate within subspaces, as well as show how to perform the subspace updates.

We define the orthogonal projection of a vector $c$ onto the subspace $\Delta$ by 
$
\Pi_{\Delta}(c) = \argmin_{\hat c \in \Delta} \|c-\hat c\|.
$
We will use the same notation $\Pi_{\Delta}(C)$ for the projection of a set $C$. We will be working with ellipsoidal cones $E(W_t,U_t)$  that belong to the subspaces $\Delta_t$ and contain the projections of the knowledge sets $\Pi_{\Delta_t}(C(\mathcal I_t))$. 
If we collect an effective difference sufficiently close to $\Delta_t$, meaning that $\|\delta_t^\pi - \Pi_{\Delta_t}(\delta_t^\pi)\| \leq \eta$ for a given $\eta > 0$ that we will choose, we say that we are in a \emph{pointed period}. Just as in Section \ref{sec:circumcenter_policy}, pointed periods can be subdivided into cone-update periods and low-regret periods. For a period where $\|\delta_t^\pi - \Pi_{\Delta_t}(\delta_t^\pi)\| > \eta$, we call it a \emph{subspace-update period}, and we perform a subspace update and define a new ellipsoidal cone in $\Delta_{t+1}$ that contains $\Pi_{\Delta_{t+1}}(C(\mathcal I_{t+1}))$. Our actions are still determined according to the circumcenter policy.

\vspace{.05in}
\noindent\textbf{Working with Subspaces.} Before studying the mechanics of the pointed-case algorithm within a subspace, we need to argue that there exists a positive-dimensional subspace $\Delta_t$ where $\Pi_{\Delta_t}(C(\mathcal I_t))$ belongs to a pointed cone. This is not necessarily true for the initial knowledge set $C_0$ (for example, if $C_0 = S^d$, then the statement is clearly untrue). The next lemma  shows that this statement is true as soon as the  the decision-maker selects an action such that $f_{t_0}(x_{{t_0}}^\pi) \neq f_{{t_0}}(x^\star_{{t_0}})$. %That is, as soon as the decision-maker incurs any regret, we can construct a positive-dimensional subspace where the projected knowledge set lives in a pointed cone. 
\begin{lemma}[Existence of subspace]\label{lemma:existence_subspace}
    Let $t_0$ be the first  period $t$ such that $f_{t}(x_{t}^\pi) \neq f_{t}(x^\star_{t})$ and let: 
    $$
\delta_{t_0}^\pi = \frac{f_{t_0}(x_{t_0}^\pi) - f_{t_0}(x^\star_{t_0})}{\|f_{t_0}(x_{t_0}^\pi) - f_\tau(x^\star_{t_0})\|}, \quad \Delta_{t_0+1} = \{c \in \mathbb{R}^d:\; c = \gamma \delta_{t_0}^\pi, \; \gamma \in \mathbb{R}\}.$$ 
Then, $\Pi_{\Delta_{t_0+1}}\left(C(\mathcal{I}_{t_0+1}^\pi)\right)$ lives in a pointed cone. Moreover, the regret is zero for every $t < t_0$.
\end{lemma}
Since we incur no regret before $f_{t_0}(x_{{t_0}}^\pi) \neq f_{{t_0}}(x^\star_{{t_0}})$, we will assume through the rest of the text that $t_0=1$. % and allow $\Delta_1$ to be a one-dimensional subspace. 
For a given period $t \geq 2$, we can assume the dimension of $\dim(\Delta_t) = p  \in \{1,...,d\}$. The dimension of the matrices $W_t$ and $U_t$ will be determined by the dimension of the subspace $\Delta_t$ since we are forcing the ellipsoidal cone $E(W_t,U_t)$ to live in $\Delta_t$. Then, the matrix $U_t  \in \mathbb R^p \times \mathbb R^p$  is represented under some $\mathbb R^p$ basis for $\Delta_t$. We define the matrix $B_{\Delta_t} \in \mathbb R^p \times \mathbb R^d$ to be a matrix with $p$ rows built with orthonormal vectors that live in $\Delta_t$ (the precise construction of $B_{\Delta_t}$ will be discussed later when we define the steps for a subspace update).
The matrix $B_{\Delta_t}$ will be useful because it allows us to pick some vector $c \in \Delta_t$ (represented under the canonical basis in $\mathbb R^d$), and write it under a basis of $\Delta_t$. In order to do so, we can simply compute $\bar c = B_{\Delta_t}c$. Then, $\bar c \in \mathbb R^p$ is a representation of $c \in \Delta_t$ under the basis of $\Delta_t$. Equivalently,  to represent any $\bar c \in E(W_t,U_t) \subset \mathbb R^p$ under the canonical basis in $\mathbb R^d$, it suffices to compute $c = B_{\Delta_t}'\bar c$.

As previously mentioned, the algorithm will classify the periods into three categories: low-regret, cone-update and subspace-update periods. Let 
$$r_t^\pi = \delta_t^\pi - \Pi_{\Delta_t}(\delta_t^\pi)$$ denote the residual of the projection of the effective difference onto $\Delta_t$. If the residual $r_t^\pi$ is small, in the sense that $\|r_t^\pi\| \leq \eta$ for a given choice of $\eta > 0$, then we will say that we are in a pointed period. Pointed periods can be of two kinds: low-regret and cone-update. Denote  $\bar \delta_t = U_t^{-1}B_{\Delta_t}'\Pi_{\Delta_t}(\delta_t^\pi)$. The multiplication with the matrix $B_{\Delta_t}'$ allows us to write $\Pi_{\Delta_t}(\delta_t^\pi)$ in a basis that is compatible with $E(W_t,U_t)$. The second operation with the matrix $U^{-1}_t$ allows us to work with ellipsoidal cones in standard-position. Analogously to Eq. \eqref{eq:condition-update}, we say a pointed period is a low-regret period if 
\begin{align}\label{eq:update_rule_general_case}
{\bar \delta_{t,[2:p]}}'W_{t} {\bar \delta_{t,[2:p]}} \leq \epsilon^2.
\end{align}
If the period is pointed but the equation above is violated, then we say period $t$ is a cone-update period. When $\|r_t^\pi\| > \eta$, we are not in a pointed period, but instead in a subspace-update period.

\vspace{.05in}
\noindent\textbf{Low-Regret Periods.} 
Low-regret periods are the simplest to analyze since the knowledge superset is not updated on these periods. Our only task here is to show that we have a certificate of low  regret for the circumcenter policy  in such periods. The next lemma shows that if (i) the projection of the knowledge set is contained in the ellipsoidal cone, (ii) the period is pointed, and (iii) Eq. \eqref{eq:update_rule_general_case} is satisfied, then the circumcenter policy incurs low regret that period. However, compared to Lemma \ref{lemma:epsilon},  we do incur an additional regret of $\eta$ since we are working with the projection of the effective difference onto $\Delta_t$ instead of working with the effective difference directly.

\begin{lemma}[Regret incurred in low-regret periods]\label{lemma:stopping_rule_2}
	 Assume that $\Pi_{\Delta_t}(C(\mathcal I_t)) \subseteq E(W_t,U_t)$. If $\|r_t^\pi\| \leq \eta$, and Eq. \eqref{eq:update_rule_general_case} is satisfied, then the regret in that period is upper bounded by $\epsilon+\eta$.
\end{lemma}

\noindent\textbf{Cone-Update Periods.}  There are two important differences in how we perform a cone update in \texttt{ProjectedCones} versus in \texttt{EllipsoidalCones}. The first difference is a simple one: we are working in a subspace $B_{\Delta_t}$ of dimension $p$, rather than the original $d$-dimensional space. The second change is more complex. When we perform a standard ellipsoid method update using a projected effective difference, we run the risk of removing the true cost from our knowledge set because the projected effective difference is not exactly the same as the effective difference. To deal with this, we modify the ellipsoid method to use what is called a shallow cut \citep{grotschel1993ellipsoid} instead of what the traditional ellipsoid method uses, which is a center cut. A shallow cut is a cut that doesn't cut through the center of the ellipsoid and instead leaves more than half of the original ellipsoid behind. The reason to use a shallow cut is to create a buffer to ensure we do not remove the true cost. The needed shallowness will be a function of our choice of $\eta$, as a larger value of $\eta$ would imply that the true effective difference might be further away from the projected effective difference, and therefore we'd need a bigger buffer. Our implementation of \texttt{ConeUpdate} already allows for shallow cuts via a margin $\eta$, and allows for us to operate in the $p$-dimensional subspace $\Delta_t$.

The next lemma shows the reduction in the product of eigenvalues when we run the robust version of a cone update. We call this cone update a robust one since it includes a positive margin $\eta$. It assumes that we are working in a dimension $p > 1$ as the ellipsoidal cone is degenerate when $p=1$ (it's a ray) and, thus, cone-update periods do not occur when $p=1$. 

\begin{lemma}[Robust ellipsoidal cone updates]\label{lemma:properties_robust_update}

    Set $\eta = \epsilon/2d$ and suppose we run the algorithm \emph{\texttt{ConeUpdate}} with inputs $W_t,U_t,\Pi_{\Delta_t}(\delta_t^\pi),\eta,$ $B_{\Delta_t}$ and $p \in \{2,...,d\}$. Assume that $\Pi_{\Delta_t}(C(\mathcal I_t)) \subseteq E(W_t,U_t)$. Then, $\Pi_{\Delta_t}(C(\mathcal I_{t+1})) \subseteq E(W_t,U_t) \cap \{c\in \mathbb R^d \; {\delta_t^\pi}'\Pi_{\Delta_t}(c) \geq 0\} \subseteq E(W_{t+1},U_{t+1})$.  Moreover, $\prod_{i=1}^{p-1}\lambda_i(W_{t+1}) \leq e^{-1/20(p-1)}\prod_{i=1}^{p-1}\lambda_i(W_t)$.
\end{lemma}

The next result leverages the reduction in the product of eigenvalues to bound the number of times cone-update periods occur while we operate in a $p$-dimensional subspace.

\begin{lemma}[Cone-updates per dimension]\label{lemma:cone_updates_per_subspace}

     Let $p \in \{2,...,d\}$ and $I_T^{\pi,p}$ be the number of cone-updates while $\dim(\Delta_t)=p$ and let $t_0^p$ be the first period such that $\dim(\Delta_t)=p$.
      Suppose $\lambda_1(W_{t_0^p}) > \left( \frac{\epsilon}{10(d-1)}\right)^2$. Then, 
    $$I_T^{\pi,p} \leq 20(p-1)^2\ln \left(\frac{10 p \tan \alpha (E(W_{t_0^p},U_{t_0^p}))}{\epsilon}\right).$$ 
 If $\lambda_1(W_{t_0^p}) \le \left( \frac{\epsilon}{10(d-1)}\right)^2$, then $I_T^{\pi,p}=0$ and the per period regret is bounded above by $\epsilon+\eta$  while $\dim(\Delta_t)=p$ .
\end{lemma}

\noindent\textbf{Subspace Update.} 
Next, focus on periods during which the effective difference is ``sufficiently orthogonal" from the subspaces. In such periods, the decision-maker can increase the dimensionality of the subspace while ensuring pointedness in a new subspace of higher dimension. We now show how the subspaces are constructed. We initialize the first-period subspace with $\Delta_1 = \{0\}$. Once the decision-maker observes an effective difference such that $\|r_t^\pi\| > \eta$ for some pre-specified $\eta$, this effective difference is incorporated into the definition of the subspace $\Delta_{t+1}$. That is, let $\tau(t)$ be the set of time periods such that $\|r_i^\pi\| > \eta$, for $i \leq t$. Then, we define $\Delta_t$ to be the subspace defined by the effective differences from periods in $\tau(t)$:
$
\Delta_t = \left\{c \in \mathbb R^d:\; c = \sum_{i \in \tau(t)} \gamma_i \delta^\pi_i, \; \gamma_i \in \mathbb R\right\}.
$

The definition of $\Delta_t$ allows us to construct the matrix $B_{\Delta_t}$. It suffices to take each row as the effective differences $\delta_i^\pi$ and apply the Gram-Schmidt procedure in order to obtain orthogonal rows with unit norm (see, for instance, \cite{bjorck1994numerics}, for implementation details). Moreover, the subspace updating rule implies that there are at most $d$ periods where we increase the dimensionality of our space. Therefore, the cumulative regret incurred in subspace-update periods is small (upper bounded by $d$). The key arguments we need  to show are that $\Pi_{\Delta_{t+1}}(C(\mathcal I_{t+1}))$ is pointed within $\Delta_{t+1}$, and that the uncertainty angle of the new set is controlled by the parameter $\eta$. In order to prove these points, we will consider a polyhedral cone that contains the projection of the knowledge set onto that period's subspace. The period $t$ polyhedral cone is given by:
\begin{equation}\label{eq:K}
K_t = \{c \in \Delta_t:\; c'\delta_i^\pi \geq 0,\; i \in \tau(t)\}.
\end{equation}
We will next show that our projected knowledge set is contained within $K_t$ and that we can bound the uncertainty angle of $K_t$.
\begin{lemma}[Uncertainty angle of polyhedral cone]\label{lemma:gamma_pointed} Let $\eta \in (0,1]$ and let $K_t$ be defined according to Eq. \eqref{eq:K}. Then, for every period $t$, $\Pi_{\Delta_t}(C(\mathcal I_t)) \subseteq K_{t}$. Also, the circumcenter of $K_t$ can be computed in polynomial time from the inputs $\delta_i^\pi$, with $i \in \tau(t)$. 
     Furthermore, the uncertainty angle of $K_t$ is bounded by
    $
    \alpha(K_{t}) \leq \arccos \left({\eta^{d-1}}/{d^{3/2}}\right).
    $
\end{lemma}
\Cref{lemma:gamma_pointed}  ensures that every time  we increase the dimension of our subspace, we can construct an ellipsoidal cone in the new higher-dimensional space that has a bounded uncertainty angle. If we perform a subspace update in period $t$, then we will restart our ellipsoidal cone in period $t+1$ with a revolution cone with aperture angle given by $\arccos ({\eta^{d-1}}/{d^{3/2}})$ and axis given by the circumcenter $\hat c(K_{t+1})$, which we compute using a polynomial-time algorithm we call \texttt{PolyCenter} (see Algorithm \ref{algo:poly_center}).

 The proof of \Cref{lemma:gamma_pointed} is based on various geometric ideas. We use duality to show that providing an upper bound for the aperture angle of the \textit{smallest} revolution cone that \textit{contains} $K_{t+1}$ is equivalent to {providing a lower bound for the aperture angle of the \textit{largest} revolution cone \textit{contained} in the dual of $K_{t+1}$. We also show that in order to establish such property, it suffices to quantify the ``degree" of noncolinearity of the generators of the dual of $K_{t+1}$. By using the fact that the dual cone of $K_{t+1}$ is generated \emph{only} by the effective differences used for the subspace updates, we have by the subspace updating rule that the effective differences are ``sufficiently" orthogonal to the subspace generated by the previous ones, the key of the argument is to characterize the extent to which this holds for the overall system.

\begin{algorithm}[ht]
	\caption{\texttt{PolyCenter}}
	\label{algo:poly_center}
	\SetKwInOut{Input}{input}
	\SetKwInOut{Output}{output}
	\Input{Set of effective differences $\{\delta_i\}_{i \in \tau(t)}$, basis $B$, new dimension $p$.}
	$\bar \delta_i \leftarrow B\delta_i$, for all $i \in \tau(t)$\;
	$\tilde z \leftarrow \sum_{i \in \tau(t)}\bar \delta_i$\;
	$z \leftarrow \tilde{z}/\|\tilde z\|$\;
	\For{$1 \leq k \leq p$}{
	Define $q_k = \chi_k/\|\chi\|_2$, where $\chi_k$ is the unique solution to $\{c \in \mathbb R^p: \;c'z = 1,\; {\bar \delta}_i'c = 0,\; i \in \tau(t), \; i \neq k\}$\; 
	}
	Solve the optimization: $
    \hat c \leftarrow \argmin \|c\|^2, \quad \mbox{s.t. } \; c = \sum_{i = 1}^{p} \gamma_iq_i, \; \sum_{k = 1}^{p} \gamma_i = 1, \; \gamma_k \geq 0
    $\;
   
    $\alpha = \arccos \left(\frac{1}{\|\hat c\|}\min_{i} \hat c'q_k\right)$\;
	\Output{$\hat c/\|\hat c\|, \alpha$.}			
\end{algorithm}

\Cref{fig:dimension_update_figures} illustrates the two types of updates. On the one hand, in \Cref{fig:subspace} we have an effective difference $\delta_t^\pi$ close to the current subspace $\Delta_t$, therefore, the vector $\delta_t^\pi$ is not sufficiently informative for a subspace update. In this case, we can work with the projection of $\delta_t^\pi$ in order to perform a cone update within the subspace (\Cref{fig:subspace_cone_update}).  
On the other hand, \Cref{fig:subspace2} shows an effective difference $\delta_t^\pi$ that gives sufficient information about the third dimension of the problem. Therefore, we can include $\delta_t^\pi$ in $\tau(t)$ and fit a sufficiently pointed ellipsoidal cone in $\mathbb R^3$ (\Cref{fig:subspace_update2}).

\begin{figure}[ht]
  \begin{subfigure}{0.49\textwidth}
    \includegraphics[width=\linewidth]{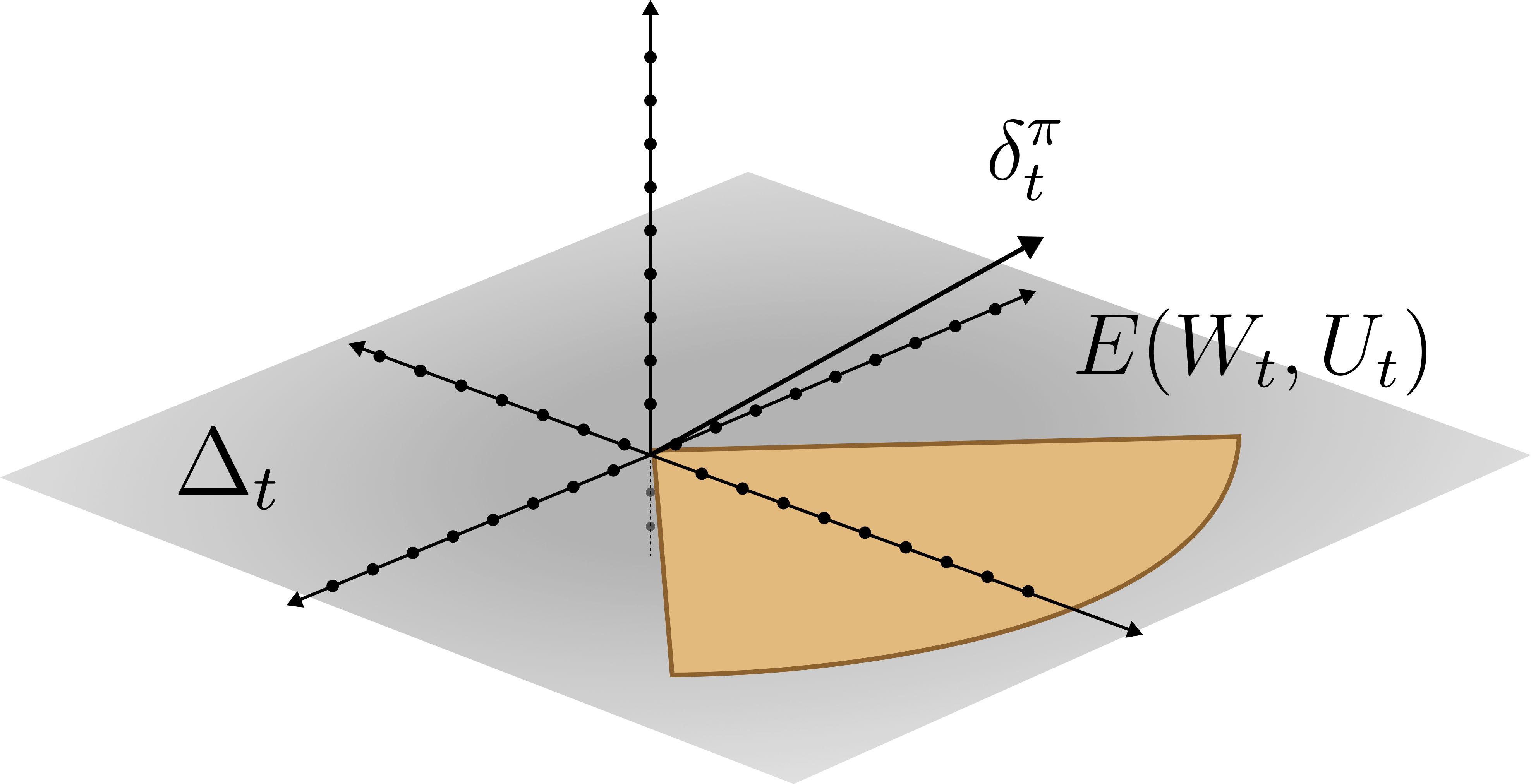}
    \caption{} \label{fig:subspace}
  \end{subfigure}
  \hspace*{\fill}  
  \begin{subfigure}{0.49\textwidth}
    \includegraphics[width=\linewidth]{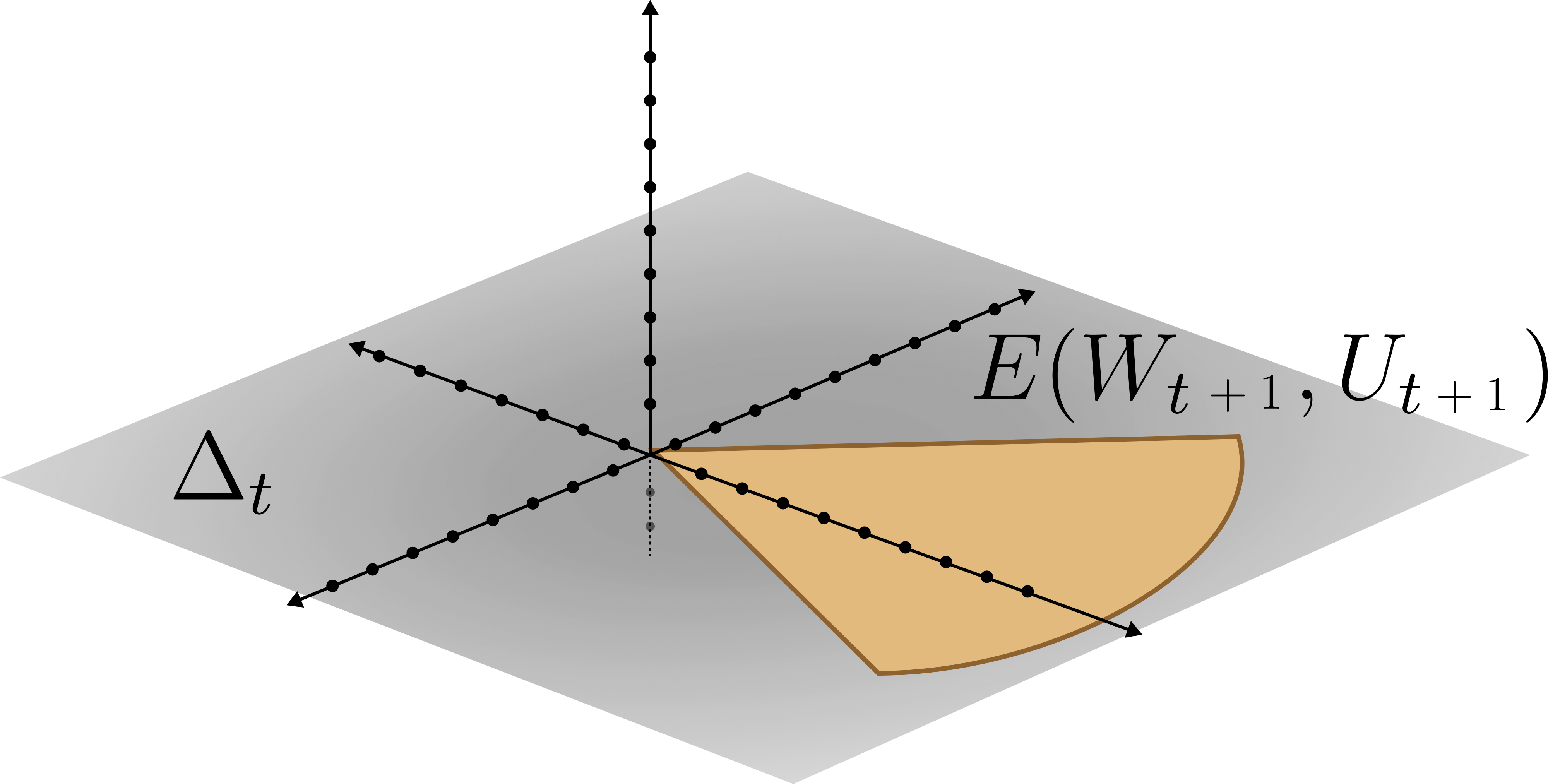}
    \caption{} \label{fig:subspace_cone_update}
  \end{subfigure}%
  \hspace*{\fill} 
  
  \begin{subfigure}{0.49\textwidth}
    \includegraphics[width=\linewidth]{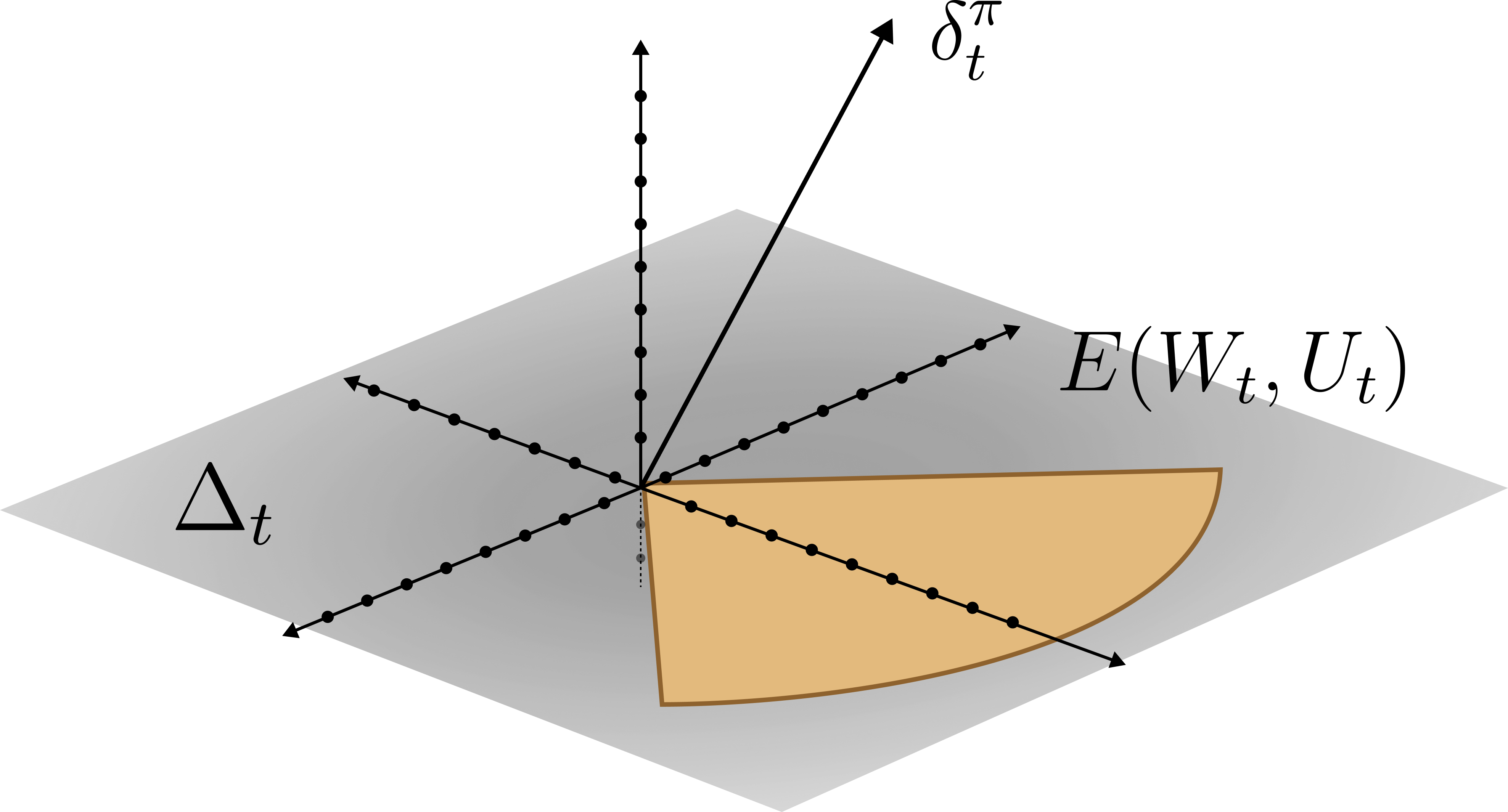}
    \caption{} \label{fig:subspace2}
  \end{subfigure}%
  \hspace*{\fill}  
  \begin{subfigure}{0.49\textwidth}
    \includegraphics[width=\linewidth]{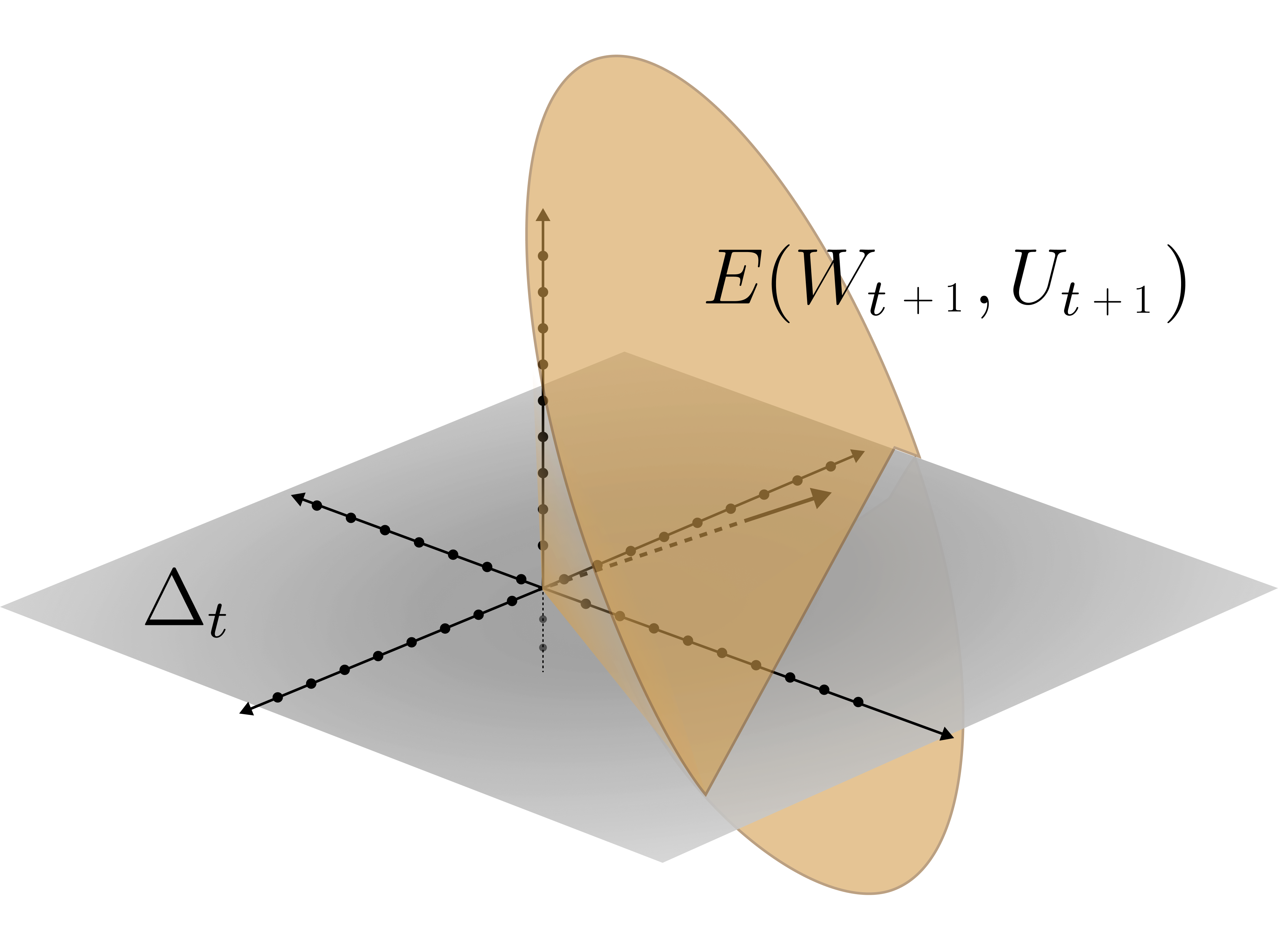}
    \caption{} \label{fig:subspace_update2}
  \end{subfigure}
  \caption{Two types of cone updates. In $(a)$ and $(b)$, we peform a cone update within the subspace $\Delta_t$. In $(c)$ and $(d)$ we have a subspace update.}
  \label{fig:dimension_update_figures}
\end{figure}

\vspace{.05in}
\noindent\textbf{The Algorithm for the General Case.} 
We  have  now presented all of the different parts of our main algorithm, \texttt{ProjectedCones}, including how to verify the kind of period we are in (low-regret, cone-update or subspace-update) and what update to perform in cone-update and subspace-update periods. The full specification of \texttt{ProjectedCones} is presented in Algorithm \ref{algo:general_algo}. 

\begin{algorithm}[!htb]
	\caption{\texttt{ProjectedCones}}
	\label{algo:general_algo}
	\SetKwInOut{Input}{input}
	\Input{$T, \eta, \epsilon$}
	Set $c_1^{\pi} = e_1$ and $x^\pi_1 \in \psi(c^\pi_1,\Fset_1,f_1)$\;
	Observe feedback $x^\star_{1}$ and set $\delta^\pi_{1} = \frac{f_1(x^\pi_1) - f_1(x^\star_1)}{\|f_1(x^\pi_1) - f_1(x^\star_1)\|}$\;
	$\tau(t) \leftarrow 1$, $\Delta_2 \leftarrow \left\{c \in \mathbb R^d:\; c =  \gamma \delta_1^\pi, \; \gamma \in \mathbb R\right\}$, $p = 1$\;
	Set $B_{\Delta_2} = \delta^{\pi}_1$, $W_2 \leftarrow 0$ and $U_2 = 1$\;
	\For{$2 \leq t \leq T$}{
		Set $c_t^{\pi} \leftarrow B_{\Delta_t}'U_te_1$ and $x^\pi_t \in \psi(c^\pi_t,\Fset_t,f_t)$\;
		Observe feedback $x^\star_{t}$\;
		Set $\delta_{t}^{\pi} = \frac{f_t(x^\pi_t) - f_t(x^\star_t)}{\|f_t(x^\pi_t) - f_t(x^\star_t)\|}$ and $r^\pi_t = \delta^\pi_t - \Pi_{\Delta_t}(\delta^\pi_t)$\;
		\eIf{$\|r_t^\pi\| \leq \eta$}{
		    Set $\delta = U_t^{-1} B_{\Delta_t}\Pi_{\Delta_t}(\delta^\pi_t,p) \in \mathbb R^{p}$\;
			\eIf{${(\delta)}'_{[2:p]}W_{t} {(\delta)}_{[2:p]} \leq \epsilon^2$}{
			$W_{t+1}, U_{t+1}, \Delta_{t+1} \leftarrow W_t, U_t, \Delta_t$ \;
			}{
				$W_{t+1}, U_{t+1} \leftarrow \mbox{\texttt{ConeUpdate}}(\Pi_{\Delta_t}(\delta_t^\pi),W_t,U_t,\eta, B_{\Delta_t}, p)$\;
				$\Delta_{t+1} \leftarrow \Delta_t$\;
			}
		}{
		    $\tau(t+1) \leftarrow (t,\tau(t)),\;p \leftarrow p+1$\; 
			$\Delta^\pi_{t+1} \leftarrow \left\{c \in \mathbb R^d:\; c = \sum_{i \in \tau(t)} \gamma_i \delta^\pi_i, \; \gamma_i \in \mathbb R\right\}$\;
			Construct $B_{\Delta_{t+1}}$ orthonormal with the linearly independent vectors $\delta_i^\pi, \; i \in \tau(t+1)$\;
			$u_1,\alpha \leftarrow \texttt{PolyCenter}(\delta_i^\pi, \; i \in \tau(t+1), B_{\Delta_{t+1}}, p)$\;

			Set $W_{t+1} \leftarrow \tan^2(\alpha) I_{p-1}$ and $U_{t+1} \subset \Delta_{t+1}$ orthonormal with first column $u_1$\;
		}
	}
\end{algorithm}

Our main result shows that this algorithm is computationally efficient and incurs regret that is logarithmic in $T$, independently of the nature of the initial knowledge set.

\begin{theorem}[Regret for  general case]\label{theorem:regret_general_case}
	Set $\epsilon = d/T$ and $\eta = \epsilon/2d$ in \emph{\texttt{ProjectedCones}}. For any set $C_0$, the worst-case regret of the  \emph{\texttt{ProjectedCones}} policy $\pi$ is upper bounded as follows.
	$${\cal WCR}^{\pi}_T\left(C_0\right) \leq \mathcal{O}\left(d^4\ln T \right). $$
	Moreover, this algorithm runs in polynomial time in $d$ and $T$.
\end{theorem}

In \Cref{theorem:regret}, we had already established the possibility of obtaining logarithmic dependence for this class of problems, but this dependence relied on the nature of initial knowledge set. With \Cref{theorem:regret_general_case} above, we have now established that the possibility of achieving logarithmic dependence on the time horizon is \textit{universal}: it is independent of the nature of the initial knowledge set. As discussed in the introduction, the best known regret bounds  to date were $\mathcal{O}(\sqrt{T})$ for linear context functions, as developed in \cite{barmann2017emulating}.

For our algorithm, the structure of the knowledge set determines the dependence on $d$. The term $d^4$ follows from us running a variation of \texttt{EllipsoidalCones} $d$ times,  essentially restarting whenever we perform a subspace update, and the regret of that algorithm being $\mathcal O(d^2 \ln (T \tan \alpha(K))$ for an initial set $K$. The final $d$ factor emerges from using \Cref{lemma:gamma_pointed} to bound $\tan \alpha(K)$.

\paragraph{Remark on the regret analysis.} \label{pg:reg-remark}
In \cite{barmann2017emulating,barmann2018online}, the authors establish that it is possible to reframe the problem as an online convex optimization problem. To be more precise, notice that
$$
(x_t^\pi-x_t^\star)'c^\star \leq (x_t^\pi-x_t^\star)'c^\star + (x_t^\star-x_t^\pi)'c^\pi_t = (x_t^\pi-x_t^\star)'(c^\star-c_t^\pi). 
$$
The RHS above can be upper bounded by the regret with respect to the sequence of linear loss functions given by $\ell_t(c) = (x^\star_t-x^\pi_t)'c$, where the sequence $\{x^\star_t-x^\pi_t\}$ is fixed. By analyzing the regret with respect to this sequence of losses, the authors are able to characterize an upper bound on $\sum_{t=1}^T(x_t^\pi-x_t^\star)'(c^\star-c_t^\pi)$. In turn, this bound leads to a bound on the regret we study in this paper $\sum_{t=1}^T(x_t^\pi-x_t^\star)'c^\star$. In our setting, we show that it is possible to analyze directly the suboptimality gap of the decision-maker's action with respect to the expert's action under the true cost, and by doing so, one can tighter upper bounds on this regret.

\section{Numerical Experiments} \label{sec:numerics}

In this section, we explore numerically the performance of the algorithms we propose. To do so,  we conduct a series of experiments in the setting of learning from revealed preferences. In \Cref{sec:simulations_1}, we anchor the numerical experiments around a setup proposed in \cite[Section 4.1]{barmann2017emulating}, where the consumer (expert) faces a set of goods with prices and budget varying over time. The goal of the decision-maker is to learn the consumer's utilities of each product by solving the same sequence of problems. Since the utilities of the goods are nonnegative, this allow us to compare our \texttt{EllipsoidalCones} algorithm (algorithm \ref{algo:cone_algo_1}) with Exponential Weights Update (EWU) proposed in \cite{barmann2017emulating} and Online Gradient Descent (OGD) proposed in \cite{barmann2018online}. In \Cref{sec:simulations_2}, we enrich the problem in order to consider the case where $c^\star$ lies anywhere in the unit sphere. For this case, goods valuation are given by the inner product of a vector of features with an unknown vector. In this case, features, prices and budget constraints are changing over time. We compare the \texttt{ProjectedCones} algorithm (\ref{algo:general_algo}) with OGD since EWU is not applicable in this case.

\subsection{Learning from Revealed Preferences - Pointed Case}\label{sec:simulations_1}
At each time $t$, a consumer chooses among goods $i =1,\cdots,d$ under a budget constraint. Prices and budget changes\ over time and each good has a nonnegative valuation for the consumer. Note that since the utilities for every product is nonnegative, we have a priori knowledge that $c^\star$ lives in the nonnegative orthant, which is a pointed cone with circumcenter given by the vector of ones and uncertainty angle given by $\arccos \sqrt{1/d}$, the \texttt{EllipsoidalCones} algorithm can be used.

We assume that the consumer chooses goods by maximizing their own utility. Let $p_t \in \mathbb R^d_+$ denote the vector of prices at time $t$ and $b_t \in \mathbb R_+$ denote the budget constraint. The consumer solves

\begin{align}\label{eq:foward_problem_preferences}
\max_{x} u'x: \quad x'p_t \leq b_t, \quad x \in \{0,1\}^d.
\end{align}

Note that this problem fits in our formulation by taking $c^\star = -u$, $f_t(x) = x$, and $\mathcal X_t$ as a polyhedron intersected with the binary constraints. At each time $t =1,\cdots,T$, the decision-maker has access to the prices realization $p_t$ and budget constraint $b_t$. The goal of the decision-maker is to mimic the consumer behavior by solving the forward problem \eqref{eq:foward_problem_preferences}. In the end of each period, the decision-maker gets to know the consumer true choice. The regret incurred at each time $t$ is the suboptimality gap of the decision-maker ``recommended action/best guest" with respect to the consumer's action.

We follow the same setup of \cite{barmann2017emulating} for generating the stochastic knapsack instances. The consumer's unknown utilities for the different goods are drawn as integer numbers from the interval $[1, 1000]$ according to a uniform distribution. We normalize the utilities using the $1$-norm in order to compare our method with other proposed algorithms. The prices for sample $t$ are chosen to be
$p_{t,[i]} = u_{[i]} + 100 + r_{t,[i]}$, $i = 1,\cdots, n$, where $r_{t,[i]}$ is an integer uniformly drawn from the interval $[-10,10]$. Finally, the budget constraint $b_t$ is also an integer drawn uniformly from the interval $[1, \sum_{i=1}^d p_{t,[i]}-1].$ We set $\epsilon = d/T$ in \texttt{EllipsoidalCones} and step sizes of $\sqrt{\log d/T}$ and $O(1/\sqrt T)$ for EWU and OGD, respectively \footnote{These step sizes follows the choices in \cite{barmann2018online} of $D/G\sqrt T$, where $D$ is the diameter (here $D = 1$ ) and $G$  is the maximal squared norm of the gradient (here $G = \sup_{x_1,x_2 \in \mathcal X} \|x_1-x_2\|^2 = d$).}. Moreover, at each iteration of OGD, we project the proxy cost back to the set $C_0 = \{c: \sum c_i = 1, c_i \geq 0, \forall i\}$ \footnote{Here we used the L1 norm intersected with the nonnegative orthant to follow the same setup of \cite{barmann2018online}}. In \Cref{fig:sim01} we compare the cumulative regret of our algorithm with OGD and EWU for $d = 10$ and $T = 1000$.  In \Cref{fig:sim02} we repeat the simulation for $d = 25$ and $T = 5000$. In both settings, the proposed algorithm \texttt{EllipsoidalCones} outperforms  both of the alternatives. Although all the methods achieve sublinear regret, we can also see from the simulations that the dimensionality increase has a large impact in the cumulative regret in all of them. For instance, when $T = 1000$, and when the dimensionality increased from 10 to 25, the regret is about 5 times higher for EWU and Ellipsoidal cones and about $8$ times higher for OGD.

\begin{figure}[!ht]
  \begin{subfigure}{0.49\textwidth}
\begin{tikzpicture}

\begin{axis}[
            title={},
	        width=7.5cm,
	        height=7.5cm,
	        xlabel =  Period, 
	        ylabel = Cumulative Regret, 
	        grid=both, 
	        legend pos=south east,
	        legend style={nodes={scale=0.8, transform shape}},]
\addplot[color = magenta, ultra thick, dashed] table[x="time", y = "OGD", col sep = comma] {Data_simulations/d_10_pointed_shaded_reduced.dat};\addlegendentry{OGD} 
\addplot[color = orange, ultra thick, dashdotted] table[x="time", y = "EWU", col sep = comma] {Data_simulations/d_10_pointed_shaded_reduced.dat};\addlegendentry{EWU} 
\addplot[blue, ultra thick] table[x="time", y = "EllipsoidalCones", col sep = comma] {Data_simulations/d_10_pointed_shaded_reduced.dat};\addlegendentry{EllipCones}

% shaded areas plot
% sup
\addplot[name path = max_ogd, color = white, thin] table[x="time", y = "max_ogd", col sep = comma] {Data_simulations/d_10_pointed_shaded_reduced.dat};
\addplot[name path = max_ewu, color = white, thin] table[x="time", y = "max_ewu", col sep = comma] {Data_simulations/d_10_pointed_shaded_reduced.dat};
\addplot[name path = max_cone, color = white, thin] table[x="time", y = "max_cones", col sep = comma] {Data_simulations/d_10_pointed_shaded_reduced.dat};
% inf
\addplot[name path = min_ogd, color = white, thin] table[x="time", y = "min_ogd", col sep = comma] {Data_simulations/d_10_pointed_shaded_reduced.dat};
\addplot[name path = min_ewu, color = white, thin] table[x="time", y = "min_ewu", col sep = comma] {Data_simulations/d_10_pointed_shaded_reduced.dat};
\addplot[name path = min_cone, color = white, thin] table[x="time", y = "min_cones", col sep = comma] {Data_simulations/d_10_pointed_shaded_reduced.dat};

% filling
\addplot [magenta, fill opacity=0.2] fill between [of = max_ogd and min_ogd];
\addplot [orange, fill opacity=0.2] fill between [of = max_ewu and min_ewu];
\addplot [blue, fill opacity=0.2] fill between [of = max_cone and min_cone];

\end{axis}
\end{tikzpicture}
\caption{$d = 10$ and $T = 1000$.} \label{fig:sim01} 
\end{subfigure}
\begin{subfigure}{0.49\textwidth}
\begin{tikzpicture}
\begin{axis}[
            title={},
	        width=7.5cm,
	        height=7.5cm,
	        xlabel =  Period, 
	        ylabel = Cumulative Regret, 
	        grid=both, 
	        legend pos=south east,
	        legend style={nodes={scale=0.8, transform shape}},]
\addplot[color = magenta, ultra thick, dashed] table[x="time", y = "OGD", col sep = comma] {Data_simulations/d_25_pointed_shaded_reduced.dat};\addlegendentry{OGD} 
\addplot[color = orange, ultra thick, dashdotted] table[x="time", y = "EWU", col sep = comma] {Data_simulations/d_25_pointed_shaded_reduced.dat};\addlegendentry{EWU} 
\addplot[blue, ultra thick] table[x="time", y = "EllipsoidalCones", col sep = comma] {Data_simulations/d_25_pointed_shaded_reduced.dat};\addlegendentry{EllipCones}

% shaded areas plot
% sup
\addplot[name path = max_ogd, color = white, thin] table[x="time", y = "max_ogd", col sep = comma] {Data_simulations/d_25_pointed_shaded_reduced.dat};
\addplot[name path = max_ewu, color = white, thin] table[x="time", y = "max_ewu", col sep = comma] {Data_simulations/d_25_pointed_shaded_reduced.dat};
\addplot[name path = max_cone, color = white, thin] table[x="time", y = "max_cones", col sep = comma] {Data_simulations/d_25_pointed_shaded_reduced.dat};
% inf
\addplot[name path = min_ogd, color = white, thin] table[x="time", y = "min_ogd", col sep = comma] {Data_simulations/d_25_pointed_shaded_reduced.dat};
\addplot[name path = min_ewu, color = white, thin] table[x="time", y = "min_ewu", col sep = comma] {Data_simulations/d_25_pointed_shaded_reduced.dat};
\addplot[name path = min_cone, color = white, thin] table[x="time", y = "min_cones", col sep = comma] {Data_simulations/d_25_pointed_shaded_reduced.dat};

% filling
\addplot [magenta, fill opacity=0.2] fill between [of = max_ogd and min_ogd];
\addplot [orange, fill opacity=0.2] fill between [of = max_ewu and min_ewu];
\addplot [blue, fill opacity=0.2] fill between [of = max_cone and min_cone];

\end{axis}
\end{tikzpicture}
\caption{$d = 25$ and $T = 5000$} \label{fig:sim02} 
\end{subfigure}
\caption{Average cumulative regret over 50 simulations for EWU, OGD and the \texttt{EllipsoidalCones} algorithms for the pointed case. The shaded regions depict bands associated with the cumulative regret falling  one standard deviation away from the average.}
\end{figure}

\subsection{Learning from Revealed Preferences - General Case}\label{sec:simulations_2}

In our last comparison, we test the \texttt{ProjectedCones} algorithm. Next we consider the case where goods are described by a vector of features in $\mathbb R^d$ and the consumer faces $p$ goods at each period. In turn, the consumer solves the following direct problem

\vspace{-.1in}
\begin{equation}\label{eq:foward_problem_preferences_2}
\max_{x} x'Z_tu: \quad x'p_t \leq b_t, \quad x \in [0,1]^d,
\end{equation}

where $Z_t \in \mathbb R^p \times \mathbb R^d$ is a matrix where each row $i$ is the vector of features $z_{i,t} \in \mathbb R^d$ of the product $i$. The goal of the decision-maker is to recover the consumer's valuation for each feature of the products. This problem  falls under our formulation by defining $f_t(x) = x'Z_t$ and making $c^\star = -u$. In this case, $\mathcal X_t$ is a polyhedron defined by the budget constraint and prices of the products. We solve a fractional knapsack problem for this experiment. The instance for each simulation was generated with the following procedure. The utility vector $u$ is sampled from a multivariate gaussian with mean vector of ones and variance matrix as a diagonal vector of fours. We use the L2 norm to normalize the utility vector. Note that now some components of $u^\star$ may be negative. For instance, it could be the coefficient associated with the level of sugar of a cereal. The prices and budget are generated in the same fashion as the previous simulation and the product features are sampled from a standard multivariate normal with dimension $p$. We make the number of products to be twice the number of features. In \Cref{fig:sim03,fig:sim04} we compare the cumulative regret of the \texttt{ProjectedCones} algorithm with OGD for $p = 20, d = 10, T = 1000$ and $p = 50, d = 25, T = 5000$, respectively. Here, at each step of the OGD algorithm, we project it back to the unit ball using the Euclidean distance. We average the regret over 50 simulations. Note that for this case EWU do not apply, since the coefficient of some of the features can be negative. The simulations demonstrate that our method does not only provide improved theoretical guarantees against adversarial instances but also performs empirically better than the benchmarks for stochastic generated instances. For $T = 1000$, note that the increase in the dimension had an impact of about 10 times in both algorithms. Therefore, the fact that our theoretical bounds have a explicit polynomial dependence on the dimension doesn't seem to be worst than what happens with OGD even though the latter has in theory, a better dependence on the dimension.

\begin{figure}[!ht]
\begin{subfigure}{0.49\textwidth}
\begin{tikzpicture}
\begin{axis}[
            title={},
	        width=7.5cm,
	        height=7.5cm,
	        xlabel =  Period, 
	        ylabel = Cumulative Regret, 
	        grid=both, 
	        legend pos=south east,
	        legend style={nodes={scale=0.8, transform shape}},]
\addplot[color = magenta, ultra thick, dashed] table[x="time", y = "OGD", col sep = comma] {Data_simulations/d_10_sphere_shaded_reduced.dat};\addlegendentry{OGD} 
\addplot[blue, ultra thick] table[x="time", y = "ProjectedCones", col sep = comma] {Data_simulations/d_10_sphere_shaded_reduced.dat};\addlegendentry{ProjCones} 

% shaded areas plot
% sup
\addplot[name path = max_ogd, color = white, thin] table[x="time", y = "max_ogd", col sep = comma] {Data_simulations/d_10_sphere_shaded_reduced.dat};
\addplot[name path = max_cone, color = white, thin] table[x="time", y = "max_cones", col sep = comma] {Data_simulations/d_10_sphere_shaded_reduced.dat};
% % inf
\addplot[name path = min_ogd, color = white, thin] table[x="time", y = "min_ogd", col sep = comma] {Data_simulations/d_10_sphere_shaded_reduced.dat};
\addplot[name path = min_cone, color = white, thin] table[x="time", y = "min_cones", col sep = comma] {Data_simulations/d_10_sphere_shaded_reduced.dat};

% filling
\addplot [magenta, fill opacity=0.2] fill between [of = max_ogd and min_ogd];
\addplot [blue, fill opacity=0.2] fill between [of = max_cone and min_cone];

\end{axis}
\end{tikzpicture}
\caption{$d = 10$ and $T = 1000$.} \label{fig:sim03} 
%\end{center}
\end{subfigure}
\begin{subfigure}{0.49\textwidth}
%\begin{center}
\begin{tikzpicture}
\begin{axis}[
            title={},
	        width=7.5cm,
	        height=7.5cm,
	        xlabel =  Period, 
	        ylabel = Cumulative Regret, 
	        grid=both, 
	        legend pos=south east,
	        legend style={nodes={scale=0.8, transform shape}},]
\addplot[color = magenta, ultra thick, dashed] table[x="time", y = "OGD", col sep = comma] {Data_simulations/d_25_sphere_shaded_reduced.dat};\addlegendentry{OGD} 
\addplot[blue, ultra thick] table[x="time", y = "ProjectedCones", col sep = comma] {Data_simulations/d_25_sphere_shaded_reduced.dat};\addlegendentry{ProjCones} 

% % sup
\addplot[name path = max_ogd, color = white, thin] table[x="time", y = "max_ogd", col sep = comma] {Data_simulations/d_25_sphere_shaded_reduced.dat};
\addplot[name path = max_cone, color = white, thin] table[x="time", y = "max_cones", col sep = comma] {Data_simulations/d_25_sphere_shaded_reduced.dat};
% % inf
\addplot[name path = min_ogd, color = white, thin] table[x="time", y = "min_ogd", col sep = comma] {Data_simulations/d_25_sphere_shaded_reduced.dat};
\addplot[name path = min_cone, color = white, thin] table[x="time", y = "min_cones", col sep = comma] {Data_simulations/d_25_sphere_shaded_reduced.dat};

% % filling
\addplot [magenta, fill opacity=0.2] fill between [of = max_ogd and min_ogd];
\addplot [blue, fill opacity=0.2] fill between [of = max_cone and min_cone];

\end{axis}
\end{tikzpicture}
\caption{$d = 25$ and $T = 5000$.} \label{fig:sim04} 
\end{subfigure}
\caption{Average cumulative regret over 50 simulations for OGD and the \texttt{ProjectedCones} algorithms for the general case. The shaded regions depict bands associated with the cumulative regret falling  one standard deviation away from the average.}
\end{figure}

\section{Conclusion and further directions}\label{sec:conclusion}

In the offline setting, we have shown that without any particular assumption about the offline data $\mathcal D$, the circumcenter policy achieves the minimax regret. Furthermore, the regret of the circumcenter policy is instance-dependent and fully characterized by the uncertainty angle of the set of consistent cost vectors implied by the offline data. We also demonstrated that the online setting presents a new set of challenges. Indeed,  we show that being myopic leads to a linear regret with respect to the time horizon and significant care should be taken to construct dynamic robust \textit{and} informative policies.  The key to ensure learning is to force what we defined as inverse exploration, where nature is induced to choose between ``exploring for us'' or causing no regret. We have shown that such inverse exploration can be conducted effectively by regularizing the knowledge sets over time, leading to the first logarithmic regret results for this class of problems.

The present paper opens many avenues for future research. A first interesting avenue is to explore if the dimensionality dependence on the regret can be improved. Another complementary question is to develop impossibility results for this class of problems.  Finally, a natural extension of the current framework would be to allow for the possibility of noisy feedback.  It would be interesting to explore whether ideas developed for contextual search in  \cite{krishnamurthy2020contextual} to handle irrational agents could be used in the present context.

\setstretch{0.7}
\bibliographystyle{abbrvnat}
\bibliography{references}   
\setstretch{1.3}

\appendix

\newpage%######################################
%####### Proofs
%######################################

\newpage

\begin{center}
 {\Large \textbf{Electronic Companion: 
 Appendix for \\
Contextual Inverse Optimization: Offline and Online Learning\\}}
\medskip
Omar Besbes\footnote{Columbia Business School --- ob2015@gsb.columbia.edu.}, ~Yuri Fonseca\footnote{Columbia Business School --- yfonseca23@gsb.columbia.edu.}, ~and Ilan Lobel\footnote{NYU Stern School of Business --- ilobel@stern.nyu.edu.}

\end{center}

\pagenumbering{arabic}
\renewcommand{\thepage}{App-\arabic{page}}
\renewcommand{\theequation}{\thesection-\arabic{equation}}
\renewcommand{\thelemma}{\thesection-\arabic{lemma}}
\setcounter{page}{1}
\setcounter{section}{0}
\setcounter{proposition}{0}
\setcounter{lemma}{0}
\setcounter{equation}{0}

\section{Examples of problem classes that fall under our formulation} \label{sec:appl}

 Next, to illustrate the framework introduced in \Cref{sec:pb-form}, we provide a few prototypical classes of problems that fall under it and relate these more precisely to some existing papers discussed in \Cref{sec:lit}. %We show throughout these examples how to cast some previous studied problems under our notation and using the overall notion of minimizing regret.

A first special case is one in which the decision-maker faces a sequence of optimization problems of the form $\min_{x \in \mathcal X_t} x'c^\star$. This formulation was studied in \cite{barmann2017emulating}. %The particular case where $\Fset_t$ are polyhedrons with a strong separation assumption for $c^\star$ was considered in \cite{amin2015online}. 
 A natural extension of this problem, also studied in \cite{barmann2017emulating}, is to consider linear context functions $f_t$. In this case, the sequence of optimization problems becomes solving $\min_{x \in \mathcal X_t} x'Z_tc^\star$ for the context function $f_t(x) = x'Z_t$. 

As a concrete application of the above,  consider a problem where the goal is to learn preferences from observing consumers' behavior. At each time $t$, the consumer faces an arbitrary bundle $S_t$ of $J_t$ products, each with features $z_t^j \in \mathbb R^d$ and price $p_t^j$. The consumer chooses products in the assortment in order to maximize his/her utility function subject to a budget constraint of $b_t$. Let $Z_t \in \mathbb R^{J_t}\times \mathbb R^d$ be the matrix where each row $j$ is given by the vector of features of the product $j$, i.e., $z_t^j$. Let also $p_t = (p_t^1,...,p_t^{J_t})$ be the vector of prices of products $j = 1,...,J_t$ available at time $t$. At each time, the consumer solves
$$
x^\star_t = \argmin_{x \in \mathcal X_t} f_t(x)'c^\star, \quad \mbox{ where } \Fset_t =  \{x \in \{0,1\}^{J_t} \;:\; x'p_t \leq b_t\}, \; f_t(x) = -x'Z_t,
$$
which is a sequence of adversarial knapsack problems where the universe of products, the budget constraint and the prices are allowed to change and are arbitrarily selected (by nature). This relates to the formulation explored in the learning from revealed preferences literature.

When there is no budget and $|S_t| = 1$, the problem becomes one of sequential customer buy/no buy decisions. The question then is how to leverage such binary feedback. The consumer is rational and buys the product if and only if the utility of buying it, modeled as ($z_t'\theta$) satisfies $z_t'\theta \geq p_t$. Here the vector $\theta$ is unknown. The consumer contextual optimization problem is given by $x^\star = \argmax_{x \in \{0,1\}} (z_t'\theta - p_t)x$. In our notation, this would correspond to context function $f_t(x) = -x(z_t,p_t)$, feasible sets $\mathcal X = \{0,1\}$ and cost vector $c^\star = (\theta,-1).$ 
 Even though the structure of the feedback and the contextual optimization problem faced by the consumer would be the same as  in the contextual pricing literature (see, e.g., \cite{cohen2016feature}), the problem described above  is of different nature since the decision-maker is not allowed to select the price of the product and can't affect the feedback directly. Here, the decision-maker's action is a passive one, where instead of  setting the price, the decision-maker merely ``guesses'' if the product would be bought or not by the consumer. This endows the decision-maker with significantly less control on the information collected.

The flexibility of our contextual optimization formulation allows to encompass works in imitation learning problems. In \cite{ward2019learning}, the authors consider an online problem where the goal is to be able to mimic the optimal scheduling policy of an expert for the following dynamic problem. At each time $t$, the length (states) of $n$ queues $z_t \in \mathbb R^n$ is observed and the expert solves $x^\star_t = \argmax_{x \in \mathcal X} x'Bz_t$, where $\mathcal X$ is the set of admissible schedule configurations that is assumed to be bounded. The matrix $B$ is unknown to the decision-maker. This problem  falls under  our formulation by defining $f_t$ to be the Kronecker product between $x$ and $z_t$, $f_t(x) = x \otimes z_t = (x_{[1]}z_t, x_{[2]}z_t, \cdots, x_{[n]}z_t)$, and letting $\mbox{vec}(B)$ be the operator that stacks the columns of $B$. We then have that at each time $t$ the expert solves $$x^\star_t = \argmin_{x \in \mathcal X} f_t(x)'c^\star, \quad \mbox{ where } f_t(x) = x \otimes z_t, \mbox{ and } c^\star = -\mbox{vec}(B).$$

Next, we illustrate that the framework also captures versions of problems in inverse reinforcement learning. A version of the problem studied in \cite{ratliff2006maximum} can be described as follows. At each period $t$, the decision-maker faces a new Markov decision process (MDP) and the objective is to match the state-action frequency of the expert. The initial state distribution and the transition probabilities for the period $t$ MDP are known. % and denoted by $\nu_t$ and $P_t(s',a,s)$, respectively. 
We let the set of feasible actions for the decision-maker $\mathcal X_t$ to be the set of  feasible state-action pair frequencies for the MDP.
 At the end of the period (after a full run of the period $t$ MDP), the decision-maker observes the optimal state-action pair frequency for that period's MDP. At period $t$, let $ \mathcal S_t$ and $\mathcal A_t$ denote the spaces for state and action, respectively.  For each state-action pair $(s,a) \in \mathcal S_t \times \mathcal A_t$, we have an associated $d$-dimensional vector of features $\phi_t(s,a)$. We assume that the cost function $r(s,a)$ associated with taking action $a$ in state $s$ is a linear with respect to the vector of features, i.e., $r(s,a) = \phi_t(s,a)'c^\star$. We also denote $r \in\mathbb R^{|\mathcal S_t|\times \mathcal {A}_t}$ to be the cost vector for each state-action pair. Let $\Phi_t \in \mathbb R^{|\mathcal S_t|\times \mathcal {A}_t} \times \mathbb R^d$ be the matrix where each row is the feature vector associated with each of the state-action pairs. Using the dual of the LP formulation for the MDP (cf. \cite{sutton2006introduction}), we have that $V_\pi = x_\pi'r =  x_\pi'\Phi_t c^\star$, where $x_\pi$ is the vector of state-action pair frequency implied by policy $\pi$. Therefore, at each time $t$ we would like to solve
$$
x^\star_t = \argmin_{x \in \mathcal X_t} f_t(x)'c^\star, \quad \mbox{ where } f_t(x) = x'\Phi_t.
$$

Finally, recall that a structured prediction problem is one where we observe some input and we would like to predict a multidimensional output. Let the inputs be denoted by $z \in \mathcal Z$, the outputs by $x \in \mathcal X(z)$ and the score function by $g(z,x)$. Our problem formulation includes special cases of this class of problems with a regret objective; when it is possible to assume a parametric form for $g(x,z) = -f(x,z)'{c^\star}$ for a known $f(x,z)$, then the forward problem is given by $x^\star = \argmin_{x \in \mathcal X(z)} f(x,z)'c^\star$ for an unknown $c^\star$. Applications of structured predictions include, e.g., natural language processing and image recognition  \citep{taskar2005learning}.

The examples above illustrate the generality of the formulation that we study, but also that there is a potential to lift up some existing formulations within a general framework of contextual inverse optimization with regret objective.

\section{Appendix: Proofs}\label{app:proofs}

\begin{proof}[\textbf{Proof of \Cref{lemma:01}}]
    We first consider the cases where $\theta(c^\star,c^\pi) \geq \pi/2$. In this case, nature can choose $f$ as the identity function and a set $\Fset = \{x_1,x_2\}$ such that $x_2-x_1 = c^\star$. Note that
    $(x_2-x_1)'c^\star = \|c^\star\|^2 = 1 \geq 0$, and therefore $x_2'c^\star \geq x_1'c^\star$, implying that $x_1 \in \psi(c^\star,\Fset,f)$. We also have that $
    (x_2-x_1)'c^\pi = {c^\star}'c^\pi \leq 0$, where the inequality follows from $\theta(c^\star,c^\pi) \geq \pi/2$. Thus,
    $x_1'c^\pi \geq x_2'c^\pi$, implying that $x_2 \in \psi(c^\pi,\Fset,f)$. Therefore, $\mathcal L(c^\pi,c^\star) \ge (x_2-x_1)'c^\star = \|c^\star\|^2 = 1$.

    We now turn our attention to the case where $0 \leq \theta(c^\star,c^\pi) < \pi/2$. The proof is organized as follows. We first introduce a relaxation of the maximization problem in the definition of $\mathcal L$ (see Eq. \eqref{pb-0}). Second, we show that the relaxation leads to a semi-definite programming (SDP) formulation based on the realizability of Gram matrices. Third, we derive an upper bound for the relaxed problem. The final step is to construct an instance that attains this upper bound.
    
    \textbf{Step 1.} Recall from Eq. \eqref{pb-0} that $\mathcal{L}(c^\pi, c^\star) = \sup_{\mathcal X \in \mathcal B,\; f \in \mathcal F, {x}^\pi \in \psi(c^\pi,{\cal X},f) }  ~ \Bigl(f(x^\pi)-f(x^\star)\Bigr)'c^\star.$
      For any $\Fset \in \mathcal{B}$, $f \in \mathcal{F}$, $c^\star,c^\pi \in S^d$, $x^\star \in \psi(c^\star,\Fset,f)$, and  $x^\pi \in \psi(c^\pi,\Fset,f)$, we have, by the optimality of $x^\star$ and $x^\pi$ (for their respective problems), that $(f(x^\pi)-f(x^\star))'c^\star \ge 0$ and $(f(x^\pi)-f(x^\star))'c^\pi \le 0$. 
    
    In particular, we have the following
    \begin{eqnarray}
     \mathcal{L}(c^\pi,c^\star) 
    &=& \sup \left\{ \: (f(x^\pi)-f(x^\star))'c^\star  \: : \: \mathcal X \in \mathcal B,\; f \in \mathcal F,   x^\pi \in \psi(c^\pi,\Fset,f),~x^\star \in \psi(c^\star,\Fset,f)\right\} \nonumber \\
    &\leq& \sup \left\{ \: \delta'c^\star  \: : \:  \delta \in \mathbb{R}^d,~ \|\delta\| \le 1,~ \delta'c^\pi \le 0 \right\} \nonumber \\
    &=& \sup \left\{ \: \delta'c^\star  \: : \:  \delta \in \mathbb{R}^d,~ \|\delta\| = 1, ~\delta'c^\pi \le 0 \right\}, \nonumber
    \end{eqnarray}
    where the inequality follows from noting that for any feasible $x^\star$, $x^\pi$, if one sets  $\delta = f(x^\star)-f(x^\pi)$,  we must have that $\|\delta\| \le \|x^\star - x^{\pi}\| \le 1$, and $\delta'c^\pi \le 0$. The last equality follows from the fact that an optimal solution will always have a vector $\delta$ with maximal norm. Note that the set of $\delta \in \mathbb R^d$ such that $\|\delta\| = 1$ is the set $S^d$. We therefore switch our attention to the optimization problem
    \begin{eqnarray} \label{eq:pb-1}
    \sup \left\{ \: \delta'c^\star  \: : \:  \delta \in {S}^d,~ \delta'c^\pi \le 0\right\},
    \end{eqnarray} which is an upper bound on the value of $\mathcal{L}(c^\pi, c^\star)$.

\textbf{Step 2.} We now analyze problem  \eqref{eq:pb-1}. In particular, we show how it can be written as an SDP and solved explicitly. For any $c^\star$, $c^\pi$ and $\delta \in S^d$, let us define the following matrices
    $$
    A = 
    \begin{bmatrix}
      \delta & c^\star & c^\pi
    \end{bmatrix},
    \quad \mbox{and} \quad 
    B = A'A.
    $$
  Then, $B$ is equal to
     $$
    B = 
    \begin{bmatrix}
      1 & \delta'c^\star & \delta'c^\pi \\
      \delta'c^\star & 1 & c'c^\pi \\
      \delta'c^\pi & {c^\star}'c^\pi & 1
    \end{bmatrix}.
    $$
  Note that $B$ is the Gram matrix associated with the matrix $A$. Therefore, it must belong to the set of symmetric positive semi-definite matrices (see Chapter 8.3.1 of \cite{boyd2004convex}). Moreover, Cauchy-Schwartz implies that all its entries are in $[-1,1]$. We also have that for any feasible solution of Problem \eqref{eq:pb-1},  the entry $B_{1,3} = \delta'c^\pi \leq 0$ and thus $B_{1,3} \in [-1,0]$.
      
    Let $\mathcal S^3_+$ denote the set of symmetric positive semi-definite matrices in $\mathbb R^{3\times 3}$ and let $\mathcal M^3$ denote the set of symmetric matrices in $\mathbb R^{3\times 3}$. In turn, we have the following relaxation:
    \begin{eqnarray}
       && \hspace*{-2cm} \sup \left\{ \: \delta'c^\star ~:~\delta \in   {S}^d   ,  \delta'c^\pi \le 0 \right\} \nonumber \\
      &\le&  \sup \left\{ B_{1,2} \: : \:  B \in \mathcal S^3_+, \quad B_{1,2} \in [-1,1], \quad  B_{1,3} \in [-1, 0], \quad B_{2,3}={c^\star}'c^\pi \right\}\nonumber \\
      &\le&  \sup \{ B_{1,2} \: : \: B \in \mathcal M^3, \quad   \det(B) \ge 0, \quad B_{1,2} \in [-1,1], \nonumber \\
      && \qquad \qquad \qquad B_{1,3} \in [-1, 0], \quad B_{2,3}={c^\star}'c^\pi\}\nonumber\\
      &=&  \sup \{ B_{1,2} \: : \: B \in \mathcal M^3, \quad   1+ 2 B_{1,2}B_{1,3}B_{2,3} - B_{1,2}^2- B_{1,3}^2 - B_{2,3}^2 \ge 0, \nonumber\\
      && \qquad \qquad \quad B_{1,2} \in [-1,1], \quad  B_{1,3} \in [-1, 0], \quad B_{2,3}={c^\star}'c^\pi \}\nonumber\\
     &=&  \sup \{ r \: : \:     1 + 2\rho z r - z^{ 2} - \rho^2 - r^2 \geq 0, \nonumber \\
     && \qquad \qquad \qquad -1 \leq \rho \leq 0, \quad -1 \leq r \leq 1, \quad  z={c^\star}'c^\pi \}, \label{eq-p2}
     \end{eqnarray}     
     where the last inequality follows from the relaxation to symmetric matrices with non-negative determinants. To simplify notation, we denoted $r = \delta'c^\star, \rho = \delta'c^\pi$ and $z = {c^\star}'c^\pi$.
     
    \textbf{Step 3.} In the final step we upper bound the problem in Eq. (\ref{eq-p2}). Let $h(r) = 1 + 2\rho z r - z^{ 2} - \rho^2 - r^2$ and note that $h$ is a quadratic function for any $\rho,z$. Also,  $h$ is concave and admits two roots
   \begin{eqnarray*}
   r_{+} &=& z\rho + \sqrt{z^{2}\rho^2 +(1-z^{ 2}-\rho^2)},\\
    r_{-} &=& z\rho - \sqrt{z^{2}\rho^2 +(1-z^{ 2}-\rho^2)},
   \end{eqnarray*}
  such that  $h(r)  \ge 0$ if and only if $r \in [r_{-},r_{+}]$. Hence, for any feasible $\rho, z$, the maximal achievable value of $r$ is  $r_{+}$. Solving the problem in Eq. \eqref{eq-p2} reduces to finding the feasible values of $\rho$ and $z$ that maximize $r_{+}$. When $\theta(c^\star,c_t^\pi) = 0$, $z = 1$, and $r_+ = 0$, so there is no regret. Now we consider the case where $0 < \theta(c^\star,c_t^\pi) < \pi/2$. Note that $r_{+}$ is differentiable with respect to $\rho$ and its derivative is given by
    $$
    \frac{\partial r_{+}}{\partial \rho} = z + 2 \rho (z^2-1) \frac{1}{2 \sqrt{z^{2}\rho^2 +(1-z^{ 2}-\rho^2)}}.
    $$
     Since $0 < \theta(c^\star,c^\pi)$ implies $z < 1$, the term inside the square root is always greater than zero and the derivative is always well-defined on the feasible set. Moreover, since $\theta(c^\star,c^\pi) < \pi/2$ implies $z > 0$ and $\rho \leq 0$ implies $\rho (z^2-1) \geq 0$, we get that the derivative is non-negative on the feasible set. Hence, independently of the value of $z$, the value of $r_{+}$ is maximized on the feasible set when $\rho$ achieves its maximum value, $0$. 

    Note that for all $\rho \le 0$, we have
    $
    r_{+} \le \sqrt{1-z^{2}},
    $
    with equality when $\rho=0$. Recall that $z = {c^\star}'c^\pi$.  Using the identity $\sin^2 x + \cos^2 x = 1$, we obtain
    $
    r_{+} \le \sqrt{1-({c^\star}'c^\pi)^2}  = \sin \theta (c^\star,c^\pi).
    $
To summarize, we have established that
 $\mathcal L(c^\pi,c^\star) \le \sin \theta (c^\star,c^\pi)$.

    \textbf{Step 4.} We now construct an instance to show that, for any $c^\star$ and $c^\pi$ with $\theta(c^\star,c^\pi) < \pi/2$, one may construct an instance  $\mathcal X$ and $f$ such that the regret is given by $\sin \theta (c^\star,c^\pi)$. Let $\Pi_{c^\pi}(c^\star)$ denote the orthogonal projection of $c^\star$ onto $c^\pi$. Since $\theta(c^\star,c^\pi) < \pi/2$ and both have unity norm, we have that $\Pi_{c^\pi}(c^\star) = \cos \theta(c^\star,c^\pi)\cdot c^\pi$. We now define $r = c^\star - \Pi_{c^\pi}(c^\star)$. Note that $r$ defines the residual of the projection of $c^\star$ onto $c^\pi$. Therefore, it is orthogonal to $c^\pi$ and $\|r\| = \sin \theta(c^\star,c^\pi)$. Then, by summing the angles within the triangle, we  have that $\theta(c^\star,c^\pi) + \theta(r,c^\star) + \pi/2 = \pi$, and thus $ \theta(r,c^\star) = \pi/2 - \theta(c^\star,c^\pi)$.
    
    Now we set $\delta = r/\|r\|$ and let nature pick ${f}(x) = x$ and ${\Fset} = \{x_1,x_2\}$ such that $x_2-x_1 = \delta$. By construction, $\delta$ is parallel to $r$ and must be orthogonal to $c^\pi$. Then, $\delta'c^\pi = (x_2-x_1)'c^\pi = 0$, which implies that $x_1,x_2 \in \psi(c^\pi,\Fset,f)$. Moreover, 
    $$
    \delta'c^\star = \frac{r'c^\star}{\|r\|} = \frac{\|r\|\|c^\star\|\cos \theta(r,c^\star)}{\|r\|} = \cos \theta(r,c^\star) = \cos (\pi/2 -\theta(c^\star,c^{\pi})) = \sin \theta(c^\star,c^\pi),
    $$
    completing the proof.

\end{proof}

\begin{proof}[\textbf{Proof of \Cref{lemma:existence_circuncenter}}] If $C = \{0\}$, then, for any $\hat c \in S^d$, $\sup_{c \in C}\theta(c,\hat c) = 0$ and any $\hat c \in S^d$ is a minimizer. Next we consider the nontrivial case where $C \setminus \{0\}$ is nonempty.

Define the set $\tilde{C} = \{\tilde c \in \mathbb R^d ~:~ \tilde{c} = \frac{c}{\|c\|} \hbox{ for some } c \in C \setminus \{0\}\}$ to be a set of normalized vectors from $C \setminus \{0\}$. Then, the uncertainty angle of $C$ satisfies:
    $$
    \inf_{\hat{c} \in S^d}\sup_{c \in C} \theta(c,\hat{c}) = \inf_{\hat{c} \in S^d}\sup_{c \in C \setminus \{0\}} \theta(c,\hat{c}) =  \inf_{\hat{c} \in S^d}\sup_{c \in C\setminus \{0\}} \arccos \frac{c'\hat{c}}{\|c\|} = \inf_{\hat{c} \in S^d}\sup_{c \in \tilde{C}} \arccos c'\hat{c}, \quad 
    $$
    where the first equality follows from the origin being a suboptimal solution of the maximization $\sup_{c \in C}$ since $\theta(0,\hat{c}) = 0$ and $\theta$ is a non-negative function, the second equality follows from the definition of an angle, and the third equality follows from replacing $C \setminus \{0\}$ with the normalized $\tilde C$.
    
    Since  $\arccos$ is a decreasing continuous function on $[-1,1]$, we have that the uncertainty angle satisfies:
    $$
    \inf_{\hat{c} \in S^d} \sup_{c \in \tilde C} \arccos c'\hat{c} = \inf_{\hat{c} \in S^d}  \arccos \inf_{c \in \tilde  C} c'\hat{c} = \arccos \sup_{\hat{c} \in S^d}   \inf_{c \in \tilde C} c'\hat{c},
    $$
    and the circumcenter $\hat c$ that optimizes the uncertainty angle is the same one that solves the problem:
    \begin{equation}\label{eq:weierstrass}
    \sup_{\hat{c} \in S^d}   \inf_{c \in \tilde C} c'\hat{c}.
    \end{equation}
    
     Define the function $g(\hat{c}) = \inf_{c \in C} c'\hat{c}$. Then, for any $r > 0$ and $u \in S^d$, we have that:
    $$
    g(\hat{c}+ru) = \inf_{c \in \tilde C}c'(\hat{c}+ru) = \inf_{c \in \tilde C} \{ c'\hat{c}+rc'u \}, 
    $$
    From the Cauchy-Schwarz inequality, we know that $|c'u| \leq \|c\| \cdot \|u\|= 1$ since both $c$ and $u$ belong to $S^d$. This implies that $-r \leq r c'u \leq r$. Then,
    $$
    g(\hat{c})-r \leq g(\hat{c}+ru) \leq g(\hat{c})+r \implies |g(\hat{c}+ru)-g(\hat{c})|\leq r,
    $$
    and we have that $g$ is a continuous function in $\mathbb{R}^d$. Therefore, the problem $\sup_{\hat c \in S^d} g(\hat c)$ from Eq. \eqref{eq:weierstrass} is an optimization problem with a continuous objective function over a compact space. By the Weierstrass theorem, the sup is attained. Since the maximizer of Eq. \eqref{eq:weierstrass} is also the minimizer of the uncertainty angle, the infimum of that problem (the circumcenter) is also attained.

    Now we prove uniqueness under the assumption that $\alpha(C) < \pi/2$. Suppose for a moment that the circumcenter is not unique and that  $\hat c_1$ and $\hat c_2$ are two distinct optimal solutions of Eq. \eqref{eq:weierstrass}. We will show that we can construct a new solution $\hat c_3$ which is strictly better than $\hat c_1$ and $\hat c_2$, leading to a contradiction. 
    
    Let $\hat c_3 = (\hat c_1 + \hat c_2)/2$. First we argue that $\hat c_3 \neq 0$. If $\hat{c}_3 = 0$, this would imply that $\hat c_1 = -\hat c_2$. However, Since $\alpha(C) < \pi/2$, then, it must be the case that $c'\hat{c}_1 > 0$, for all $c \in C \setminus \{0\}$ and $c'\hat{c}_2 > 0$, for all $c \in C \setminus \{0\}$, which implies that $c \in C \setminus \{0\}$ is empty, violating the assumption of the lemma. Thus, $\|\hat c_3\| > 0$. We also have that $\|\hat c_3\| < 1$ since $\hat c_3$ is a convex combination of two distinct vectors on the unit sphere.
    
    Let $z$ be the optimal value of Eq. \eqref{eq:weierstrass}, i.e., $z = \inf_{c \in \tilde C} c' \hat c_1 = \inf_{c \in \tilde C}  c'\hat c_2$. Then,
    \begin{align*}
   z &= \frac{1}{2} \inf_{c \in \tilde C}  c'\hat c_1 + \frac{1}{2} \inf_{c \in \tilde C}  c'\hat c_2 \stackrel{(a)}{<} \frac{\frac{1}{2} \inf_{c \in \tilde C}  c'\hat c_1 +\frac{1}{2}\inf_{c \in \tilde C}  c'\hat c_2}{\|\hat c_3\|} 
   \stackrel{(b)}{\leq} \frac{\inf_{c \in \tilde C}  c'\left(\frac{1}{2} \hat c_1 + \frac{1}{2} \hat c_2\right)}{\|\hat c_3\|} = \inf_{c \in \tilde C} c'\frac{\hat{c}_3}{\|\hat c_3\|},
    \end{align*}
    where $(a)$ follows from $\|\hat c_3\| < 1$ and $(b)$ follows from combining two infinimums. Since ${\hat{c}_3}/{\|\hat c_3\|}$ is a feasible solution of Eq. \eqref{eq:weierstrass} and its objective value is strictly above  $z$, this would violate the optimality of $\hat c_1$ and $\hat c_2$. This is a contradiction, and the circumcenter must be unique.
\end{proof}

\begin{proof}[\textbf{Proof of \Cref{theorem:one_period_problem}}]

 First we show how to construct the upper bound by leveraging the result provided in \Cref{lemma:01}. Next we how the construct a set $C$ such that any policy incurs a worst-case regret at least as high as the upper bound provided for the circumcenter policy.

\textit{Step 1}. We start by showing that for any policy $c^\pi \in \mathcal P'$, we can bound the objective function from Eq. \eqref{eq:wcr-offline} using the worst-case regret loss function $\mathcal{L}$ for specific choices of nature $(c^\star, \mathcal X, f)$,  defined in Eq. \eqref{pb-0}. The following inequality holds for any knowledge set $C \subseteq S^d$: 

\begin{align}\label{eq:opt_problem}
 \inf_{\pi \in {\cal P}} \: \sup_{\Fset \in \mathcal{B},\: f \in \mathcal{F}, \:c^\star \in C} \: \bigl(f(x^{\pi})-f(x^\star)\Bigr)'c^\star 
 &\stackrel{(a)}{\leq}    \inf_{\pi \in \mathcal P'} \: \sup_{\Fset \in \mathcal{B},\: f \in \mathcal{F}, \:c^\star \in C} \: \bigl(f(x^{\pi})-f(x^\star)\Bigr)'c^\star \nonumber \\ 
 &\stackrel{(b)}{=}    \inf_{c^\pi \in S^d} \: \sup_{\Fset \in \mathcal{B},\: f \in \mathcal{F}, \:c^\star \in C} \: \bigl(f(x^{\pi})-f(x^\star)\Bigr)'c^\star \nonumber \\ 
&\stackrel{(c)}{\leq}  \inf_{c^\pi \in S^d}  \: \sup_{c^\star \in C} \: \mathcal{L}(c^\pi, c^\star). %\nonumber 
\end{align}
where $(a)$ follows from restricting $\pi$ to ${\cal P}'$, $(b)$ follows from representing $\pi \in \mathcal P'$ in terms of its cost vector $c^\pi$, and  $(c)$ follows from the definition of $\mathcal{L}$ (it would be an equality if $x^\pi$ were unique). 

Let us define:
    \begin{eqnarray*}
       g(x) 
       &=& \begin{cases}
      \sin x & \mbox{if }0 \leq x < \pi/2, \\
      1 & \mbox{if }x \geq \pi/2.
   \end{cases}
    \end{eqnarray*}
\Cref{lemma:01} shows that $\mathcal{L}(c^\pi, c^\star) = g(\theta(c^\star,c^\pi))$. Combining with Eq. \eqref{eq:opt_problem}, we have:
\begin{eqnarray*}
 \inf_{\pi \in {\cal P}} \: \sup_{\Fset \in \mathcal{B},\: f \in \mathcal{F}, \:c^\star \in C} \: \bigl(f(x^{\pi})-f(x^\star)\Bigr)'c^\star 
 &\leq&  \inf_{c^\pi \in S^d}  \: \sup_{c^\star \in C} \: g(\theta(c^\star,c^\pi)) \\
 &=& g\left(\inf_{c^\pi \in S^d}  \: \sup_{c^\star \in C} \: \theta(c^\star,c^\pi)\right) = g(\alpha(C)), %\nonumber 
\end{eqnarray*}
where the first equality follows from $g(\cdot)$ being   nondecreasing and continuous, the second equality follows from the definition of the uncertainty angle and the third equality follows from applying the tighest lower bound by using the circumcenter policy. Therefore,
$$
	 \inf_{\pi \in {\cal P}} \sup_{\Fset \in \mathcal{B},\: f \in \mathcal{F}, \: c^\star \in C}  \:  \bigl(f(x^{\pi})-f(x^\star)\Bigr)'c^\star \leq  g(\alpha(C)) = g(\bar \alpha).
$$

\textit{Step 2.} To show that no policy can be uniformly better than the circumcenter, we construct instances of the problem where any policy incurs at least  the regret $g(\bar \alpha)$. We first consider the case where $\bar \alpha \leq \pi/2$. Let $e = \frac{1}{\sqrt{d}}(1,\cdots, 1)'$ and $\tilde C= \{c \in S^d:\: \theta(e,c) \leq \bar \alpha\}$, which is the intersection between a sphere and the revolution cone with axis $e$ and aperture angle of $\bar \alpha$.  Note that if $\bar \alpha = \pi/2$ there is no revolution cone, but it suffices to consider the halfspace $\tilde C= \{c \in S^d:\: c'e \geq  0\}$. Moreover, let ${f}(x) = x$ and ${\Fset} = \{x_1,x_2\}$ such that $\delta = x_2-x_1$ is orthogonal to $e$ and $\|\delta\|= 1$. 

Define $c_1^\star = \sin \bar \alpha \cdot \delta + \cos \bar \alpha \cdot e$. We now argue that $c_1^\star$ belongs to $\tilde C$. We first show that $c_1^\star \in S^d$:
\[\|c_1^\star\|^2 = \sin^2 \bar \alpha \cdot \|\delta\|^2 + \cos^2 \bar \alpha \cdot \|e\|^2 = \sin^2 \bar \alpha + \cos^2 \bar \alpha = 1,\]
where the first equality follows from $\delta$ and $e$ being orthogonal vectors, and the second equality follow from $\delta$ and $e$ having norm 1. We now show that $\theta(e,c_1^\star) = \bar \alpha$:
\[\theta(e,c_1^\star) = \arccos \frac{e'c_1^\star}{\|e\| \|c_1^\star\|} = \arccos {e'c_1^\star} = \arccos {\cos \bar \alpha} = \bar \alpha,\]
where the second equality follows from $e$ and $c_1^\star$ having norm 1, and the third equality follows from $\delta$ and $e$ being orthogonal.

We now construct a second vector in $\tilde C$, $c_2^\star = -\sin \bar \alpha \cdot \delta + \cos \bar \alpha \cdot e$. By the same argument as above, $c_2^\star$ also belongs to $\tilde C$. Note  that $x_1$ is the optimal action if $c_1^\star$ is the true cost and $x_2$ is the optimal action if  $c_2^\star$ is the true cost since $(x_2-x_1)'c_1^{\star} = \sin(\bar{\alpha}) \ge 0$ and $(x_2-x_1)'c_2^{\star} = - \sin(\bar{\alpha}) \le 0$. Since for any policy $\pi \in {\cal P}$ the decision-maker must choose either $x_1$ or $x_2$, we have:

\begin{align*}
\inf_{\pi \in {\cal P}} \sup_{\Fset \in \mathcal{B},\: f \in \mathcal{F}, \: c^\star \in \tilde{C}}  \:  \bigl(f(x^{\pi})-f(x^\star)\Bigr)'c^\star &\stackrel{(a)}{\geq}  \inf_{x^\pi \in \{x_1,x_2\}} \sup_{\: c^\star \in \tilde{C}} \:  \bigl(x^{\pi}-x^\star\Bigr)'c^\star \\ 
&\stackrel{(b)}{\geq} \inf_{x^\pi \in \{x_1,x_2\}} \sup_{\: c^\star \in \{c_1^\star,c_2^\star\}} \:  \bigl(x^\pi-x^\star\Bigr)'c^\star \\ 
&\stackrel{(c)}{=} \min\{(x_2-x_1)'c_1^\star,(x_1-x_2)'c_2^\star\},\\
&\stackrel{(d)}{=} \min\{\delta'c_1^\star,-\delta'c_2^\star\}  \\
&\stackrel{(e)}{=} \min\{\sin \bar \alpha,\sin \bar \alpha\} = \sin \bar \alpha,
\end{align*}
where $(a)$ follows from our choice of instance $(\Fset,f)$, $(b)$ follows from restricting the choice of $c^\star \in C$ to $\{c_1^\star,c_2^\star\}$, $(c)$ follows from the fact that $x_1$ is optimal for $c_1^\star$ and $x_2$ is optimal for $c_2^\star$, $(d)$ follows from the definition of $\delta$, and $(e)$ follows from the definitions of $c_1^\star$ and $c_2^\star$ and the fact that $\delta$ and $e$ are orthogonal. This completes the argument for the case where $\bar \alpha \leq \pi /2$.

Finally, we now consider the case where $\pi/2 < \bar \alpha \leq \pi$. Similarly to the previous case, let $\tilde C = \{c \in S^d: \; \theta(e,c) \leq \bar \alpha\}$, however, notice that this set is not a revolution cone anymore. Despite that, it is still well defined and $\alpha(C) = \bar \alpha$, $\hat c(C) = e$ by construction. Moreover, let ${f}(x) = x$ and ${\Fset} = \{x_1,x_2\}$ such that $\delta = x_2-x_1$ is orthogonal to $e$ and $\|\delta\|= 1$. But note that in this case, $x_1 \in \tilde C$ since $x_1'e = 0$ which implies that $\theta(x_1,e) = \pi/2 < \bar \alpha$. By the same argument, $x_2 \in \tilde C$. Therefore, no matter the policy $\pi$ used to choose between $x_1$ and $x_2$, nature can always pick a vector in $\tilde C$ to cause maximum regret. This completes the proof.

\end{proof}

\begin{proof}[\textbf{Proof of \Cref{lemma:insuff_circumcenter}}]

Let $\pi = \pi_{greedy}$ and let $e_i$ denote the $i$-th vector of the canonical basis in $\mathbb{R}^d$. For $c \in \mathbb R^d$, we use $c_{[i]}$ to denote the $i$-th entry of the vector $c$. 

The proof strategy is organized as follows. In step 1, we define a useful family of sets that will characterize our sequence of knowledge sets. In step 2, we construct choices of sets $\mathcal X_t$ and context functions $f_t$ so that the greedy policy implies that the sequence of knowledge sets $C(\mathcal I_t)$ always belongs to the family of sets that we defined in step 1. For this construction, we will focus on a 3-dimensional case, i.e., $d=3$. In step 3, we show that the regret in every time period must be uniformly bounded away from zero regardless of the time horizon, leading to the linear regret.

\textit{Step 1.} We first define a family of sets that will be central in the construction of instances with linear regret.

Define $h_1 = (2\sin^2\bar \alpha, -\sin 2\bar \alpha, \sin 2\bar \alpha)$, $h_2 = (2\sin^2\bar \alpha, -\sin 2\bar \alpha, -\sin 2\bar \alpha)$ and $h_3 = (-\varepsilon, 1, 0)$.  Consider the following family of sets indexed by $\bar \alpha$ and $\varepsilon$ with $0 < \bar \alpha < \pi/2$ and $0 \leq \varepsilon \leq (\tan \bar \alpha)/2$:
$$
C_{\varepsilon,\bar \alpha} = \{c \in \mathbb{R}^d :\;h_1'c \geq 0\} \cap \{c \in \mathbb{R}^d :\;h_2'c \geq 0\} \cap \{c \in \mathbb{R}^d :\;h_3'c \geq 0\} \cap S^d.
$$

Note that $C_{\varepsilon,\bar \alpha}$ is the intersection of a polyhedral cone with the unit sphere. Using the halfpsaces defined by $h_1$, $h_2$ and $h_3$, we can compute the generators of such a polyhedral cone. The generators of the cone are $g_1 = (1,\varepsilon, \varepsilon-\tan \bar \alpha)$, $g_2 = (1,\varepsilon, \tan \bar \alpha - \varepsilon)$ and $g_3 = (\cos \bar \alpha,\sin \bar \alpha,0)$. For simplicity, the generators were not normalized. The set $C_{\varepsilon,\bar \alpha}$ is always nonempty if $\varepsilon \leq \tan \bar \alpha$ and $\bar \alpha < \pi/2$ and we fix $c^\star = g_3 = (\cos \bar \alpha,\sin \bar \alpha,0)$. Using \Cref{def:set_angle}, one can also show that $\hat c (C_{\varepsilon,\bar \alpha}) = (1, \varepsilon, 0)$. In \Cref{fig:example_C_epsilon}, we depict an example of initial knowledge set $C_0$ for $\varepsilon = 0$ and $\bar \alpha = \pi/4$.

\textit{Step 2.}  Fix $0 < \bar \alpha < \pi/2$ and define a sequence of instances as follows. We let $C_0 = C_{0,\bar \alpha}$ (\Cref{fig:example_C_epsilon} depicts the set $C_0$ for $\bar \alpha = \pi/4$).  Suppose that $f_t$ is the identity for all $t\ge 1$ and let  
\bearn
&&\varepsilon_1 = \frac{\tan \bar \alpha}{2T}, \qquad  \varepsilon_t = t \varepsilon_1, \quad t \ge 2, \\
&& \Fset_t = \{\bar x_t, 0\}, \quad t \ge 1,
\eearn
where 
$$ \bar x_t = (e_2 - \varepsilon_t e_1)/\|e_2 - \varepsilon_t e_1\|.$$

In \Cref{fig:circ_fail}, we provide an illustration for $\bar \alpha = \pi/4$ and $T$ taken to be 5, so $\varepsilon_1 = 0.1$. 

\begin{figure}[ht]
  \begin{subfigure}{0.49\textwidth}
    \includegraphics[width=\linewidth]{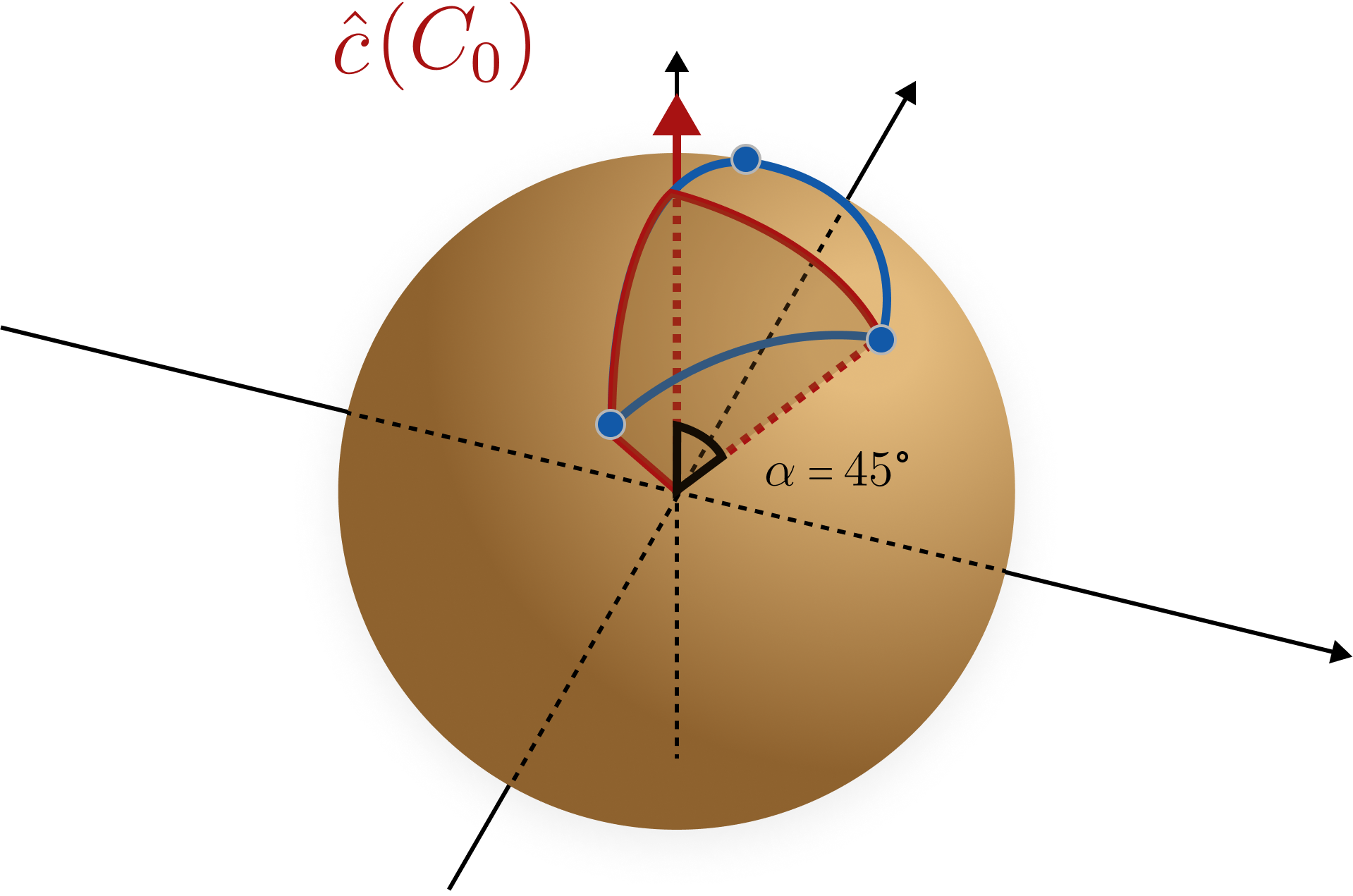}
    \caption{} \label{fig:example_C_epsilon}
  \end{subfigure}%
  \hspace*{\fill}   % maximizeseparation between the subfigures
  \begin{subfigure}{0.49\textwidth}
    \includegraphics[width=\linewidth]{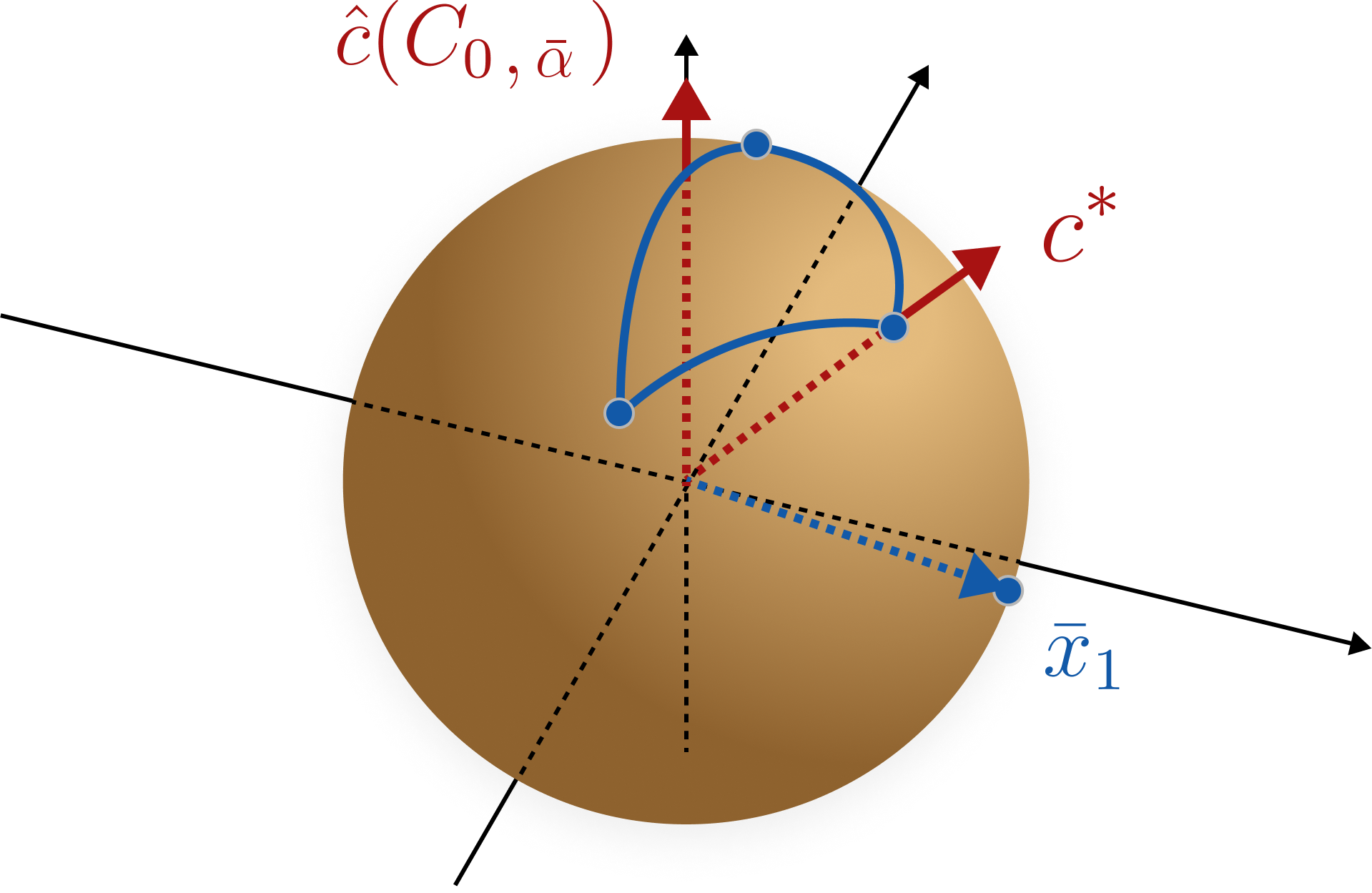}
    \caption{} \label{fig:circ_fail}
  \end{subfigure}
  \caption{In $(a)$, we depict the initial knowledge set $C_0$ and its circumcenter. In $(b)$, we depict the first instance of the optimization problem faced by the decision-maker and the true cost vector $c^\star$.}
\end{figure}

Next, we establish  by induction on $t$ that, under the greedy circumcenter policy, $C(\mathcal I_t) = C_{\varepsilon_{t-1},\bar \alpha}$ for $t \ge 2$. Note that the result is trivial for $t = 1$ since $C(\mathcal I_t) = C_0 = C_{0,\bar \alpha}$ by construction. Next we establish the base case ($t = 2$).  

By definition (Eq. \eqref{eq:knowledge_set}) we have that
$$
C(\mathcal I_2) = C_0 \cap \{c \in \mathbb{R}^{d}:    c'f_1(\bar x_1) \leq c'f_1(x), \;\forall \;x\in {\cal X}_1\}.
$$
Moreover, since $C(\mathcal I_1) = C_{0,\bar \alpha}$, $f_1$ is the identity and $\mathcal X_1 = \{\bar x_1,0\}$, we must have that $\hat c(C(\mathcal I_1)) = e_1$, $\psi(e_1, f_1, \Fset_1) = \bar x_1$ and $\psi(c^\star, f_1, \Fset_1) = 0$. Which leads to
$$
C(\mathcal I_2) =  C_0 \cap \{c \in \mathbb{R}^{d}\;:\;    -\varepsilon_1 c_{[1]} + c_{[2]} \geq 0\} = C_{\varepsilon_1,\bar \alpha},
$$
and the base case is established. Next we show the induction step. Suppose that the result holds for $t$. Then, the circumcenter of $C(\mathcal I_t)$ is given by $(1,\epsilon_{t-1},0)$. Therefore, $\psi(c_t^\pi, f_t, \Fset_t) = \bar x_t$ and $\psi(c^\star, f_, \Fset_t) = 0$, which leads to the update:

\begin{align*}
C(\mathcal I_{t+1}) &= C(\mathcal I_t) \cap \{c \in \mathbb{R}^{d}\;:\;    -\varepsilon_{t} c_{[1]} + c_{[2]} \geq 0\} \\
&= C_0 \cap \{c \in \mathbb{R}^{d}\;:\;    -\varepsilon_{t-1} c_{[1]} + c_{[2]} \geq 0\} \cap \{c \in \mathbb{R}^{d}\;:\;    -\varepsilon_{t} c_{[1]} + c_{[2]} \geq 0\} \\
&=  C_0 \cap \{c \in \mathbb{R}^{d}\;:\;    -\varepsilon_{t} c_{[1]} + c_{[2]} \geq 0\} = C_{\varepsilon_t,\bar \alpha},
\end{align*}
which concludes the proof by induction. Having established the above, we now analyze the regret in each period $t$. 

\textit{Step 3.} From step 2, we have for every time $t$ that
$C(\mathcal I_{t+1}) = C_{\varepsilon_t, \bar \alpha}$, $\hat c(C(\mathcal I_{t+1})) = \frac{1}{\sqrt{1+\varepsilon_t^2}}(1, \varepsilon_t, 0)$ and the regret at period $t$ is given by 
$$
{\delta_t^\pi}'c^\star = \frac{(-\epsilon_t e_1+e_2)'(\cos \bar \alpha, \sin \bar \alpha, 0)}{\sqrt{1+\varepsilon_t^2}} = \frac{\sin \bar \alpha - \varepsilon_t \cos \bar \alpha}{\sqrt{1+\varepsilon_t^2}} \geq \frac{T \sin \bar \alpha - (t/2T) \sin \bar \alpha }{\sqrt{2}} \geq \frac{\sin \bar \alpha}{4},
$$
where the first inequality follows from the fact that $\varepsilon^2 \leq 1$. Therefore, the cumulative regret must be $\Omega(T)$. 

In \Cref{fig:initial_set}, we can see the initial knowledge set $C_0$ for $\bar \alpha = \pi/4$. In \Cref{fig:cir_fail_2}, we have the updated knowledge set $C(\mathcal I_{2}) = C_{\varepsilon, \bar \alpha}$ after solving the first optimization instance. In \Cref{fig:circ_fail_final} we have the final set after collecting the feedback of the last time period $T=5$. No matter the horizon $T$, nature can always adjust $\epsilon_1$ as a function of $\bar \alpha$ and $T$ in in order to ensure that the updates are not enough to make the circumcenter and the true cost vector sufficiently close to each other.

\begin{figure}[ht]
  \begin{subfigure}{0.32\textwidth}
    \includegraphics[width=\linewidth]{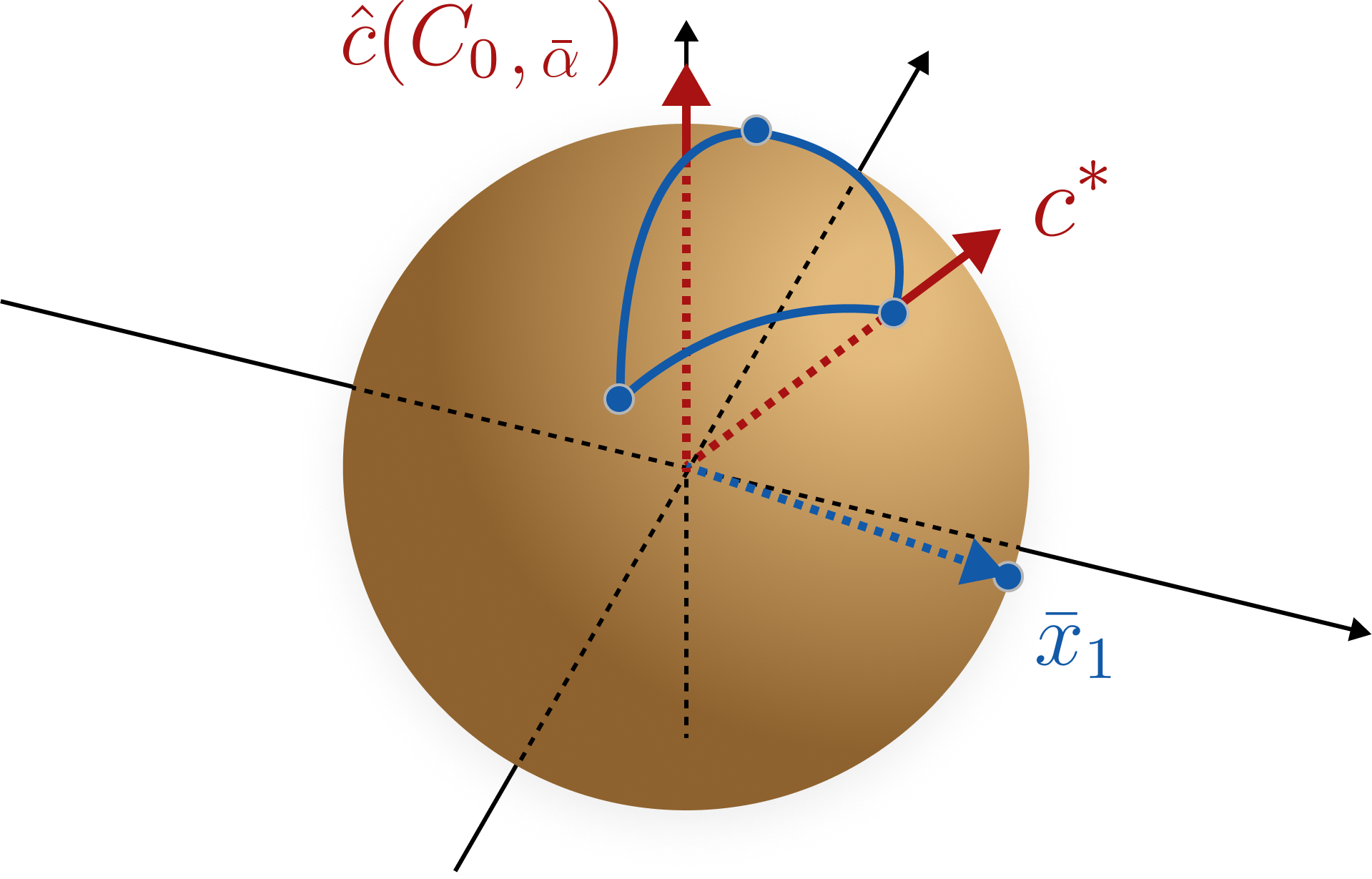}
    \caption{}\label{fig:initial_set} 
  \end{subfigure}%
  \hspace*{\fill}   % maximizeseparation between the subfigures
  \begin{subfigure}{0.32\textwidth}
    \includegraphics[width=\linewidth]{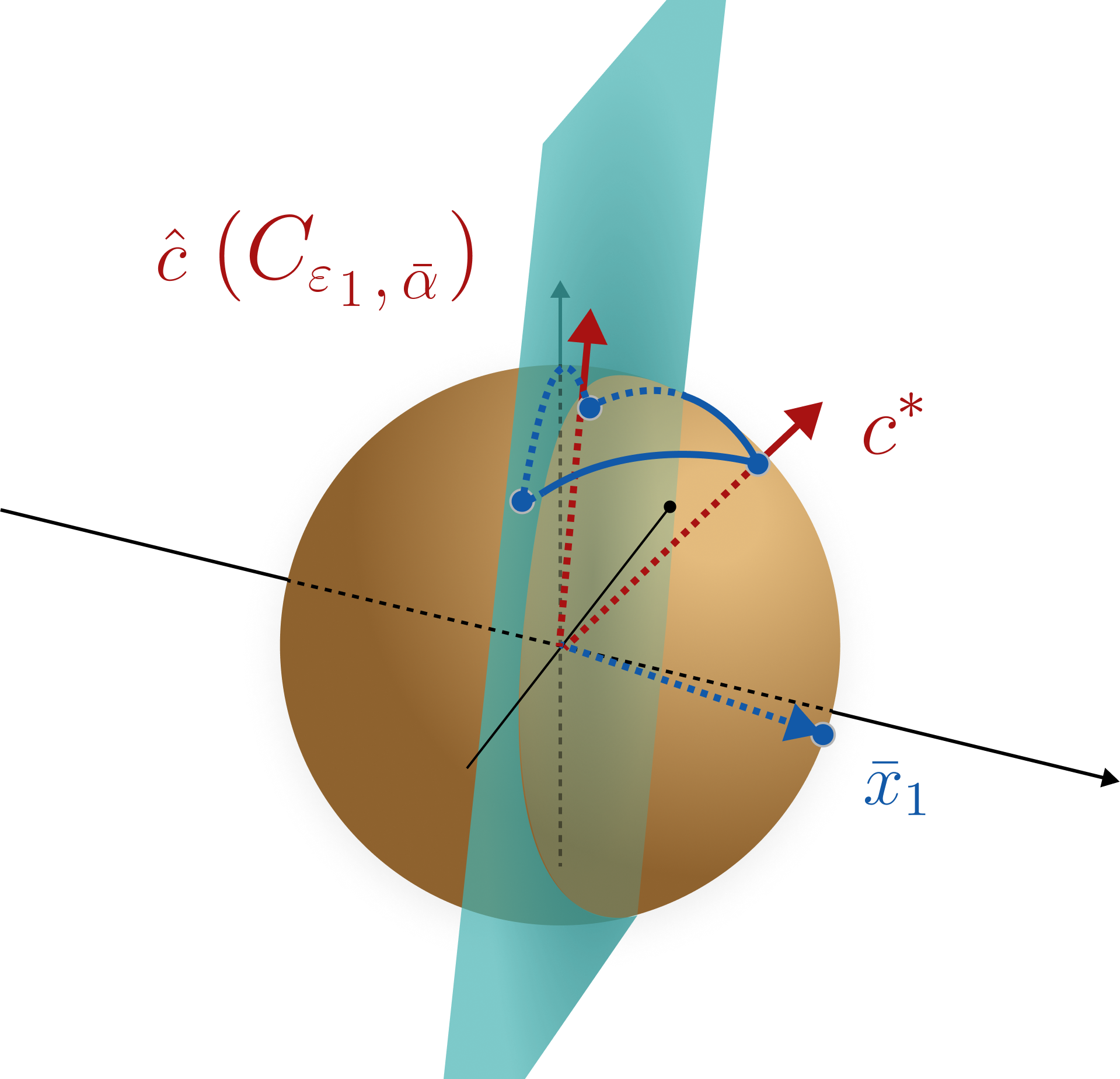}
    \caption{}\label{fig:cir_fail_2} 
  \end{subfigure}
    \hspace*{\fill}   % maximizeseparation between the subfigures
  \begin{subfigure}{0.32\textwidth}
    \includegraphics[width=\linewidth]{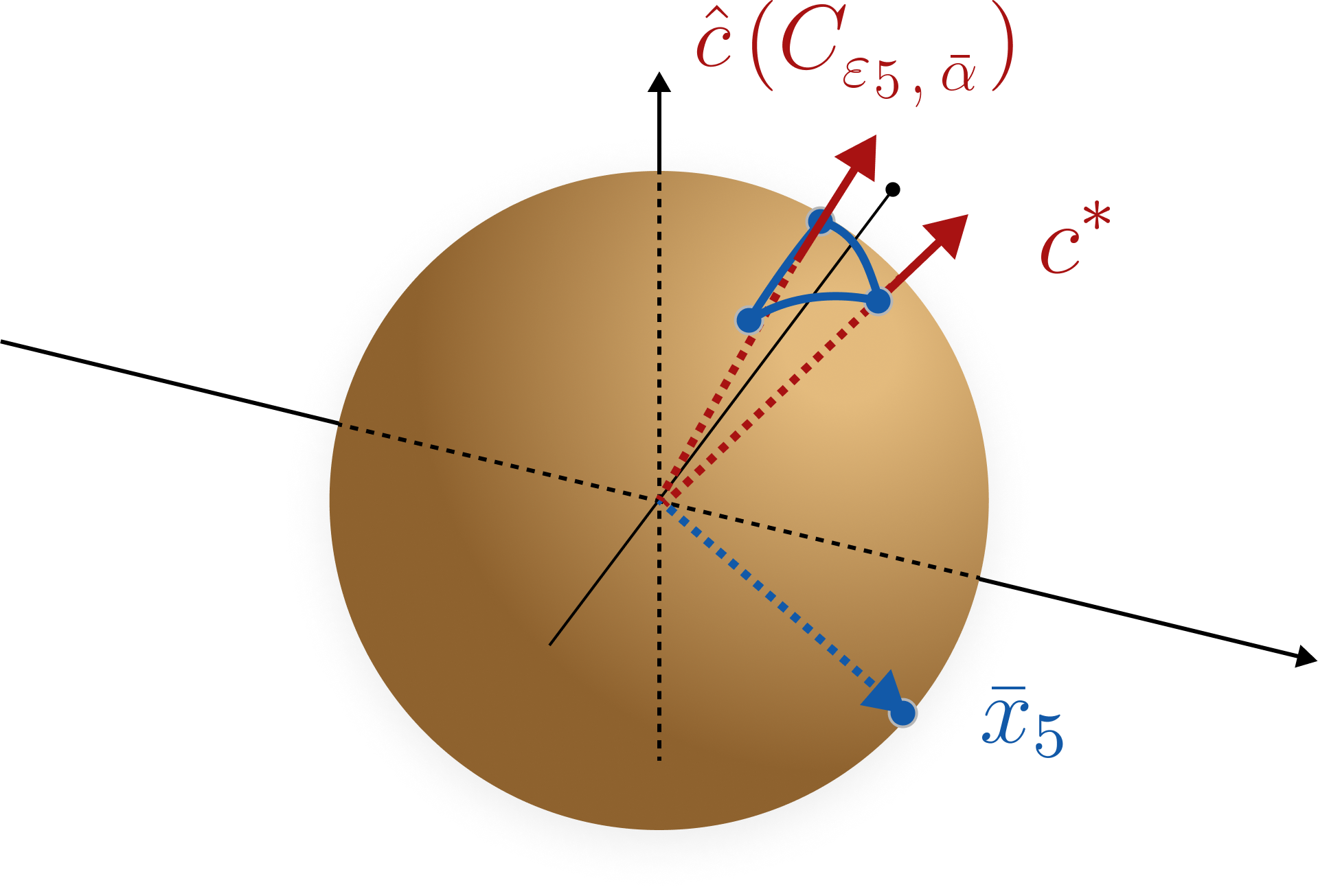}
    \caption{}\label{fig:circ_fail_final}
  \end{subfigure}
  \caption{In $(a)$ we have the initial knowledge set and the first instance of the optimization problem. In $(b)$, the updated set $C(\mathcal I_2)$, the new circumcenter $\hat c(C(\mathcal I_2))$, the true cost vector and the previous action $\bar x_1$. In $(c)$, the knowledge set after collecting the feedback of period $T$.}
\end{figure}

\end{proof}

\begin{proof}[\textbf{Proof of \Cref{lemma:ellipsoidal_cone_update}}] 

\Cref{lemma:ellipsoidal_cone_update} is a direct application of a more general result, described in \Cref{lemma:cone_update_general},  stated and proved in \Cref{app:aux}. In particular, if $t$ is a cone update period, we take $p = d$, $\eta = 0$, $\delta = \delta_t^\pi$ and $E(W,U) = E(W_t,U_t)$. We have that  ${\delta}'\hat c(E(W_t,U_t)) = {\delta_t^\pi}'c_t^\pi \leq 0$ where the inequality holds by definition of the effective difference $\delta_t^\pi$ (Eq. \ref{eq:feasible_feedback}) and the circumcenter policy applied to ellipsoidal cones.
\end{proof}

\begin{proof}[\textbf{Proof of \Cref{prop:angle_and_eigenvalue}}]
        
 Note that the circumcenter of $E(W,U)$ is $U e_1$. Let $E_H(W) = E(W) \cap H$, where $H = \{c \in \mathbb R^d:\; c'  e_1 = 1\}$. Next, recall that for any $c \in E(W) \setminus \{0\}$, the first component is always greater than zero. Then,
        $$
        \alpha(E(W,U)) \:=\: \sup_{c \in E(W,U)} \theta(c,Ue_1) = \sup_{c \in E(W)} \theta(c,e_1) = \sup_{c \in E(W)\setminus \{0\}} \theta(c,e_1) = \sup_{c \in E_H(W)} \theta(c,e_1), 
        $$
        where the second equality follows from the the fact that angles are preserved by orthonormal transformations, the third equality follows from the suboptimality of $\{0\}$ and the last equality follows from the fact that scaling a vector by a positive constant does not affect the angle. Moreover, for every $c \in E_H(W)$, we have that
        $$
        \tan \theta(c,e_1) = \tan \arccos \frac{c'e_1}{\|c\|} = \tan \arccos \frac{1}{\|c\|},
        $$
        where the first equality follows from the definition of an angle and the second one follows from $E_H(W) \subset H$. Hence,
        $$
        \tan \theta(c,e_1) = \sqrt{\|c\|^2-1} = \|c_{[2:d]}\|,
        $$
        where we used the trigonometric identity $\tan \arccos x = \frac{\sqrt{1-x^2}}{x}$ for $x \in [0,\pi/2)$. Using  the fact that $E_H(W)$ is an ellipsoid  one gets that $\|c_{[2:d]}\| \leq \sqrt {\lambda_{\max(M)}}$ and the inequality is tight for some $c \in E_H(W)$. Since $\tan (\cdot)$ is continuous and monotone increasing on $[0,\pi/2)$, we get
        
        $$
        \alpha(E(W,U)) = \sup_{c \in E_H(W)} \theta(c,e_1) = \sup_{c \in E_H(W)} \arctan \left(\tan \theta(c,e_1)\right) = \arctan \sqrt {\lambda_{\max(M)}}.
        $$
        This completes the proof.  
\end{proof}

\begin{proof}[\textbf{Proof of \Cref{lemma:epsilon}}]
    
Consider any $c \in E(W_t,U_t)\cap S^d$. Let $\tilde{c} = U_t^{-1}c$ and $\tilde{\delta}_t^{\pi} = U_t^{-1}\delta_t^{\pi}$. Recalling that for an orthonormal matrix $U$, $U^{-1} = U'$, the regret in period $t$ if the true underlying cost is $c$ is given by
	\begin{eqnarray*}
		{\delta^\pi_t}'c &=& {\delta^\pi_t}'U_t U_t^{-1}c \:=\:  {(U_t^{-1}{\delta^\pi_t})}' U_t^{-1}c \: = \: (\tilde{\delta}_t^{\pi})' \tilde{c}.
\end{eqnarray*}		
Note that $\tilde{c} \in E(W_t)\cap S^d$ and hence $\tilde{c}_{[1]}\in (0,1]$. Let $\tilde{\nu} =  \tilde{c}_{[2:d]}/\tilde{c}_{[1]}$ and note that $(1,\tilde{\nu}) \in E_H(W)$.  In turn, we have
\begin{eqnarray*}		
{\delta^\pi_t}'c &=& \tilde{c}_{[1]} \left(\tilde{\delta}^\pi_{t,[1]}+(\tilde{\delta}^\pi_{t,[2:d]})'\tilde{\nu} \right) \\
		&\leq& \tilde{c}_{[1]} \left(\tilde{\delta}^\pi_{t,[1]}+\sup_{(1,\nu)\in E_H(W)}(\tilde{\delta}^\pi_{[2:d]})'\nu\right) \\
		&\stackrel{(a)}{=}& \tilde{c}_{[1]} \left(\tilde{\delta}^\pi_{t,[1]}+\sqrt{(\tilde{\delta}^\pi_{t,[2:d]})'W_t\tilde{\delta}^\pi_{t,[2:d]}}\right) \\
		&\stackrel{(b)}{\leq}& \tilde{c}_{[1]} \left(\sqrt{(\tilde{\delta}^\pi_{t,[2:d]})'W_t\tilde{\delta}^\pi_{t,[2:d]}}\right) \\
		&\stackrel{(c)}{\leq}& \epsilon,
\end{eqnarray*}
where  $(a)$ follows from the fact that $E_H(W)$ is an ellipsoid and the optimization of a linear function over an ellipsoid has an analytical solution as given above (see, for instance, \cite{boyd2004convex}),  $(b)$ holds due to the fact that ${\tilde{\delta}}^\pi_{t,[1]} = (\tilde{\delta}^\pi)'e_1 = {\delta_t^\pi}' U_t e_1 = {\delta_t^\pi}' c_t^\pi \leq 0$ (see Eq. \eqref{eq:feasible_feedback}), and $(c)$ follows from $c_{[1]}^\star \leq 1$ and the assumption of the lemma. This completes the proof.

\end{proof}

\begin{proof}[\textbf{Proof of \Cref{lemma:stopping_rule}}]
 Let $\lambda_i(W)$ denote the $i$-th eigenvalue of $W$ in nondecreasing order. Suppose first that $\lambda_1(W_1) > \left( \frac{\epsilon}{10(d-1)}\right)^2$.  By \Cref{lemma:ellipsoidal_cone_update}, we have that that for every update period $t$, 
\bearn 
\prod_{i=1}^{d-1} \lambda_i(W_{t+1}) &\leq& e^{-1/(d-1)}\prod_{i=1}^{d-1} \lambda_i(W_t).
\eearn
For any time $t$ after exactly $I_T^{\pi}$ updates took place, we can apply the latter recursively to obtain
\bear \label{eq:ineqeigen1}
\prod_{i=1}^{d-1} \lambda_i(W_{t}) &\leq& e^{-I_T^{\pi}/(d-1)}\prod_{i=1}^{d-1} \lambda_i(W_1) \:\le \: e^{-I_T^{\pi}/(d-1)} (\lambda_{max}(W_1))^{d-1}.
\eear

Next we lower bound, for any $t$, the eigenvalues  $\lambda_i(W_{t})$. We will establish that  $\lambda_1(W_t) \ge  \left( \frac{\epsilon}{10d}\right)^2$ for all $t \ge 1$.

If $I_T^{\pi}=0$, then we have that $\lambda_1(W_t) = \lambda_1(W_1) > \left( \frac{\epsilon}{10(d-1)}\right)^2 
> \left( \frac{\epsilon}{10 d}\right)^2$.

Suppose now that $I_T^{\pi}>0$. We will show by induction  that  $\lambda_1(W_t) \ge  \left( \frac{\epsilon}{10d}\right)^2$ for all $t \ge 1$.

This clearly true for $t=1$. Suppose that it is true at time $s$. If there is no update, then this is trivially true for $s+1$. If there is an update $s$, two cases can happen. 

Case 1: $\lambda_1(W_s) > \left( \frac{\epsilon}{10(d-1)}\right)^2$. Note that the largest decrease possible in any eigenvalue after applying the ellipsoid method in $\mathbb{R}^{d-1}$ is $(d-1)^2/d^2$ (this is the decrease that happens when the cut is along that particular eigenvector). As a result, we must have $\lambda_1(W_{s+1}) \ge ((d-1)^2/d^2) \left( \frac{\epsilon}{10(d-1)}\right)^2 = \left( \frac{\epsilon}{10d}\right)^2$.

Case 2: $\lambda_1(W_s) \le \left( \frac{\epsilon}{10(d-1)}\right)^2$. In this case, given that updates are only performed when $\delta_{s,[2:d]}'W_s\delta_{s,[2:d]} > \epsilon^2$,  we can use \cite[Lemma 4]{cohen2016feature}. Indeed, in an update period, the second to the fourth update equations in algorithm \texttt{ConeUpdate} are precisely the update equations for the ellipsoid method for an ellipsoid in $\mathbb{R}^{d-1}$. \cite{cohen2016feature} introduce a version of the ellipsoid method where updates are not performed if the length of the ellipsoid along the direction to be cut is smaller than a certain threshold $\epsilon$. The condition in Eq. \eqref{eq:condition-update} for an update in our paper is essentially the same (our condition is on the distance from the center to the edge of the ellipsoid, but this is simply half the length).

Recall the notation used in algorithm \texttt{ConeUpdate} and note that every time that we update the matrix $W_t$ with the ellipsoid method, we get a matrix $N$. Let $\lambda_{1}(N)$ denote the smallest eigenvalue of the matrix $N$. By construction, if $t$ is an update period, we have that $\lambda_{1}(W_{t+1}) = \lambda_{1}(N)$. Lemma 4 from \cite{cohen2016feature} says that if we update our ellipsoidal cone according to algorithm \texttt{ConeUpdate}, if $\lambda_{1}(W_s) \leq \left(\frac{\epsilon}{10(d-1)}\right)^2$ and $\delta_{s,[2:d]}'W_s\delta_{s,[2:d]} > \epsilon^2$, then $\lambda_{1}(W_{s+1}) \geq \lambda_{1}(W_s)$, i.e., the smallest eigenvalue does not decrease after the update.

In this case, we deduce that $\lambda_1(W_{s+1}) \ge \lambda_1(W_s)\ge \left( \frac{\epsilon}{10d}\right)^2$, where the last inequality  follows from the induction hypothesis. This concludes the induction.

Combining this result with \eqref{eq:ineqeigen1}, we obtain
$$
\lambda_{max}(W_t) \leq \left(\frac{10d}{\epsilon}\right)^{2(d-2)} \left(\lambda_{max}(W_1)\right)^{(d-1)} e^{-I_T^\pi/(d-1)}.
$$

 Suppose now that $\lambda_1(W_1) \le \left( \frac{\epsilon}{10(d-1)}\right)^2$. By construction, $W_1$ is a revolution cone and $\lambda_i(W_1)=  \lambda_1(W_1)$ for $i=1,...,d-1$. For any vector $\delta \in \mathbb{R}^{d-1}$ with $\| \delta \| \le 1$, we have  
$\delta' W_1 \delta \le \|\delta\|^2 \lambda_{max}(W_1) \le  \left( \frac{\epsilon}{10(d-1)}\right)^2 \le \epsilon^2$. Hence, in this case, no update takes place, $I_T^{\pi} = 0$,  $W_t=W_1$ for all $t$, and by   \Cref{lemma:epsilon}, the per period regret is bounded above by $\epsilon$. This concludes the proof.
\end{proof} 

\begin{proof}[\textbf{Proof of \Cref{theorem:regret}}]
We prove \Cref{theorem:regret} in four steps. In the first step, we prove that $C_0 \subseteq E(W_1,U_1)$ and, that for every $t \ge 1$, $C(\mathcal{I}_t) \subseteq E(W_{t+1},U_{t+1})$. In the second step, we decompose the cumulative regret based on periods in which  the condition in Eq. \eqref{eq:condition-update} is satisfied or violated. In the third step, we upper bound the number of times that Eq. \eqref{eq:condition-update} can be violated, to obtain a bound on the cumulative regret. In the fourth step, we prove the running time claim.

\noindent
\textbf{Step 1.}  Note that $E(W_1,U_1)$, by construction, is precisely the revolution cone with aperture angle $\alpha(C_0)$ that contains $C_0$. The inclusions $C(\mathcal{I}_t) \subseteq E(W_{t+1},U_{t+1})$, $1 \leq t \leq T$ follows from our construction and  \Cref{lemma:ellipsoidal_cone_update}.

\noindent
\textbf{Step 2.} Recall that $I^{\pi}_T = \sum_{t=1}^T \mathbf{1} \{{\delta}'_{t,[2:d]}W_t \delta_{t,[2:d]} > \epsilon^2\}$ is the number of cone-update periods. We have the following regret bound.
	\begin{eqnarray}
        {\cal R}^{\pi}_T\left(c^\star,\vec{\cal X}_T,\vec f_T\right)
        &=& \sum_{t=1}^T \left(f_t(x_t^{\pi})-f_t(x^\star_t)\right)'c^\star \nonumber\\
        &=&  \sum_{t=1}^T {\|f_t(x^\pi_t)-f_t(x^\star_t)\|}{\delta_t^\pi}'c^\star  \nonumber \\
		&\stackrel{(a)}{\leq}& \sum_{t=1}^T \left(\mathbf{1} \{{\delta}'_{t,[2:d]}W_t \delta_{t,[2:d]} > \epsilon^2\} + \mathbf{1} \{{\delta}'_{t,[2:d]}W_t \delta_{t,[2:d]} \leq \epsilon^2\}\right){\delta_t^\pi}'c^\star \nonumber \\
		&\stackrel{(b)}{\le}& \sum_{t=1}^T  \mathbf{1} \{{\delta}'_{t,[2:d]}W_t \delta_{t,[2:d]} > \epsilon^2\}{\delta_t^\pi}'c^\star 
	 + \sum_{t=1}^T \mathbf{1} \{{\delta}'_{t,[2:d]}W_t \delta_{t,[2:d]} \leq \epsilon^2\}\epsilon \nonumber \\
	 &\stackrel{(c)}{\le}& \sum_{t=1}^T  \mathbf{1} \{{\delta}'_{t,[2:d]}W_t \delta_{t,[2:d]} > \epsilon^2\}
	 +  \sum_{t=1}^T\mathbf{1} \{{\delta}'_{t,[2:d]}W_t \delta_{t,[2:d]} \leq \epsilon^2\} \epsilon \nonumber \\
		&=&  I^{\pi}_T +   \left(T-I^{\pi}_T \right) \epsilon \nonumber\\
		&\leq& I^{\pi}_T + T\epsilon,\label{eq:i-t-pi}
	\end{eqnarray}
where $(a)$ follows from the fact that $\|f_t(x^\pi_t)-f_t(x^\star_t)\| \leq 1$ since $f_t \in \mathcal{F}$, $\Fset_t \in \mathcal{B}$, for all $t\leq T$, $(b)$ follows from \Cref{lemma:epsilon}, and $(c)$  follows from Cauchy-Schwarz inequality: ${\delta_t^\pi}'c^\star \leq \|\delta_t^\pi\|\|c^\star\| = 1$.

\noindent
\textbf{Step 3.} We now provide an upper bound for $I^{\pi}_T$. We need to consider  two separate cases as a function of $\alpha (C_0)$. 

If $\alpha(C_0) \leq \arctan \epsilon$, an application of \Cref{theorem:one_period_problem} shows that the regret of every period is less than $\sin \arctan \epsilon = \epsilon/(1+\epsilon)$ which is enough to ensure a performance of at least $\epsilon$ when using the circumcenter as the proxy cost vector. Moreover, \Cref{prop:angle_and_eigenvalue} and our choice of $E(W_1,U_1)$ in Algorithm \texttt{EllipsoidalCones} implies that $\lambda_{max}(W_1) \leq \epsilon^2$, and hence, for every possible vector $\delta$, $\delta_{[2:d]}'W_1\delta_{[2:d]} \leq \epsilon^2$. In this case, the algorithm never has a cone-update period and $I_T^\pi = 0$. Therefore, Eq. \eqref{eq:i-t-pi} implies that the regret is bounded by $T\epsilon$, which is equal to $d$ by our choice of $\epsilon$.

When $\alpha(C_0) > \arctan \epsilon$, we may have cone-update periods. We show that after an $\mathcal O (\log 1/\epsilon)$ amount of update steps, it must be the case that $\alpha(E(W_t,U_t)) < \arctan \epsilon$. Then, from this time onward, it must be the case that the algorithm never updates again and \Cref{lemma:epsilon} implies that the regret is upper bounded by $\epsilon$ for every period from this time onward. 

Suppose that we updated the ellipsoidal cone $\tau$ times. \Cref{lemma:stopping_rule} and the revolution cone initialization, 
$\lambda_1(W_1) = \cdots = \lambda_{d-1}(W_1) = \tan^2 \alpha(C_0)$ gives us that 
\begin{align*}
\lambda_{max}(W_t) &\leq \left(\frac{10d}{\epsilon}\right)^{2(d-2)} \left(\lambda_{max}(W_1) \right)^{d-1} e^{-\frac{\tau}{(d-1)}} = \left(\frac{10d}{\epsilon}\right)^{2(d-2)} (\tan \alpha(C_0))^{2(d-1)} e^{-\frac{\tau}{(d-1)}},
\end{align*}
for every $t$ after $\tau$ updates. Note that if we have $\tau = 2(d-1)^2\ln \left(\frac{10 d \tan \alpha(C_0)}{\epsilon }\right)$, then
\begin{align*}
\lambda_{max}(W_t) &\leq \left(\frac{10d}{\epsilon}\right)^{2(d-2)} (\tan \alpha(C_0))^{2(d-1)} e^{-\frac{2(d-1)^2\ln \left(\frac{10 d \tan \alpha(C_0)}{\epsilon }\right)}{(d-1)}} \\
& \leq \left(\frac{10d}{\epsilon}\right)^{2(d-1)-2}\left(\tan \alpha(C_0)\right)^{2(d-1)}  \left(\frac{10 d \tan \alpha(C_0)}{\epsilon }\right)^{-2(d-1)} \\
&\leq \left(\frac{10d}{\epsilon}\right)^{-2}.
\end{align*}
In this case, starting after $\tau$ updates, we are  as in case $ii.)$ of \Cref{lemma:stopping_rule}; there will be no  cone-update periods onward, and the per period regret will bounded by $\epsilon$ in future periods. In turn,

\begin{equation}\label{eq:upper_I_d}
I_T^\pi \leq 2(d-1)^2\ln \left(\frac{10 d \tan \alpha(C_0)}{\epsilon }\right).
\end{equation}
Combining Eqs. \eqref{eq:i-t-pi} and \eqref{eq:upper_I_d}, and selecting $\epsilon = d/T$ leads to:
\begin{eqnarray*}\label{eq:regret_case2}
        {\cal R}^{\pi}_T\left(c^\star,\vec{\cal X}_T,\vec f_T\right)  \leq   2(d-1)^2\ln \left(\frac{10 d \tan \alpha(C_0)}{\epsilon}\right) + T\epsilon 
\leq  2(d-1)^2\ln \left(10 T \tan \alpha(C_0)\right)+d.
\end{eqnarray*}

\noindent\textbf{Step 4.} This algorithm runs in polynomial time in $d$ and $T$ since every period's computation is a function only of $d$. Low-regret periods are computationally very cheap, while each cone-update period requires a spectral decomposition, which has a running time upper bound of $\mathcal O(d^3)$. 

The proof is complete.
\end{proof}

\begin{proof}[\textbf{Proof of \Cref{lemma:existence_subspace}}]
The subspace $\Delta_{t_0+1}$ is a one-dimensional object and, by the definition of $\Delta_{t_0+1}$, we have that for every element $c \in \Pi_{\Delta_{t_0+1}}\left(C(\mathcal{I}_{t_0+1})\right)$, $c = \gamma \delta^\pi_{t_0}$ for some $\gamma \in \mathbb R$. Moreover, for every $t \geq {t_0}$, Eq. \eqref{eq:relaxed_update} implies that we must have for every $c \in C(\mathcal{I}_{t})$ that ${\delta_{t_0}^\pi}'c \geq 0$. Since $\delta^\pi_{t_0} \in \Delta_{t_0+1}$,  we have that ${\delta_{t_0}^\pi}'c \geq 0 \iff {\delta_{t_0}^\pi}'\Pi_{\Delta_{t_0+1}}(c) \geq 0$. Hence,
$$
\Pi_{\Delta_{t_0+1}}\left(C(\mathcal{I}_{t_0+1})\right) \subseteq \{c \in \mathbb{R}^d:\; c = \gamma \delta_{t_0}^\pi, \; \gamma \geq 0\} \implies \alpha(\Pi_{\Delta_{t_0+1}}\left(C(\mathcal{I}_{t_0+1})\right)) = 0,
$$
and $\Pi_{\Delta_{t_0+1}}\left(C(\mathcal{I}_{t_0+1})\right)$ lives in a pointed cone.

The second affirmative follows directly from the fact that whenever $f_t(x_t^\pi) = f_t(x_t^\star)$,  there is no suboptimality gap and the regret is zero.
\end{proof}

\begin{proof}[\textbf{Proof of \Cref{lemma:stopping_rule_2}}]
For every time $t$, we have that
\begin{eqnarray*}
   \sup_{x^\star_t \in \psi(c^\star,\Fset_t,f_t)} \left(f_t(x_t^{\pi})-f_t(x^\star_t)\right)'c^\star &\stackrel{(a)}{\leq}&   \; {\delta_t^\pi}'c^\star = {r_t^\pi}'c^\star + {\Pi_{\Delta_t}(\delta_t^\pi)}'c^\star 
  \stackrel{(b)}{\leq}  \eta + {\Pi_{\Delta_t}(\delta_t^\pi)}'\Pi_{\Delta_t}(c^\star),
\end{eqnarray*}
where $(a)$ follows from the fact that $f_t \in \mathcal{F}$, $\Fset_t \in \mathcal{B}$ and $(b)$ follows from bounding ${r_t^\pi}'{c^\star}$ with $\eta$ since we assumed $\|r_t^\pi\| \leq \eta$, and replacing $c^\star$ with its projection onto $\Delta_t$.

Next, we show that ${\Pi_{\Delta_t}(\delta_t^\pi)}'\Pi_{\Delta_t}(c^\star) \leq \epsilon$ when the assumptions of the lemma are satisfied. For that, note that $\Pi_{\Delta_t}(\delta_t^\pi) \in \Delta_t$ and $E(W_t,U_t) \subset \Delta_t$. Moreover, we have by assumption that 
$\Pi_{\Delta_t}(c^\star) \in \Pi_{\Delta_t}(C(\mathcal I_t)) \subseteq E(W_t,U_t)$.
%Hence, an application of \newob{the arguments} in the proof of \Cref{lemma:epsilon} completes the proof. 
Hence, an application of \Cref{lemma:epsilon}
in the subspace $\Delta_t$ completes the proof.

\end{proof} 

\begin{proof}[\textbf{Proof of \Cref{lemma:properties_robust_update}}]

First we show that we do not exclude any feasible vector when we update the ellipsoidal cone using the projected effective difference. By the definition of the residual, we have that, for any $c \in C(\mathcal{I}_{t+1})$, that
$
{\delta^\pi_t}'c = {\Pi_{\Delta_{t+1}}(\delta^\pi_t)}'c+  {r^\pi_t}'c \geq 0,
$
where the inequality follows from Eq. (\ref{eq:relaxed_update}). Therefore, 
\begin{eqnarray*}
\Pi_{\Delta_{t+1}}\left(C(\mathcal{I}_{t+1})\right) &\subseteq& \Pi_{\Delta_{t+1}}\left(C(\mathcal{I}_t)\cap \{c \in \mathbb{R}^d:\; {\delta^\pi_t}'c \geq 0\} \right) \\ 
&\subseteq& \Pi_{\Delta_{t+1}}\left(C(\mathcal{I}_t)\cap \{c \in \mathbb{R}^d:\; {\Pi_{\Delta_t}(\delta^\pi_t)}'c \geq -\eta\} \right) \\ 
&=& \Pi_{\Delta_t}\left(C(\mathcal{I}_t)\cap \{c \in \mathbb{R}^d:\; {\Pi_{\Delta_t}(\delta^\pi_t)}'\Pi_{\Delta_t}(c) \geq -\eta\} \right) \\ 
&\subseteq& \Pi_{\Delta_t}\left(C(\mathcal{I}_t)\right)\cap\{c \in \Delta_t:\; {\Pi_{\Delta_t}(\delta^\pi_t)}'c \geq -\eta\} \\
&\subseteq& E(W_t,U_t)\cap \{c \in \Delta_t:\; {\Pi_{\Delta_t}(\delta^\pi_t)}'c \geq -\eta\},
\end{eqnarray*}
where the first inclusion follows from the fact that including elements in a set can only include elements in the projection of that set, the second inclusion follows from applying Chauchy-Schwarz for ${r^\pi_t}'c$ and using that $\|r_t^\pi\| \leq \eta$, the  equality follows the fact that $\Pi_{\Delta_t}(\delta^\pi_t) \in \Delta_{t}$, which implies that ${\Pi_{\Delta_t}(\delta^\pi_t)}'c = {\Pi_{\Delta_t}(\delta^\pi_t)}'\Pi_{\Delta_t}(c)$ and the fact that $\Delta_{t+1} = \Delta_t$, the third inclusion follows from the fact that the orthogonal projection of the intersection of sets must belong to the orthogonal projection of each of the sets and,  finally, the last inclusion follows from assumption of the lemma.

The next step is a direct application of \Cref{lemma:properties_robust_update}. For every cone update period $t$, we focus on the subspace $\Delta_t$ which has dimension $2 \leq p \leq d$. We represent the vectors in $\Delta_t$ under its basis representation by calculating $B_{\Delta_t}c \in \mathbb R^p$ for every $c \in \Delta_t^\pi$. Therefore, it suffices to apply \Cref{lemma:cone_update_general} with $p = p$, $\delta = B_{\Delta_t}\Pi_{\Delta_t}(\delta_t^\pi)$, $\eta = \eta$ and $E(W_t,U_t)$, where the matrices $W_t,U_t$ are already written in the basis representation of $\Delta_t$ by construction. 

The final step of this proof is to show that our choice of $\eta$ and $\delta$ satisfy the assumptions of the lemma, meaning that the shallow-cut is sufficiently deep to ensure a volume reduction. Since
$$
\Pi_{\Delta_t}(\delta_t)_{[2:p]}'W_{t}\Pi_{\Delta_t}(\delta_t)_{[2:p]} > \epsilon^2 \implies \frac{1}{\sqrt{\Pi_{\Delta_t}(\delta_t)_{[2:p]}'W_{t}\Pi_{\Delta_t}(\delta_t)_{[2:p]}}} < \frac{1}{\epsilon},
$$
we have that, for $\eta = \epsilon/2d$ and $1 < p \leq d$:
$$
\eta \leq \frac{\sqrt{\Pi_{\Delta_t}(\delta_t)_{[2:p]}'W_{t}\Pi_{\Delta_t}(\delta_t)_{[2:p]}}}{2d} \leq \frac{\sqrt{\Pi_{\Delta_t}(\delta_t)_{[2:p]}'W_{t}\Pi_{\Delta_t}(\delta_t)_{[2:p]}}}{2(p-1)} = \frac{\sqrt{(B_{\Delta_t}\Pi_{\Delta_t}(\delta_t)_{[2:p]})'W_{t}B_{\Delta_t}\Pi_{\Delta_t}(\delta_t)_{[2:p]}}}{2(p-1)} 
$$

The desired inclusion and the reduction in the product of the eigenvalues follows directly from \Cref{lemma:cone_update_general}.
\end{proof}

\begin{proof}[\textbf{Proof of \Cref{lemma:cone_updates_per_subspace}}]
In order to prove the result, we will follow the same strategy as in the proof of \Cref{lemma:stopping_rule}. For the boundary cases where $t = t_0^p$ for $2 \leq p \leq d$, the ellipsoidal cone $E(W_t,U_t)$ is constructed by a dimension update period, and we must have that $\lambda_1(W_{t_0^p}) = \lambda_{max}(W_{t_0^p})$.

Suppose $\lambda_1(W_{t_0^p}) > \left(\frac{\epsilon}{10p}\right)^2$. The shallow-cut equivalent of Lemma 4 from \cite{cohen2016feature} says that if we update our ellipsoidal cone according to Eq. \eqref{eq:update_rule_general_case}, if $\lambda_{1}(W_t) \leq \frac{\epsilon^2}{100(p-1)^2}$ and $\delta_{t,[2:p]}'W_t\delta_{t,[2:p]} > \epsilon^2$, then $\lambda_{1}(W_{t+1}) \geq \lambda_{1}(W_t)$, i.e., the smallest eigenvalue does not decrease after the update. Since the largest decrease possible in any eigenvalue after applying the ellipsoid method in $\mathbb{R}^{p-1}$ is $(p-1)^2/p^2$ (this is the decrease that happens when the cut is along that particular eigenvector), we have that for all $t_0^p < t < t_0^{p+1}$ (all time periods where the dimensionality of $\Delta_t$ is $p$) that
$$
\lambda_{1}(W_{t+1}) \geq \frac{(p-1)^2}{p^2}\frac{\epsilon^2}{100(p-1)^2} = \left(\frac{\epsilon}{10p}\right)^2.
$$
We omit the induction argument here since it mimics the proof of \Cref{lemma:stopping_rule}. Hence, 
$$
\lambda_{p-1}(W_t) \geq \left(\frac{\epsilon}{10p}\right)^2.
$$
Moreover, the shrinking factor of \Cref{lemma:properties_robust_update} gives us that
$$
\prod_{i=1}^{p-1}\lambda_i(W_{t+1}) \leq e^{-1/20(p-1)}\prod_{i=1}^{p-1}\lambda_i(W_t).
$$
Using the lower bound for the eigenvalues of $W_t$ gives us that
$$
\lambda_{max}(W_t) \leq \left(\frac{10p}{\epsilon}\right)^{2(p-2)} \left(\lambda_{max}(W_1)\right)^{(p-1)} e^{-\frac{I_T^{\pi,p}}{20(p-1)}},
$$
similarly to the proof of step 3 of \Cref{theorem:regret}, we must have that
$$
I_T^{\pi,p} \leq 20(p-1)^2\ln \left(\frac{10 p \tan \alpha(E(W_{t_0^p},U_{t_0^p}))}{\epsilon }\right),
$$
where $t_0^p$ is the first time such that the dimension of $\Delta_t = p$ and $E(W_{t_0},U_{t_0})$ the initial ellipsoidal cone.

 If $\lambda_1(W_{t_0^p}) \leq \left(\frac{\epsilon}{10p}\right)^2$, then noting that $\lambda_{max}(W_{t_0^p}) = \lambda_1(W_{t_0^p}) $, the same argument from \Cref{lemma:stopping_rule} holds and $I_T^{\pi,p} = 0$, while \Cref{lemma:stopping_rule_2} ensures a upper bound of $\epsilon$ for each period regret $t$ such that $\mbox{dim}(\Delta_t) = p$.
\end{proof}

\begin{proof}[\textbf{Proof of \Cref{lemma:gamma_pointed}}]

The proof is divided in three parts. First we show the inclusion, then we show how to compute the circumcenter, and finally we prove the bound for the uncertainty angle. 

\textbf{Inclusion.} The definition of the knowledge set  implies that for every $c \in C(\mathcal I_t)$, ${\delta_t^\pi}'c \geq 0$. In particular, this holds for all $i \in \tau(t)$. Moreover, by definition, $\delta_i^\pi \in \Delta_t$ for all $i \in \tau(t)$. Hence, for every $c \in C(\mathcal I_t)$, we have that ${\delta_i^\pi}'c = {\delta_i^\pi}'\Pi_{\Delta_t}(c) \geq 0$ for all $i \in \tau(t)$, which implies that $\Pi_{\Delta_t}(C(\mathcal I_t)) \subseteq K_t$. 

\textbf{Computation of circumcenter.} The computation of the circumcenter of $K_t$ is done via algorithm \texttt{PolyCenter} (see Algorithm \ref{algo:poly_center}). The subspace updating rule implies that $\delta_i^{\pi}$, $i \in \tau(t)$ are linearly independent. Hence, the system of equations described has one and only one solution for each iteration $k$. In addition, we have that $z$ is a vector in the interior of the cone $\{c \in \mathbb R^p:\; c = \sum \gamma_i \bar \delta_i, \; \gamma_i \geq 0, \; i \in \tau(t)\}$, which is the dual cone of $K_t$  \citep{boyd2004convex}. Hence, $c'z > 0$, for all $c \in K_t$ and $K_t \cap \{c \in \mathbb R^p :\; c'z \leq 1\}$ is bounded with extreme points given by the rays (not normalized) of $K_t$ and the origin. To see why the quadratic program in \texttt{PolyCenter} yields to the solution of the circumcenter, we refer to \cite{seeger2017measuring}. The algorithm runs in polynomial time in $p$ since it contains $p$ linear systems with $p$ equations and one quadratic programming formulation.

\textbf{Upper bound for the uncertainty angle.}

For any convex, closed and pointed cone $K \subset \mathbb R^d$, we define the circumradius of the cone $\mu(K)$ to be equal to sine of its uncertainty angle, i.e., $\mu(K) = \sin \alpha(K)$. In this proof, we will show that
\begin{equation}\label{eq:main-lemma11}\mu(K_{t+1}) {\leq} \sqrt{1-\frac{\eta^{2(p-1)}}{p^3}}.\end{equation}
Once we prove that Eq. \eqref{eq:main-lemma11} is true, it follows   that 
$$
\cos \alpha(K_{t+1}) = \sqrt{1 - \mu^2(K_{t+1})} \geq \frac{\eta^{(p-1)}}{p^{3/2}} \geq \frac{\eta^{(d-1)}}{d^{3/2}},
$$
where the second inequality follows from the facts that $0 < \eta < 1$ and $p \leq d$, completing the result.

Just as circumradius (and uncertainty angle) are defined by the smallest revolution cone that contains our cone of interest, we also need to define the largest revolution cone that fits within a cone of interest. For any convex, closed and solid cone $K$, we define the inradius of $K$ to be:
$$
\rho(K) = \max_{x \in S^d \cap K}\min_{y \in \partial K}\|x-y\|.
$$

We denote $x_{\rho} = \argmax_{x \in S^d \cap K}\min_{y \in \partial K}\|x-y\|$ as the incenter of $K$, which is analogous to the circumcenter, but referring to the axis of the largest revolution cone inside our cone of interest. The inradius and circumradius of a cone $K$ are dual quantities in the sense that for every closed convex cone, we have that:
\begin{equation}\label{eq:dual_inradius}
\mu^2(K) + \rho^2(K^\star) = 1,
\end{equation}
where $K^\star = \{c \in \mathbb{R}^d :\; c'x \geq 0, \; \forall x \in K\}$ denotes the dual cone of $K$  \cite[Theorem 1.4]{henrion2010inradius}. We will develop a lower bound on $\rho(K^\star_{t+1})$ and then obtain through  Eq. \eqref{eq:dual_inradius}  our desired upper bound on $\mu(K_{t+1})$.

The lower bound for $\rho(K^\star_{t+1})$ is obtained by an application of three lemmas that are interesting by its own that we prove in the appendix. The first one, \Cref{lemma:simplicial_cones}, shows that the inradius of $K^\star_{t+1}$ can be lower bounded by the ratio of the largest and smallest eigenvalue of the gram-matrix (the square matrix) constructed with it's generators. Suppose $\mbox{dim}(\Delta_{t+1}) = p \leq d$. We denote $g_i = B_{\Delta_{t+1}}\delta^\pi_{\tau(t+1)}(i) \in \mathbb R^p$, for $i = 1,...,p$, where we used the notation $\delta^\pi_{\tau(t+1)}(i)$ to denote the $i-th$ effective difference that belongs to $\tau(t+1)$. Let $G$ be the matrix such that its columns are given by $g_i$, $i = 1,...,p$. By construction, the columns of $g_i$ are linearly independent and we have $p$ vectors generating $K^\star_{t+1}$ that lives in a subspace with dimension $p$ and the assumptions of \Cref{lemma:simplicial_cones} are satisfied. The lemma gives us that
$$
\rho(K^\star_{t+1}) \geq \frac{1}{\sqrt p} \sqrt{\frac{\lambda_{max}(G'G)}{\lambda_{min}(G'G)}}.
$$

The second lemma, \Cref{lemma:cond_number}, allows us to provide an upper-bound for $\frac{\lambda_{max}(G'G)}{\lambda_{min}(G'G)}$. Note that $G$ is full column rank with unit norm vectors by construction (the change of basis through $B_{\Delta_{t+1}}$ do not affect the norm of the effective differences vectors), and the assumption of the lemma is satisfied. Let $\Pi_{g_{-i}}(.)$ denote the projection operator on the subspace generated by $\{g_1,\ldots,g_p\} \setminus \{g_i\}$. Define 
$$
\varphi = \min_{i\leq p} \|g_i-\Pi_{g_{-i}}(g_i)\|,
$$
and note that $\phi$ is the minimum norm obtained by regressing the column $g_i$ on every other columns. \Cref{lemma:cond_number} implies that
$$
\frac{\lambda_{max}(G'G)}{\lambda_{min}(G'G)} \leq \left(\frac{p}{\varphi}\right)^2.
$$

The final step in our proof is to show that $\varphi$ cannot be arbitrarily small. In \Cref{lemma:FWL}, we show that if the sequence of the $g_i$'s satisfy $\|g_i-\Pi_{g_1,...,g_{i-1}}(g_i)\|^2 \geq \eta^2$, (which is true due to the subspace updating rule) then it must be the case that $\|g_i-\Pi_{g_{-i}}(g_i)\| \geq \eta^{2(p-i)}$ for $i = 1$ or greater than $\eta^{2(p-i)}$ if $i > 1$. Taking $i = 1$ (or $i = 2$) ensures that $\|g_i-\Pi_{g_{-i}}(g_i)\| \geq \eta^{2(p-1)}$ for every $i$. Then, the combination of the three lemmas establishes that the inradius of $K^\star_{t+1}$ is lower bounded by $\frac{\eta^{p-1}}{p^{3/2}} \geq \frac{\eta^{d-1}}{d^{3/2}}$, concluding the proof of \Cref{lemma:gamma_pointed}.  

\end{proof}

\begin{proof}[\textbf{Proof of \Cref{theorem:regret_general_case}}]
We prove the result in four steps. First, we establish that $\Pi_{\Delta_t}(C(\mathcal{I}_t)) \subseteq E(W_t,U_t)$ for every $t$, and thus we never lose track of the true cost. In a second step,  we decompose the cumulative regret as a function of the different kinds of periods in algorithm \texttt{ProjectedCones} and provide an upper bound for the one-period regret under each case. In the third step, we upper bound the number of periods that we use cone-updates to obtain our regret bound. In the fourth step, we prove the polynomial runtime.

\textbf{Step 1.} We show that, for each time $t$, the set $E(W_t,U_t)$ contains $\Pi_{\Delta_{t}}\left(C(\mathcal{I}_{t})\right)$. We establish the result by induction on $t$. $E(W_2,U_2)$ trivially contains $\Pi_{\Delta_2}\left(C(\mathcal{I}_2)\right)$, so the base case is satisfied. We next consider $t \ge 2$. Suppose that $E(W_t,U_t)$ is such that $\Pi_{\Delta_t}\left(C(\mathcal{I}_t)\right) \subseteq E(W_t,U_t)$. We then analyze $E(W_{t+1},U_{t+1})$ as a function of the three situations that can happen at time $t$, no update (low-regret period), cone update, or subspace update.

\textit{No update.} This case is trivial since $C(\mathcal I_{t+1}) \subseteq C(\mathcal I_t)$, $E(W_{t+1},U_{t+1}) = E(W_t,U_t)$, and $\Delta_{t+1} = \Delta_t$, thus $\Pi_{\Delta_{t+1}}\left(C(\mathcal{I}_{t+1})\right) \subseteq E(W_{t+1},U_{t+1})$.

\textit{Cone update.} This inclusion follows from \Cref{lemma:properties_robust_update}.

\textit{Subspace update.} \Cref{lemma:gamma_pointed} shows that $\Pi_{\Delta_{t+1}}(C(\mathcal I_{t+1}))$ is contained in $K_{t+1}$. Our choices of $W_{t+1}$ and $U_{t+1}$ ensure that $K_{t+1}$ is included in $E(W_{t+1},U_{t+1})$. The same lemma also shows that the constructed ellipsoidal cone is large enough to contain $K_{t+1}$.

\textbf{Step 2.} \Cref{lemma:stopping_rule_2} shows that we incur at most regret $\epsilon + \eta$ in low-regret periods. Therefore, our total regret from low-regret periods is bounded by $T(\epsilon + \eta)$. For all other periods, we use the trivial regret upper bound of 1. There are at most $d$ subspace-update periods, so the total regret from subspace-update periods is bounded by $d$. Let $I^{\pi,p}_T$ be the number of periods where the subspace has dimension $p$ and we use a cone-update. 
Bounding the regret of these cone-update periods by 1 as well, we have that for any $c^\star \in S^d$, $f_t \in \mathcal{F}$, and $\Fset_t \in \mathcal{B}$, the total regret is bounded by:
\begin{eqnarray}\label{eq:regret_decomp}
    {\cal R}^{\pi}_T\left(c^\star,\vec{\cal X}_T,\vec f_T\right) 
   \leq T(\epsilon + \eta) + d +  \sum_{p=2}^d I^{\pi,p}_T,
\end{eqnarray}
where the last sum starts from $p=2$ because there are never cone-updates when $p=1$. %In the last step,

\textbf{Step 3.} \Cref{lemma:cone_updates_per_subspace} shows that for any $p=2,...,d$, we have $$I_T^{\pi,p} \leq 20(p-1)^2\ln \left(\frac{10 p \tan \alpha (E(W_{t_0},U_{t_0}))}{\epsilon}\right),$$ where $t_0$ refers to the period where the subspace was increased to $p$. For simplicity, we replace $p-1$ and $p$ with the larger value $d$: $I_T^{\pi,p} \leq 20d^2\ln \left(\frac{10 d \tan \alpha (E(W_{t_0},U_{t_0}))}{\epsilon}\right)$. At period $t_0$, the subspace update constructs a revolution cone such that $\alpha(E(W_{t_0},U_{t_0})) = \arccos \eta^{d-1}/d^{3/2}$ (see Lemma \ref{lemma:gamma_pointed}). Since $\tan(x) \leq 1/\cos(x)$, we have that $\tan \alpha(E(W_{t_0},U_{t_0})) \leq d^{3/2}/\eta^{d-1}$. Summing over all $p$:
\[ \sum_{p=2}^d I_T^{\pi,p}  \leq 20d^3\ln \left(\frac{10 d^{5/2}}{\epsilon \eta^{d-1}}\right).\]
Plugging the bound above into Eq. \eqref{eq:regret_decomp} and selecting $\epsilon = d/T$, $\eta = \epsilon/2d$ leads to:
\begin{eqnarray*}
{\cal R}^{\pi}_T\left(c^\star,\vec{\cal X}_T,\vec f_T\right)  &\leq& 
 T(\epsilon+\eta)+d + {20 d^3}\ln\left(\frac{10d^{5/2}}{\epsilon \eta^{d-1}}\right)\\
&=& d+2 + d + {20 d^3}\ln(5 d^{3/2} T^d 2^{d}) \\
&=& 2 + 2d + {20 d^3} \ln 5 + {30 d^3}\ln(d) + {20 d^4}\ln(2 T) = \mathcal O(d^4 \ln T).\end{eqnarray*}

\noindent \textbf{Step 4.} This algorithm runs in polynomial time in $d$ and $T$ because every period's computation is a function only of $d$. Low-regret periods are computationally very cheap. Each cone-update period requires a spectral decomposition which has a running time upper bound of $\mathcal O(p^3)$. Each subspace-update period requires the computation of a new circumcenter via algorithm \texttt{PolyCenter}, which   runs in polynomial time in $d$ (cf. \Cref{lemma:gamma_pointed}). This completes the proof. 
\end{proof}

%\newpage

\subsection{Auxiliary Results} \label{app:aux}
%\textbf{Ellipsoidal Cone Updates}

\begin{lemma}[Robust ellipsoidal cone updates]\label{lemma:cone_update_general}
    Consider a diagonal and positive-definite matrix $W \in \mathbb D_{++}^{p-1}$ and an orthonormal matrix  $U \in \mathbb R^p \times \mathbb R^p$ and define the ellipsoidal cone $E(W,U) \subset \mathbb R^p$. Fix $ \eta  \ge 0$ and a vector $\delta \in \mathbb R^p$ such that   $\eta \leq \sqrt{{\delta}'_{[2:p]}W{\delta}_{[2:p]}} (2(p-1))^{-1}$ and   $\delta'\hat c(W,U) \leq 0$. Define
    \begin{align*}
            &\bar \delta = U^{-1}\delta_{[2:p]}/\|U^{-1}\delta_{[2:p]}\|, \; \beta = -\frac{\eta}{\sqrt{{\bar \delta}'W{\bar \delta}}}, \; b = \frac{W{\bar \delta}}{\sqrt{{{\bar \delta}}'W{\delta}}}, \; a = \frac{1+(p-1)\beta}{p}b,
    \end{align*}
    and 
    \begin{align*}
        N = \frac{(p-1)^2}{(p-1)^2-1}(1-\beta^2)\left(W-\frac{2(1+(p-1)\beta)}{p(1+\beta)}bb'\right), \; M =  \begin{pmatrix} 1 & a' \\ a & aa' - N \end{pmatrix}.
    \end{align*}
    Let $V \Lambda V'$ denote the spectral decomposition of the matrix $M$. Then $$E(W,U) \cap \{c \in \mathbb R^d \;:\; \delta'c \geq -\eta\} \subseteq E(\widetilde W, \widetilde U)$$ for 
    $\widetilde U = UV,$ and $\widetilde W$ is a diagonal matrix such that $\; \widetilde W_{i,i} = \lambda_{i}(N), \; i = 1,\cdots, d-1.$, where the eigenvalues are in a nonincreasing order. Moreover, if $\eta = 0$, $\prod_{i=1}^{d-1}\lambda_i(\widetilde W) \leq e^{-1/2(d-1)}\prod_{i=1}^{d-1}\lambda_i(W)$. Otherwise $\prod_{i=1}^{d-1}\lambda_i(\widetilde W) \leq e^{-1/20(d-1)}\prod_{i=1}^{d-1}\lambda_i(W)$.
\end{lemma}

\begin{proof}[\textbf{Proof of \Cref{lemma:cone_update_general}}] We first consider the case where $U$ is the identity matrix, so that $E(W,U) = E(W)$ and $\bar \delta = \delta$. The proof strategy is as follows. In the first step, we show that  $E(W) \cap \{c\in \mathbb R^d: \; \delta'c \geq 0\}$ is contained in a half-ellipsoidal cone. Second, we show that we can use a variation of the ellipsoid method to update this half-ellipsoidal cone and construct an appropriate updated ellipsoidal cone. In a third step, we apply a theorem from \cite{seeger2017measuring} in order to characterize the obtained ellipsoidal cone in terms of a standard-position cone and its orthonormal rotation.

\textbf{Step 1.} Let us consider the intersection of $E(W)$ and a hyperplane characterized by its circumcenter $\hat c(E(W)) = e_1$. We denote such a hyperplane as $H = \{c \in \mathbb R^d:\; e_1'c = 1\}$. Since $c_{[1]}>0$ and $E(W)$ is a cone, we can scale any element $c \in E(W)$ such that the scaled vector lies in the intersection. Moreover, the intersection will be an ellipsoid given by 
$$
E_H(W) = \{c \in \mathbb R_+ \times \mathbb R^{d-1}:\;c_{[2:d]}'W^{-1}c_{[2:d]} \leq 1, \; c_{[1]} = 1\},
$$ 
which is precisely the equation of an ellipsoid in $\mathbb R^{d-1}$ that lives in a $(d-1)$- dimensional subspace of $\mathbb R^d$. The ellipsoid $E_H(W)$ defined above was obtained by a specific type of projection known in the literature as the perspective projection of $E(W)$ onto the hyperplane $H$. 

Let us define also $A = E(W)\cap\{c \in \mathbb{R}^d:\; {{\delta}}'c \geq -\eta\}$. Next, consider some $c \neq 0 \in A$. Since $A \subseteq E(W)$, we must have $c_{[1]}>0$. Moreover, $A$ is a cone since it is the intersection of two cones. Then,  $c/c_{[1]} \in A$ and $c/c_{[1]} \in H$. Hence,
	\begin{eqnarray*}
	c/c_{[1]} &\in& A \cap H \\
	&=& E(W) \cap H \cap \{c \in \mathbb{R}^p:\; \delta'c \geq -\eta\}  \\
	&=& E_H(W) \cap \{c \in \mathbb{R}^d:\; {\delta}'c \geq -\eta\} \\
	&\stackrel{(a)}{=}& E_H(W) \cap \{c \in \mathbb{R}^d:\; \delta_{[1]} + {\delta_{[2:d]}}'c_{[2:d]} \geq -\eta\} \\
	&\stackrel{(b)}{\subseteq}& E_H(W) \cap \{c \in \mathbb{R}^d:\; {\delta_{[1]}+\delta_{[2:d]}}'c_{[2:d]}  \geq -\eta + \delta'\hat c(E(W))\} \\
	&\stackrel{(c)}{\subseteq}& E_H(W) \cap \{c \in \mathbb{R}^d:\; {\delta_{[2:d]}}'c_{[2:d]}  \geq -\eta\},
	\end{eqnarray*}
	where $(a)$ follows from the fact that $E_H(W) \subset H$ so $c_{[1]}$ = 1. $(b)$ follows from the fact that $\hat c(E(W))'\delta \leq 0$ by assumption and $(c)$ follows from the fact that = $\hat c(E(W))'\delta = e_1'\delta = \delta_{[1]} \leq 0$. The set of the last equation is precisely the ellipsoid $E_H(W)$ with a shallow-cut, and hence is a description of a half-ellipsoid.
	
\textbf{Step 2.} We now use the ellipsoid method update to replace this half-ellipsoid with its own L\"owner-John ellipsoid. The definitions of $a$, $N$ and $\beta$ are precisely the ones for an ellipsoid update with shallow-cut (See Eq. (3.1.16) and Eq. (3.1.17) of \cite{grotschel1993ellipsoid}). Therefore,
	$$
	E_H(W) \cap \{c \in \mathbb{R}^d:\; \bar \delta_{[2:d]}'c_{[2:d]} \geq  -\eta \} \subseteq \{c \in \mathbb{R}^d:\; (c_{[2:d]}-a)'N^{-1}(c_{[2:d]}-a) \leq 1, \; c_{[1]} = 1\}.
	$$
	Furthermore, since $c_{[1]} > 0$, we have that 
	$$
	c/c_{[1]} \in \{c \in \mathbb{R}^d:\; (c_{[2:d]}-a)'N^{-1}(c_{[2:d]}-a) \leq 1\} 
	$$
	if and only if
	$$
	c \in \{c \in \mathbb{R}^d:\; (c_{[2:d]}-c_{[1]}a)'N^{-1}(c_{[2:d]}-c_{[1]}a) \leq c^2_{[1]}\}.
	$$
	Hence, if $c \in A$, then we must have that $c \in \{c \in \mathbb{R}^d:\; (c_{[2:d]}-c_{[1]}a)'N^{-1}(c_{[2:d]}-c_{[1]}a) \leq c^2_{[1]}\}$.
	
	We finish the second step by showing the contraction in the product of the eigenvalues of the matrix $N$ (or equivalently, its volume). 
	
	Note that if $\eta = 0$, the update equations reduces to the standard update equations of the ellipsoid method, then the standard volume reduction of the holds and we get $\prod_{i=1}^{p-1} \lambda_i(N) \leq \prod_{i=1}^{p-1} \lambda_i(W)e^{1/2(p-1)}$. 
	
	If $\eta > 0$, we need to argue that the shallow cuts are still sufficiently deep to induce a reduction in the product of the eigenvalues
	first need to ensure that $\beta \geq 1/(p-1)$ in order to have $\prod_{i=1}^{p-1} \lambda_i(N) \leq \prod_{i=1}^{p-1} \lambda_i(W)e^{(1+\beta (p-1))/5(p-1)}$ (See Eq. (3.3.21) of \cite{grotschel1993ellipsoid}). Since in this case
$$
0 < \eta \leq \frac{\sqrt{{\delta}'_{[2:p]}W{\delta}_{[2:p]}}}{2(p-1)}
$$
we have that, 
$$
\beta  = -\frac{\eta}{\sqrt{{\delta}'_{[2:p]}W{\delta}_{[2:p]}}} \geq -\frac{1}{2(p-1)},
$$
substituting the lower bound of $\beta$ above leads to $\prod_{i=1}^{p-1} \lambda_i(N) \leq \prod_{i=1}^{p-1} \lambda_i(W)e^{1/20(p-1)}$.
	
\textbf{Step 3.} We have now constructed an ellipsoidal cone  $\{c \in \mathbb{R}^d:\; (c_{[2:d]}-c_{[1]}a)'N^{-1}(c_{[2:d]}-c_{[1]}a) \leq c^2_{[1]}\}$ that contains the half ellipsoidal cone of interest, which is a cone not in standard position. However, instead of having the set described by a rotation orthonormal basis, we have it described via a translation of the center of the ellipsoid at $H$. In the remainder of this proof, we show how to construct a representation of this ellipsoidal cone that is consistent with our Definition \ref{def:ellipsoidal_cone}. That is, we need to find a mapping from the parameters $N$ and $a$ to the matrices $W$ and $U$. 
	
	To find this mapping, we apply a theorem from \cite[Theorem 4.4 page 296]{seeger2017measuring}. The theorem states that for 
	$$
	M =  \begin{bmatrix} 1 & a' \\ a & aa'- N \end{bmatrix},
	$$
	the matrix $M$ is invertible and the spectral decomposition $M =  V \Lambda V'$ allows us to write the ellipsoidal cone $E(\hat W, \hat U)$ that is identical to $\{c \in \mathbb{R}^d:\; (c_{[2:d]}-c_{[1]}a)'N^{-1}(c_{[2:d]}-c_{[1]}a) \leq c^2_{[1]}\}$, but in standard representation. This ellipsoidal cone is given by $\hat U = V$
	and $\hat W$ equal to the diagonal matrix with diagonal entries $\hat W_{ii}$ given by $-\lambda_{i+1}(M)/\lambda_{1}(M)$, for $i = 1,\cdots, d-1$, where the eigenvalues of $M$ are in nonincreasing order. 
	
	For our ellipsoidal update, we will use the same orthonormal rotation as the one produced by Seeger and Vidal-Nu\~nez's theorem: $\widetilde U = \hat U$. At this point, we could declare the proof done if we had also defined in the algorithm $\widetilde W$ to be equal to $\hat W$. However, in order to facilitate our analysis, we use a different choice of $\widetilde W$ by setting $\widetilde W_{ii} = \lambda_{p-i}(N)$ for all $i=1,...,p-1$. Hence, in order to conclude the proof, we still need to show that our choice of $\widetilde W$ has eigenvalues that are at least as large as $\hat W$. To be specific, we need to prove that for all $i=1,...,p-1$, $\lambda_i(\hat W) = -\lambda_{i+1}(M)/\lambda_{1}(M) \leq \lambda_{i}(N) = \lambda_{i}(\widetilde W)$. Our definitions of $M$, $N$ and $a$ satisfy:
$$
M = \begin{bmatrix}
  0 & 0 \\ 0 & -N
\end{bmatrix}
+
\begin{bmatrix}
  1 \\  a
\end{bmatrix}
\begin{bmatrix}
  1 \\  a
\end{bmatrix}'.
$$
Since the matrix $N$ is obtained by one iteration of the ellipsoid method, we must have that $N$ has strictly positive eigenvalues. Hence, the interlacing theorem \citep{hwang2004cauchy} implies that $M$ has one positive eigenvalue and $d-1$ negative eigenvalues. %This implies that $M^{-1}$ also has $d-1$ negative eigenvalues and one positive eigenvalue.

Furthermore, we have that $\lambda_{max}(M) = \lambda_{1}(M) = \sup_{c \in S^d} c'Mc$. Since $c = e_1$ satisfies $c'Mc = 1$, we must have that $\lambda_{1}(M) \geq 1$. Finally, another application of the interlacing theorem for $M$ gives us that for $i = 1, \cdots, d-1$, satisfies $\lambda_{i+1}(M) \geq \lambda_i(-N)$, and we get that
$$
-\lambda_{i+1}(M)/\lambda_1(M) \leq -\lambda_{i+1}(M) \leq -\lambda_{i}(N) = \lambda_{p-i}(N) = \lambda_{i}(\widetilde W), \; i = 1,\cdots, p-1.
$$
Hence, the choice of $\widetilde W$ in algorithm \texttt{ConeUpdate} is at least as large as necessary, which concludes the proof for the case where $U$ is the identity matrix. 

For the general case (when $U$ is not the identity matrix), it suffices to rotate $\delta$ by $U^{-1}$, which is equivalent to analyze the problem under the basis representation given by the rows of $U$. After that, we return to the canonical basis by rotating $V$ by $U$, which leads to $\widetilde U = UV$.
\end{proof}

\begin{lemma}[Inradius of simplicial cones]\label{lemma:simplicial_cones}
Fix $p \geq 2$ and let $K\subset \mathbb R^p$ denote a simplicial cone, i.e., it can be written as $K = \{c \in \mathbb R^p:\; c = \sum_{i = 1}^p \alpha_i g_i, \; \alpha_i \geq 0\}$, where the $g_i$'s are linearly independent unit norm vectors. Define $G \in \mathbb R^p \times \mathbb R^p$ to be the matrix where the columns are the generators $g_i$, $i = 1,\cdots, p$. We have that the inradius $\rho(K) := \max_{x \in S^d \cap K}\min_{y \in \partial K}\|x-y\|$ is lower bounded as follows

$$
\rho(K)  \geq \frac{1}{\sqrt{p}} \sqrt{\frac{\lambda_{min}(G'G)}{\lambda_{max}(G'G)}}.
$$
\end{lemma}

\begin{proof}[\textbf{Proof of \Cref{lemma:simplicial_cones}}]

We denote $x_{\rho} = \argmax_{x \in S^d \cap K}\min_{y \in \partial K}\|x-y\|$ as the incenter of $K$, which is analogous to the circumcenter, but referring to the axis of the largest revolution cone inside our cone of interest. Let $P$ denote the nonnegative orthant in $\mathbb{R}^p$. One can check that $\rho(P) = \frac{1}{\sqrt{p}}$ and the incenter of $P$ is given by $x_{\rho} = \mathbf{1}\frac{1}{\sqrt{p}}$ (see, for instance, \cite{henrion2010inradius}). We have

\begin{eqnarray*}
\rho(K) &\stackrel{(a)}{=}& \max_{\pi \in S^d\cap K}\min_{\delta \in \partial(K)} \|\pi-\delta\| \\
&\stackrel{(b)}{=}& \max_{\|Gx\| = 1, x \in P}\;\min_{Gy \in \partial(K)} \|Gx-Gy\| \\
&\stackrel{(c)}{=}& \max_{\|Gx\| = 1, x \in P}\;\min_{y \in \partial(P)} \|Gx-Gy\| \\
&\stackrel{}{=}& \max_{\|Gx\| = 1, x \in P}\;\min_{y \in \partial(P)} \|G(x-y)\| \\
&\stackrel{}{=}& \max_{\|Gx\| = 1, x \in P}\;\min_{y \in \partial(P)} \sqrt{(x-y)'G'G(x-y)} \\
&\stackrel{}{\ge}& \max_{\|Gx\| = 1, x \in P}\;\min_{y \in \partial(P)} \|x-y\| \sqrt{\lambda_{min}(G'G)} \\
\end{eqnarray*}
$(a)$ follows from the definition of inradius, $(b)$ follows from the fact that $\pi \in K \iff \pi = Gx$ for some $x \in P$ since $K$ is defined by the nonnegative linear combination of the columns of $G$, and by the fact that $G$ is full rank by assumption, so $y = G^{-1}\delta$ is well defined. $(c)$ holds since $G$ is linear and establishes a bijection between the generators of $P$ and $K$. 
The  inequality follows from the fact that $G'G$ is positive definite (which follows from $G$ being full column-rank). 
 
 Moreover, since $G$ has full column-rank, we have that $Gx = 0$ if and only if $x = 0$, implying that $x = \frac{1}{\|G\mathbf{1}\|}\mathbf{1}$ has only positive components and is well defined since $\|G\mathbf{1}\| \neq 0$. Thus, $x = \frac{1}{\|G\mathbf{1}\|}\mathbf{1}$ is feasible for the maximization problem presented above and we have
 \begin{eqnarray*}
\rho(K) &\stackrel{(a)}{\ge}& \sqrt{\lambda_{min}(G'G)} \min_{y \in \partial(P)} \left\|\frac{1}{\|G\mathbf{1}\|}\mathbf{1}-y\right\|  \\
&\stackrel{}{=}& \frac{\sqrt{p}}{\|G\mathbf{1}\|} \sqrt{\lambda_{min}(G'G)} \min_{y \in \partial(P)} \left\|\frac{1}{\sqrt{p}}\mathbf{1}-y\frac{\|G\mathbf{1}\|}{\sqrt{p}}  \right\|  \\
&\stackrel{(b)}{=}& \frac{\sqrt{p}}{\|G\mathbf{1}\|} \sqrt{\lambda_{min}(G'G)} \min_{y \in \partial(P)} \left\|\frac{1}{\sqrt{p}}\mathbf{1}-y \right\|  \\
&\stackrel{(c)}{\geq}& \frac{\sqrt{p}}{\|G\mathbf{1}\|}   \sqrt{\lambda_{min}(G'G)} \frac{1}{\sqrt{p}}\\
&=& \frac{1}{\|G\mathbf{1}\|}   \sqrt{\lambda_{min}(G'G)},
\end{eqnarray*}
where $(a)$ follows from the fact that $\frac{1}{\|G\mathbf{1}\|}$ is feasible for the maximization problem, $(b)$ follows from the fact that we can always scale by a positive constant our choice for $y$ in the minimization problem, and $(c)$ follows from the fact that the minimization problem over $y$ is lower bounded by the definition of the inradius of the nonnegative orthant.

In addition, we have $\|G\mathbf{1}\|^2 = \mathbf{1}' G' G\mathbf{1} \le \|\mathbf{1}\|^2 \lambda_{max}(G' G) = p \lambda_{max}(G' G)$. In turn, we have
\begin{eqnarray*}
\rho(K) &\stackrel{}{\ge}& \frac{1}{\sqrt{p}} \sqrt{\frac{\lambda_{min}(G'G)}{\lambda_{max}(G'G)}}.
\end{eqnarray*}
 
\end{proof}

\begin{lemma}[Condition number of Gram-Matrices]\label{lemma:cond_number}
Let $p \geq 2$ and let $G$ be a full-column rank matrix with unit norm columns denoted by $g_i$, $i = 1,\cdots, p$. Let $\Pi_{g_{-i}}(\cdot)$ denote the projection operator on the subspace generated by $\{g_1,\ldots,g_p\} \setminus \{g_i\}$. Let
\bear \label{eq:condition_on_r}
\varphi = \min \{\|g_i - \Pi_{g_{-i}}(g_i) \|: i=2,...,p\}.
\eear
Then
$$
\frac{\lambda_{max}(G'G)}{\lambda_{min}(G'G)} \leq \left(\frac{p}{\varphi}\right)^2.
$$
\end{lemma}

\begin{proof}[\textbf{Proof of \Cref{lemma:cond_number}}]

We start from an upper bound on the ratio of largest to smallest eigenvalues of $G'G$ in terms of its trace and the trace of its inverse. Recall that $G$ is full column-rank, which implies that $G'G$ is positive definite and we have that
\bear \label{eq:trace}
 \frac{\lambda_{max}(G'G)}{\lambda_{min}(G'G)} = \lambda_{max}(G'G)\lambda_{max}((G'G)^{-1}) \leq \mbox{tr}(G'G)\mbox{tr}((G'G)^{-1}),
\eear
where the last inequality follows from the fact that $G'G$ and $(G'G)^{-1}$ are positive definite and all eigenvalues are nonnegative. 

 Since the diagonal elements of $G'G$ are equal to one, $\mbox{tr}(G'G) = p$. Next, we establish an upper bound on $\mbox{tr}((G'G)^{-1})$. In order to do that, we use a fact about the inverse of correlation matrices, commonly denoted as precision matrices. Note that $G'G$ is a correlation matrix because all the columns of $G$ have norm one and we can consider that each $g_i$ is an ``observation" of some data in $\mathbb R^p$. For the inverse of correlation matrices, we have that the diagonal elements ($a_{ii}$) of $(G'G)^{-1}$ satisfy
\begin{align}\label{eq:precision_matrix}
R^2_i =  1 - \frac{1}{a_{ii}},
\end{align}
where $R^2_i$ is the coefficient of determination of the linear regression problem of the column $g_i$ onto the other columns $g_j$, $j \neq i$  (\cite{raveh1985use}). Note that since the $g_i$'s have unit norm, $R_i^2 = \|\Pi_{g_{-i}}(g_i) \|^2 \le 1$  and it follows from Eq. \eqref{eq:precision_matrix} that
$$
a_{ii}  =  \frac{1}{1 - \|\Pi_{g_{-i}}(g_i) \|^2} \leq \frac{1}{1- \max_{i=1,...,p} \|\Pi_{g_{-i}}(g_i) \|^2}. %= \frac{1}{\varphi^2}, 
$$
Note that
$\| g_i \|^2 = \|\Pi_{g_{-i}}(g_i) \|^2 + \|g_i - \Pi_{g_{-i}}(g_i) \|^2$, and using the fact that $\|g_i\| = 1$ and \eqref{eq:condition_on_r}, we have
$$
\max_{i=1,...,p} \|\Pi_{g_{-i}}(g_i) \|^2 \leq 1 - \varphi^2.
$$
In turn, we deduce that for $i=1,...p$
$$
a_{ii}  \le \frac{1}{\varphi^2}, 
$$
which implies that $\mbox{tr}((G'G)^{-1}) = \sum_i a_{ii} \leq p/\varphi^2$. Hence, returning to Eq. \eqref{eq:trace}, we have
$$
\frac{\lambda_{max}(G'G)}{\lambda_{min}(G'G)}  \leq \left(\frac{p}{\varphi}\right)^2,
$$
which concludes the proof.
\end{proof}

\begin{lemma}[Residuals]\label{lemma:FWL}
Let $g_i \in \mathbb R^p$, be a sequence of $p$ unit norm vectors such that
\bear \label{eq:condition_on_g}
%\min_{\gamma \in \mathbb{R}^{i-1}} \|g_i - \sum_{j=1}^{i-1} \gamma_j g_j\| \geq \eta
\|g_i - \Pi_{g_1,\cdots,g_{i-1}}(g_i) \|^2\ge \eta^2, \quad i = 2,\cdots, p,
\eear
where $\Pi_{g_1,\cdots,g_{i-1}}(\cdot)$ denote the projection operator on the subspace generated by $g_1,\cdots,g_{i-1}$.  Then the residuals of the projection of an arbitrary vector on the subspace generated by all other vectors can be lower bounded as follows
\begin{equation}\label{eq:regress-coeff}
 \|g_i - \Pi_{g_{-i}}(g_i) \|^2 \ge \eta^{2(p-i)} (\eta^2)^{\mathbf{1}\{i>1\}}, \quad i=1,...,p,
\end{equation}
where $\Pi_{g_{-i}}(\cdot)$ denotes the projection operator on the subspace generated by $\{g_1,\ldots,g_p\} \setminus \{g_i\}$.
%\newob{Do we have a different case for $i=1$?}
\end{lemma}

\begin{proof}[\textbf{Proof of \Cref{lemma:FWL}}]

In order to understand how close $\|g_i - \Pi_{g_{-i}}(g_i) \|^2$ can be to zero, we need to understand how much of the residual $\|g_i - \Pi_{g_1,...,g_{i-1}}(g_i) \|^2$ can be reduced when adding the new vectors $g_j$'s, for $j = i+1,..,p$. The more the vectors $g_j$'s can explain, the closest  $\|g_i - \Pi_{g_{-i}}(g_i) \|^2$ gets to zero. We will leverage that the vectors $g_j$'s  satisfy Eq. \eqref{eq:condition_on_g} to establish that  there is a limit for how small $\|g_i - \Pi_{g_{-i}}(g_i) \|^2$ can be.  

Two key properties that allow us to establish the result can be derived from the Frisch-Waugh-Lovell Theorem \citep{lovell2008simple}. Intuitively, the theorem states that when we add a new regressor to a linear regression model, it suffices to consider only the component of the new regressor that is orthogonal to the linear subspace of the regressors already present. We will use the following properties implied by the theorem:

\begin{itemize}
    \item {Projection decomposition: for $i=1,...,p$ and $k=0,...,p-i$, \begin{align}\label{eq:fwl_theorem}
\Pi_{g_1,...,g_{i+k}}(g_{i+k+1}) = \Pi_{g_1,...,g_{i-1},g_{i+1},...,g_{i+k}}(g_{i+k+1})+\Pi_{(g_i - \Pi_{g_1,...,g_{i-1},g_{i+1},...,g_{i+k}}(g_i))}(g_{i+k+1}), 
\end{align}}
    \item {$R^2$ decomposition: for $i=1,...,p$ and $k=0,...,p-i$,
    \begin{align}\label{eq:fwl_square}
  \hspace*{-1.4cm}  \|\Pi_{g_1,...,g_{i+k}}(g_{i+k+1})\|^2 = \|\Pi_{g_1,...,g_{i-1},g_{i+1},...,g_{i+k}}(g_{i+k+1})\|^2+\|\Pi_{(g_i - \Pi_{g_1,...,g_{i-1},g_{i+1},...,g_{i+k}}(g_i))}(g_{i+k+1})\|^2.
\end{align}}
\end{itemize}

The first property states that we can decompose the projection of $g_{i+1}$ in two parts. First, we consider the projection of $g_{i+1}$ onto the subspace without $g_i$. Second, we consider the projection of $g_{i+1}$ on the component of $g_i$ that is orthogonal to the subspace considered in the previous step, i.e., the projection of $g_{i+1}$ onto the vector $g_i - \Pi_{g_1,...,g_{i-1}}(g_{i})$. For the second property, since $g_i - \Pi_{g_1,...,g_{i-1}}(g_{i+1})$ is orthogonal to the vectors $g_1,...,g_{i-1}$, we have a situation in which the triangle equality holds and we can decompose the coefficient of determination $R^2$ in the contribution of each orthogonal set of regressors.

 Note that  when $i = p$, the result follows from the assumption (cf. Eq.\eqref{eq:condition_on_g}). Next, fix $i$ in $\{1,...,p-1\}$ and define the following sequence of residuals:
\begin{align}
\label{eq:residual_xi}
r^k_i = g_{i}-\Pi_{g_1,...,g_{i-1},g_{i+1},...,g_{i+k}}(g_{i}), \quad k = 0,...,p-i. 
\end{align}
 When $i = 1$ and $k=0$ we will define by convention $r^0_1 = g_1$. 
 
 We will show that the sequence of residuals satisfies $\|r^{k+1}_i\|^2 \geq \|r^k_i\|^2\eta^2$. This recursive relationship allows us to write 
$$
\|r_i^{k+1}\|^2 \geq \|r_i^0\|^2\eta^{2(k+1)},
$$
Moreover, $\|r_i^0\|^2 \geq \eta^2$ for $i>1$ (cf. Eq. \eqref{eq:condition_on_g}), and $\|r_1^0\|^2 = \|g_1\|^2 = 1$, which implies that for  $k = 0,...,p-i$,
$$
\|r^k_i\| \geq \begin{cases} \eta^{2k}, & i = 1 \\ \eta^{2(k+1)}, & 2 \leq i \leq p-1,\end{cases}
$$
Taking $k = p-i$  leads to \eqref{eq:regress-coeff}. Next, we  establish  that $\|r_i^{k+1}\|^2 \geq \|r_i^0\|^2\eta^{2(k+1)}$ holds.

  We define the following vectors 
$$
\nu_{i}^{k} = g_{i+k}-\Pi_{g_1,...,g_{i-1},g_{i+1},...,g_{i+k-1}}(g_{i+k}), \quad k = 0,...,p-i.
$$

Note that Eq. \eqref{eq:fwl_square} implies that $\nu_i^{k}$ is the vector that effectively contributes to reduce the residual $r_i^{k-1}$ when $g_{i+k}$ is used to explain $g_i$ and $\{g_1,...,g_{i+k-1}\}\setminus \{g_i\}$ were already considered.  Eq. \eqref{eq:fwl_square} gives us that
$$
\|\Pi_{g_1,...,g_{i-1},g_{i+1},...,g_{i+k+1}}(g_{i})\|^2 = \|\Pi_{g_1,...,g_{i-1},g_{i+1},...,g_{i+k}}(g_{i})\|^2+\|\Pi_{\nu_{i}^{k+1}}(g_{i})\|^2,
$$
and since $g_i$ has unit norm, by construction, we have
\bearn
\|\Pi_{g_1,...,g_{i-1},g_{i+1},...,g_{i+k+1}}(g_{i})\|^2  &=& 1-\|r^{k+1}_i\|^2\\
\|\Pi_{g_1,...,g_{i-1},g_{i+1},...,g_{i+k}}(g_{i})\|^2 &=& 1-\|r^k_i\|^2,
%= \|\Pi_{g_1,...,g_{i-1},g_{i+1},...,g_{i+k}}(g_{i})\|^2+\|\Pi_{\nu_{i}^{k+1}}(g_{i})\|^2.
\eearn
and hence the following recursive relationship
\bear \label{eq:rk1}
\|r^{k+1}_i\|^2 &=& \|r^k_i\|^2 - \|\Pi_{\nu_{i}^{k+1}}(g_{i})\|^2.
\eear
Next, we upper bound $\|\Pi_{\nu_{i}^{k+1}}(g_{i})\|^2$. We have that
\begin{align}\label{eq:r2_induction}
\|\Pi_{\nu_{i}^{k+1}}(g_{i})\|^2 &\stackrel{}{=} \frac{(g_i'\nu_{i}^{k+1})^2}{\|\nu_{i}^{k+1}\|^2} \stackrel{(a)}{=} \frac{\Big((\Pi_{g_1,...,g_{i-1},g_{i+1},...,g_{i+k}}(g_{i})+r^k_i)'\nu_{i}^{k+1}\Big)^2}{\|\nu_{i}^{k+1}\|^2} \stackrel{(b)}{=} \frac{\Big({r^k_i}'\nu_{i}^{k+1}\Big)^2}{\|\nu_{i}^{k+1}\|^2}, 
\end{align}
where $(a)$ follows from the definition of $r_i^k$ (Eq. \eqref{eq:residual_xi}), and $(b)$ follows from the fact that $\nu_{i}^{k+1}$ and $\Pi_{g_1,...,g_{i-1},g_{i+1},...,g_{i+k}}(g_{i})$ are orthogonal since we removed the projection of $g_{i+1}$ onto the subspace generated by $\{g_1,...,g_{i+k}\}\setminus \{g_i\}$ when constructing $\nu_{i}^{k+1}$.  Next we develop an alternative representation for $\nu_{i}^{k+1}$. We have
\bearn
\nu_{i}^{k+1} &=& g_{i+k+1} - \Pi_{g_1,...,g_{i-1},g_{i+1},...,g_{i+k}}(g_{i+k+1}) \\
&\stackrel{(a)}{=}& \Pi_{g_1,...,g_{i+k}}(g_{i+k+1})+ r^0_{i+k+1} -  \Pi_{g_1,...,g_{i-1},g_{i+1},...,g_{i+k}}(g_{i+k+1})\\
&\stackrel{(b)}{=}& \Pi_{g_1,...,g_{i-1},g_{i+1},...,g_{i+k}}(g_{i+k+1}) + \Pi_{\beta_i^k}(g_{i+k+1}) + r^0_{i+k+1} -  \Pi_{g_1,...,g_{i-1},g_{i+1},...,g_{i+k}}(g_{i+k+1})\\
&=& \Pi_{\beta_i^k}(g_{i+k+1}) + r^0_{i+k+1},
\eearn
where (a) follows from the definition of $r^0_{i+k+1}$ and in (b), we have used Eq. \eqref{eq:fwl_theorem} and defined $
\beta_i^k = g_{i}-\Pi_{g_1,...,g_{i-1},g_{i+1},...,g_{i+k}}(g_{i})
$.  Using this new representation for $\nu_{i}^{k+1}$ in Eq. \eqref{eq:r2_induction} yields
\begin{align}
\Big({r_i^k}'\nu_{i}^{k+1}\Big)^2 &= \Big({r_i^k}'(\Pi_{\beta_i^k}(g_{i+k+1}) +  r_{i+k+1}^0)\Big)^2 \nonumber \\
&\stackrel{(a)}{=} \Big({r_i^k}'\Pi_{\beta_i^k}(g_{i+k+1})\Big)^2 \nonumber \\
&\stackrel{(b)}{\leq} \|{r_i^k}\|^2\|\Pi_{\beta_i^k}(g_{i+k+1})\Big\|^2,
\end{align}
where $(a)$ follows from the fact that $r_{i+k+1}^0$ is orthogonal to the subspace generated by $g_1,...,g_{i+k}$, to which $r_i^k$ belongs and (b) follows from the Cauchy-Schwarz inequality. In addition, we have that $\Pi_{\nu_i^k}(g_{i+k+1})$ and $r_{i+k+1}^0$ are orthogonal since $r_{i+k+1}^0$ is by definition orthogonal to the subspace generated by $g_1,...,g_{i+k}$, which $\Pi_{\beta_i^k}(g_{i+k+1})$  belongs to. Therefore, we have that 
$$
\|\nu_i^{k+1}\|^2 = \|\Pi_{\beta_i^k}(g_{i+k+1})+r_{i+k+1}^0\|^2 = \|\Pi_{\beta_i^k}(g_{i+k+1})\|^2+\|r_{i+k+1}^0\|^2,
$$
which implies that
$$
\frac{\|\Pi_{\beta_i^k}(g_{i+k+1})\|^2}{\|\nu_i^{k+1}\|^2} = 1- \frac{\| r_{i+k+1}^0\|^2}{\|\nu_i^{k+1}\|^2} \leq 1- \| r_{i+k+1}^0\|^2,
$$
where the inequality follows from the fact that $\nu_i^{k+1}$ is a projection of a unit vector,so $\|\nu_i^{k+1}\|^2 \le 1$.

 Returning to  Eq. \eqref{eq:rk1} and  Eq. \eqref{eq:r2_induction}, we have established  that
$$
\|r^{k+1}_i\|^2 \ge \|r^k_i\|^2 - \|{r_i^k}\|^2(1-\|r_{i+k+1}^0\|^2) = \|{r_i^k}\|^2 \|r_{i+k+1}^0\|^2 \ge \|{r_i^k}\|^2 \eta^2,
$$
where the last inequality follows from assumption (cf. Eq. \eqref{eq:condition_on_g}). This completes the proof.

\end{proof}

\section{Additional Figures}
\label{app:figures}

We provide here additional analysis for the simulation results. For the cases depicted in Figures \ref{fig:sim01}, \ref{fig:sim02}, \ref{fig:sim03}  and \ref{fig:sim04}, we depict the results in log-log scale in Figures  \ref{fig:sim_sqrt1}, \ref{fig:sim_sqrt2}, \ref{fig:sim_sqrt3}  and \ref{fig:sim_sqrt4}, respectively. A dependence of the form $T^{\alpha}$ would lead to a linear relationship on this graph. Each graph depicts for reference the curve $T^{1/2}$.  We can see that our algorithms appear to achieve logarithmic regret when the time horizon is sufficiently long.  It is worth noticing that the transition regime (before achieving the logarithmic regret rate) can be affected by the dimension of the problem, which is also observed for the competing methods. Interestingly, for both OGD and EWU, we observe over the numerics that these  seem to achieve regret rates better than $O(\sqrt T)$ for the stochastic case. Whereas the  theoretical guarantees for these methods in the adversarial case are currently of order $O(\sqrt T)$, the numerics highlight an interesting theme for future research: delineating  whether or not OGD and EWU can achieve better regret bounds in non-adversarial environments.

\begin{figure}[!ht]
  \begin{subfigure}{0.49\textwidth}
\begin{tikzpicture}
\begin{loglogaxis}[
            title={},
	        width=7.5cm,
	        height=7cm,
	        xlabel =  Period, 
	        ylabel = Cumulative Regret, 
	        grid=major, 
	        legend pos=south east]
\addplot[color = magenta, ultra thick, dashed] table[x="time", y = "OGD", col sep = comma] {Data_simulations/d_10_pointed_sqrt.dat};\addlegendentry{OGD} 
\addplot[color = orange, ultra thick, dashdotted] table[x="time", y = "EWU", col sep = comma] {Data_simulations/d_10_pointed_sqrt.dat};\addlegendentry{EWU} 
\addplot[blue, ultra thick] table[x="time", y = "EllipsoidalCones", col sep = comma] {Data_simulations/d_10_pointed_sqrt.dat};\addlegendentry{EllipCones} 
\addplot[color = black] table[x="time", y = "sqrt_time", col sep = comma] {Data_simulations/d_10_pointed_sqrt.dat}; 
\end{loglogaxis}
\end{tikzpicture}
\caption{$d = 10$ and $T = 1000$.} \label{fig:sim_sqrt1} 
\end{subfigure}
\begin{subfigure}{0.49\textwidth}
\begin{tikzpicture}
\begin{loglogaxis}[
            title={},
	        width=7.5cm,
	        height=7cm,
	        xlabel =  Period, 
	        ylabel = Cumulative Regret, 
	        grid=major, 
	        legend pos=south east]
\addplot[color = magenta, ultra thick, dashed] table[x="time", y = "OGD", col sep = comma] {Data_simulations/d_25_pointed_sqrt.dat};\addlegendentry{OGD} 
\addplot[color = orange, ultra thick, dashdotted] table[x="time", y = "EWU", col sep = comma] {Data_simulations/d_25_pointed_sqrt.dat};\addlegendentry{EWU} 
\addplot[blue, ultra thick] table[x="time", y = "EllipsoidalCones", col sep = comma] {Data_simulations/d_25_pointed_sqrt.dat};\addlegendentry{EllipCones}
\addplot[color = black] table[x="time", y = "sqrt_time", col sep = comma] {Data_simulations/d_25_pointed_sqrt.dat};
\end{loglogaxis}
\end{tikzpicture}
\caption{$d = 25$ and $T = 5000$} \label{fig:sim_sqrt2} 
\end{subfigure}
\caption{Average cumulative regret over 50 simulations for EWU, OGD and the \texttt{EllipsoidalCones} algorithms for the pointed case in loglog scale.}
\end{figure}

\begin{figure}[!ht]
\begin{subfigure}{0.49\textwidth}
\begin{tikzpicture}
\begin{loglogaxis}[
            title={},
	        width=7.5cm,
	        height=7cm,
	        xlabel =  Period, 
	        ylabel = Cumulative Regret, 
	        grid=major, 
	        legend pos=south east]
\addplot[color = magenta, ultra thick, dashed] table[x="time", y = "OGD", col sep = comma] {Data_simulations/d_10_sphere_sqrt.dat};\addlegendentry{OGD} 
\addplot[blue, ultra thick] table[x="time", y = "ProjectedCones", col sep = comma] {Data_simulations/d_10_sphere_sqrt.dat};\addlegendentry{ProjCones} 
\addplot[color = black] table[x="time", y = "sqrt_time", col sep = comma] {Data_simulations/d_25_pointed_sqrt.dat};
\end{loglogaxis}
\end{tikzpicture}
\caption{$d = 10$ and $T = 1000$.} \label{fig:sim_sqrt3} 
\end{subfigure}
\begin{subfigure}{0.49\textwidth}
\begin{tikzpicture}
\begin{loglogaxis}[
            title={},
	        width=7.5cm,
	        height=7cm,
	        xlabel =  Period, 
	        ylabel = Cumulative Regret, 
	        grid=major, 
	        legend pos=south east]
\addplot[color = magenta, ultra thick, dashed] table[x="time", y = "OGD", col sep = comma] {Data_simulations/d_25_sphere_sqrt.dat};\addlegendentry{OGD} 
\addplot[blue, ultra thick] table[x="time", y = "ProjectedCones", col sep = comma] {Data_simulations/d_25_sphere_sqrt.dat};\addlegendentry{ProjCones} 
\addplot[color = black] table[x="time", y = "sqrt_time", col sep = comma] {Data_simulations/d_25_pointed_sqrt.dat};
\end{loglogaxis}
\end{tikzpicture}
\caption{$d = 25$ and $T = 5000$.} \label{fig:sim_sqrt4} 
\end{subfigure}
\caption{Average cumulative regret over 50 simulations for OGD and the \texttt{ProjectedCones} algorithms for the general case in loglog scale.}
\end{figure}

\end{document}